\documentclass[a4paper]{article}
\usepackage[margin=.95in]{geometry}
\usepackage{helvet}

\usepackage{amssymb, amsmath, amsthm}
\usepackage{url} 

\usepackage{authblk} 
\usepackage{algorithmic, algorithm}
\usepackage{color}
\usepackage{mathrsfs}
\usepackage{mathabx}  
\usepackage{multirow} 
\usepackage[all]{xy}
\usepackage{booktabs, xcolor, colortbl} 

\usepackage[figuresright]{rotating}

\colorlet{tableheadcolor}{gray!25} 
\colorlet{tablerowcolor}{gray!10} 
\newcommand{\refeq}[1]{\eqref{#1}} 
\newcommand{\reffig}[1]{Figure~\ref{#1}}
\newcommand{\reftab}[1]{Table~\ref{#1}}
\newcommand{\refsec}[1]{\S~\ref{#1}} 
\newcommand{\refapp}[1]{\ref{#1}} 
\newcommand{\refprop}[1]{Proposition~\ref{#1}}
\newcommand{\refeg}[1]{Example~\ref{#1}}
\newcommand{\refcor}[1]{Corollary~\ref{#1}}
\newcommand{\refthm}[1]{Theorem~\ref{#1}}

\newtheorem{thm}{Theorem}[section]

\newtheorem{prop}[thm]{Proposition}
\newtheorem{eg}[thm]{Example}
\newtheorem{cor}[thm]{Corollary}
\theoremstyle{definition}
\newtheorem{dfn}[thm]{Definition}
\theoremstyle{remark}

\newcommand{\subhead}[2]{\vspace{\baselineskip} \noindent \textbf{#1:} #2 \\}

\newcommand{\rone}[1]{#1}
\newcommand{\rtwo}[1]{#1}
\newcommand{\rsho}[1]{#1}
\newcommand{\redt}[1]{#1}

\newcommand{\xx}{\mathbf{x}}
\newcommand{\xxi}{\boldsymbol{\xi}}
\newcommand{\yy}{\mathbf{y}}
\newcommand{\ba}{\mathbf{a}}
\newcommand{\bu}{\mathbf{u}}

\newcommand{\dd}{\mathrm{d}}

\newcommand{\RR}{\mathbb{R}}
\newcommand{\Sph}{\mathbb{S}}
\newcommand{\CC}{\mathbb{C}}
\newcommand{\KK}{\mathrm{K}}
\newcommand{\NN}{\mathbb{N}}

\newcommand{\Half}{\mathbb{H}}

\newcommand{\YY}{\mathbb{Y}}

\newcommand{\params}{\mathbb{Y}}
\newcommand{\spSch}{\mathcal{S}}
\newcommand{\spsmooth}{\mathcal{E}}
\newcommand{\spcomp}{\mathcal{D}}

\newcommand{\spB}{\mathcal{B}} 
\newcommand{\spM}{\mathcal{O_M}}
\newcommand{\sppoly}{\mathcal{P}}
\newcommand{\spCD}{\mathcal{O'_C}}
\newcommand{\spTD}{\mathcal{S}'}
\newcommand{\spSD}{\mathcal{D}'}
\newcommand{\spED}{\mathcal{E}'}
\newcommand{\spLD}[1]{\mathcal{D}'_{L^#1}}
\newcommand{\spL}[1]{\mathcal{D}_{L^#1}}
\newcommand{\spBdot}{\dot{\mathcal{B}}}
\newcommand{\spLiz}{\mathcal{S}_0}
\newcommand{\spLizD}{\mathcal{S}'_0}
\newcommand{\spanyD}{\mathcal{A}'}
\newcommand{\spany}{\mathcal{A}}
\newcommand{\spA}{\mathcal{A}}
\newcommand{\spX}{\mathcal{X}}
\newcommand{\spY}{\mathcal{Y}}
\newcommand{\spZ}{\mathcal{Z}}
\newcommand{\spW}{\mathcal{W}}
\newcommand{\sploc}[1]{L^{#1}_\mathrm{loc}}
\newcommand{\bigdual}[1]{\left\langle#1\right\rangle }
\newcommand{\dual}[1]{\langle#1\rangle}
\newcommand{\ip}[1]{\left(#1\right) }
\newcommand{\ttprod}{\rone{\widehat{\otimes}}}

\newcommand{\eps}{\varepsilon}

\newcommand{\supp}{\mathrm{supp\,}}
\newcommand{\sign}{\mathrm{sgn}}

\newcommand{\dawson}{F}

\newcommand{\sth}[1]{#1^\text{th}}

\newcommand{\rad}{\mathrm{R}}
\newcommand{\drad}{\mathrm{R}^*}
\newcommand{\refl}[1]{\widetilde{#1}}
\newcommand{\wav}{\mathcal{W}}

\newcommand{\mellin}{\mathcal{M}}
\newcommand{\rid}{\mathscr{R}}
\newcommand{\drid}{\mathscr{R}^{\dag}}
\newcommand{\hil}{\mathcal{H}}

\newcommand{\bp}{\mathrm{\Lambda}}
\newcommand{\intrep}{\mathrm{T}}
\newcommand{\measY}{\mu}
\newcommand{\TT}{\mathrm{T}}
\newcommand{\hsch}{\mathrm{D}}
\newcommand{\gauss}{\mathrm{G}}
\newcommand{\ind}{\mathbf{1}}

\newcommand{\rbf}{\mathrm{G}}
\newcommand{\sig}{\sigma}
\newcommand{\sps}{\sig^{(-1)}}

\newcommand{\fz}{\zeta}
\newcommand{\fp}{\omega}
\newcommand{\pz}{z}
\newcommand{\vk}{v}

\newcommand{\riemann}{\int_{-\infty}^\infty}
\newcommand{\notsure}{\int_{-\infty}^\infty}
\newcommand{\lebesgue}{\int_{\RR}}
\newcommand{\neighbour}{\Omega}

\title{Neural Network with Unbounded Activation Functions is \\ Universal Approximator}
\author[1]{Sho Sonoda\thanks{s.sonoda0110@toki.waseda.jp}}
\author[1]{Noboru Murata}
\affil[1]{Faculty of Science and Engineering, Waseda University}
\date{}

\begin{document}
\maketitle

\begin{abstract}
This paper \redt{presents an investigation of} the approximation property of \rone{neural networks} with unbounded activation functions, such as the rectified linear unit (ReLU), which is the new de-facto standard of deep learning.
The ReLU network can be analyzed by the ridgelet transform with respect to Lizorkin distributions. 
By showing \rsho{three} reconstruction formulas by using the Fourier slice theorem, the Radon transform, \rsho{and Parseval's relation,} it is shown that \rone{a neural network} with unbounded activation functions still \rone{satisfies} the universal approximation property. As an additional consequence, the ridgelet transform, or the backprojection filter in the Radon domain, is what the network learns after backpropagation.
Subject to a constructive admissibility condition, the trained network can be obtained by simply discretizing the ridgelet transform, without backpropagation.
Numerical examples not only support the consistency of the admissibility condition but also imply that some non-admissible cases result in low-pass filtering.
\end{abstract}



\section{Introduction}
Consider approximating a function $f : \RR^m \to \CC$ by the neural network $g_J$ with an activation function $\eta : \RR \to \CC$ 
\begin{align}
g_J(\xx) = \frac{1}{J}\sum_j^J c_j \, \eta ( \ba_j \cdot \xx - b_j ), \quad (\ba_j, b_j, c_j ) \in \RR^m \times \RR \times \rone{\CC}
\end{align}
where we refer to $(\ba_j,b_j)$ as a hidden parameter and $c_j$ as an output parameter. 
Let $\YY^{m+1}$ denote the space of hidden parameters $\RR^m \times \RR$.
The network $g_J$ can be obtained by discretizing the {\em integral representation of the neural network}
\begin{align}
g(\xx) = \int_{\YY^{m+1}} \intrep(\ba,b) \eta ( \ba \cdot \xx - b ) \rsho{\dd \measY(\ba,b)},
\end{align}
\rsho{where $\intrep : \YY^{m+1} \to \CC$ corresponds to a continuous version of the output parameter; $\measY$ denotes a measure on $\YY^{m+1}$.}
The right-hand side expression is known as the {\em dual ridgelet transform} of $\intrep$ with respect to $\eta$
\begin{align}
\drid_\eta \intrep(\xx) = \int_{\YY^{m+1}} \intrep(\ba,b) \eta( \ba \cdot \xx - b ) \frac{\dd \ba \dd b}{\rsho{\| \ba \|}}.
\end{align}
By substituting in $\intrep(\ba,b)$ the {\em ridgelet transform} of $f$ with respect to $\psi$
\begin{align}
\rid_\psi f (\ba,b) := \int_{\RR^m} f(\xx) \overline{ \psi( \ba \cdot \xx - b )} \| \ba\| \dd \xx,
\end{align}
under some good conditions, namely the {\em admissibility} of $(\psi,\eta)$ and some regularity of $f$, we can reconstruct $f$ by
\begin{align}
\drid_\eta \rid_\psi f = f.
\end{align}
By discretizing the reconstruction formula,
we can \rtwo{verify} the approximation property of \rsho{neural networks} with the activation function $\eta$.

In this study, we investigate the approximation property of \rsho{neural networks} for the case in which $\eta$ is a Lizorkin distribution, 
by extensively constructing the ridgelet transform with respect to Lizorkin distributions.
The Lizorkin distribution space $\spLizD$ is 
such a large space that contains
 the {\em rectified linear unit} (ReLU) $z_+$, {\em truncated power functions} $z_+^k$, and 
other unbounded functions that have at most polynomial growth (but \rone{do not have polynomials as such}).
\reftab{tab:acts} and \reffig{fig:acts} give some examples of Lizorkin distributions.
\vspace{-.5cm}
\begin{table}[h]
  \centering
    \caption{Zoo of activation functions
\rtwo{with which the corresponding neural network can approximate arbitrary functions in $L^1(\RR^m)$ in the sense of pointwise convergence (\refsec{sec:reconst})
and in $L^2(\RR^m)$ in the sense of mean convergence (\refsec{sec:L2}).
The third column indicates the space $\spW(\RR)$ to which an activation function $\eta$ belong (\refsec{sec:ex.liz}, \ref{sec:ex.K}).}
}
  \begin{tabular}{@{\hspace{2em}}lll}
    \toprule
    activation function & $\eta(z)$ & $\spW$ \\
    \midrule
    \multicolumn{3}{l}{\textbf{unbounded functions}} \\
    truncated power function & $z_+^k := \begin{cases} z^k & z > 0 \\ 0 & \rone{z \leq 0} \end{cases}, \quad k \in \NN_0$ & $\spLizD$ \\
    rectified linear unit (ReLU) & $z_+$  & $\spLizD$ \\
    softplus function & $\sps(z) := \log( 1 + e^z ) $ & $\spM$ \\
    \multicolumn{3}{l}{\textbf{bounded but not integrable functions}} \\
    unit step function & $z_+^0$ & $\spLizD$ \\
    (standard) sigmoidal function & $\sigma(z) := (1+e^{-z})^{-1}$ & $\spM$ \\
    hyperbolic tangent function & $\tanh(z)$ & $\spM$ \\
    \multicolumn{3}{l}{\textbf{\rtwo{bump functions}}} \\
    (Gaussian) radial basis function & $\rbf(z) := (2 \pi)^{-1/2} \exp \left( -z^2/2 \right)$ & $\spSch$ \\
    the first derivative of sigmoidal function & $\sig'(z)$ & $\spSch$ \\
    Dirac's $\delta$ & $\delta(z)$ & $\spLizD$\\
    \multicolumn{3}{l}{\textbf{oscillatory functions}} \\
    the $\sth{k}$ derivative of RBF & $\rbf^{(k)}(z)$ & $\spSch$ \\
    the $\sth{k}$ derivative of sigmoidal function & $\sig^{(k)}(z)$ & $\spSch$\\
    the $\sth{k}$ derivative of Dirac's $\delta$ & $\delta^{(k)}(z)$ & $\spLizD$\\
    \bottomrule
  \end{tabular} \label{tab:acts}
\end{table}

\begin{figure}[h]
  \centering
  \includegraphics[width=.7\linewidth]{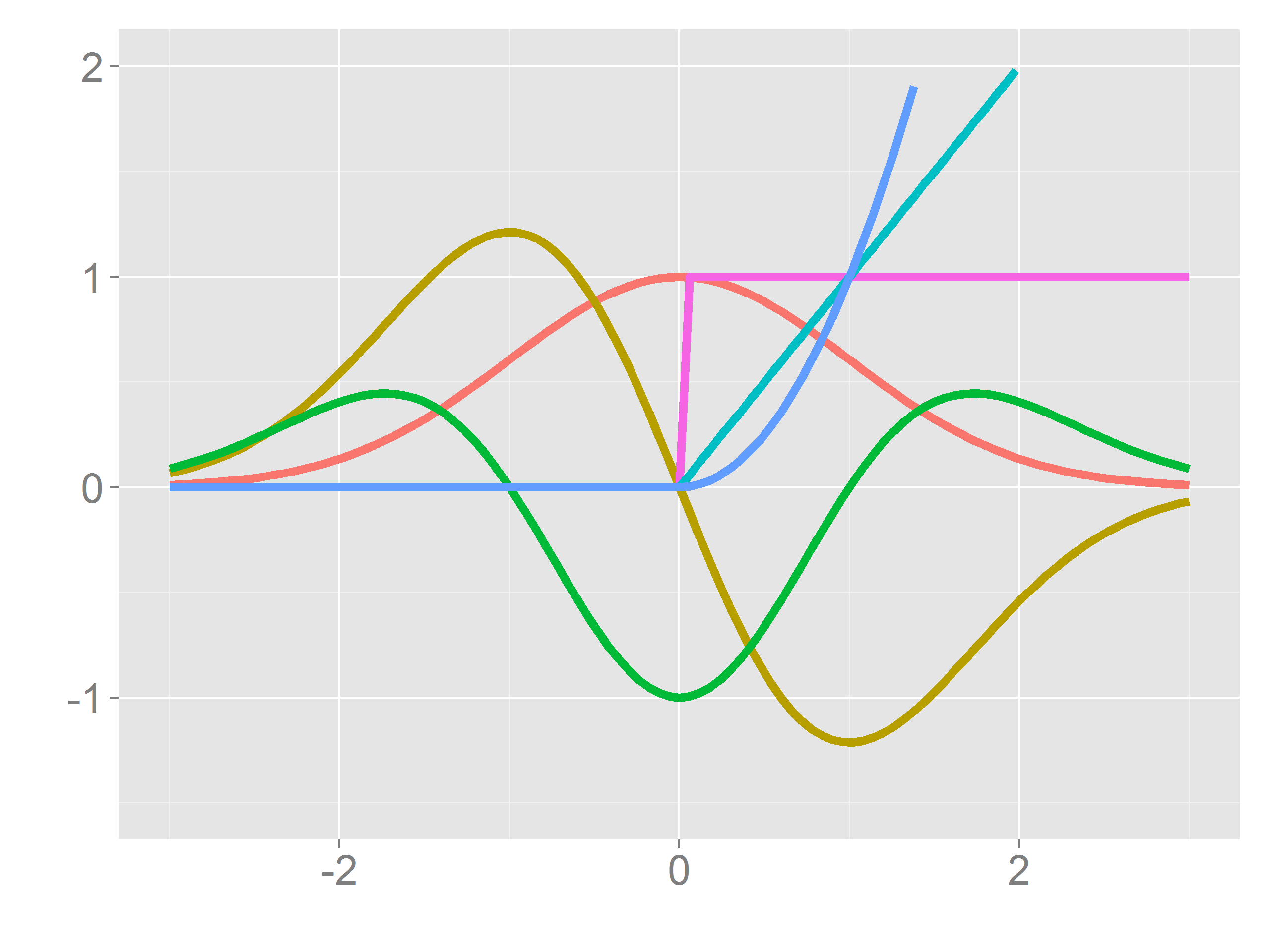}
  \vspace{-.5cm}
  \caption{Zoo of activation functions:
the Gaussian $\gauss(z)$ (red), the first derivative $\gauss'(z)$ (yellow), the second derivative $\gauss''(z)$ (green);
a truncated power function $z_+^2$ (blue), the ReLU $z_+$ (sky blue), the unit step function $z_+^0$ (rose).}
\label{fig:acts}
\end{figure}

\rsho{Recall that the derivative of the ReLU $z_+$ is the step function $z_+^0$.
Formally, the following suggestive formula
\begin{align}
\int_{\YY^{m+1}} \intrep(\ba,b) \eta'( \ba \cdot \xx - b )  \frac{\dd \ba \dd b}{\| \ba\|}
= \int_{\YY^{m+1}} \partial_b \intrep(\ba,b) \eta( \ba \cdot \xx - b )  \frac{\dd \ba \dd b}{\| \ba\|},
\end{align}
holds, because the integral representation is a convolution in $b$.
This formula suggests that once we have $\intrep_\mathrm{step}(\ba,b)$ for the step function,
which is implicitly known to exist based on some of our preceding studies \cite{Murata1996, Sonoda2014},
then we can formally obtain $\intrep_\mathrm{ReLU}(\ba,b)$ for the ReLU by differentiating $\intrep_\mathrm{ReLU}(\ba,b) = \partial_b \intrep_\mathrm{step}(\ba,b).$}

\subsection{ReLU and Other Unbounded Activation Functions}
The ReLU \cite{ReLU.Glorot, maxout, ReLU.Hinton, ReLU.Leaky}
became a new building block of {\em deep neural networks},
in the place of traditional bounded activation functions such as the sigmoidal function and the radial basis function (RBF).
\rsho{Compared with} traditional \rtwo{units}, a neural network with the ReLU is said \cite{ReLU.Glorot, Jarrett2009, Krizhevsky2012, ReLU.Zeiler, ReLU.Leaky}
to learn faster because it has larger gradients that can alleviate the vanishing gradient \cite{ReLU.Glorot},
and perform \redt{more efficiently} because it extracts sparser features.
To date, these hypotheses have only been empirically verified without analytical evaluation.

It is worth noting that in approximation theory,
it was already shown in the 1990s that neural networks with such unbounded activation \rtwo{functions}
 have the universal approximation property.
To be precise, if the activation function is {\em not} a polynomial function,
then \rtwo{the family of all neural networks} is dense in some functional spaces such as $L^p(\RR^m)$ and $C^0(\RR^m)$.
Mhaskar and Micchelli \cite{Mhaskar.Micchelli} seem to be the first to \rone{have shown} such universality by using the B-spline.
Later, Leshno et al. \cite{Leshno1993} reached a stronger claim by using functional analysis. Refer to Pinkus \cite{Pinkus.survey} for more details.

In this study, we initially work through the same statement by using harmonic analysis, or the ridgelet transform.
One strength is that our results are very constructive.
Therefore, we can construct what the network will learn during backpropagation.
\rtwo{Note that for bounded cases this idea is already implicit in \cite{Ito.Radon} and \cite{Sonoda2014}, and explicit in \cite{Kainen2007}.}

\subsection{Integral Representation of Neural Network and Ridgelet Transform}

We use the {\em integral representation of neural networks} introduced by Murata \cite{Murata1996}.
As already mentioned,
the integral representation corresponds to the dual ridgelet transform.
In addition, the ridgelet transform corresponds to the composite of a wavelet transform after the Radon transform.
Therefore, neural networks have a profound connection with harmonic analysis and tomography.

As K\r{u}rkov\'{a} \cite{Kurkova2012} noted, the idea of discretizing integral transforms to obtain an approximation is very old in approximation theory.
As for neural networks, at first, Carroll and Dickinson \cite{Carroll.Dickinson} and Ito \cite{Ito.Radon}
regarded a neural network as a Radon transform \cite{Helgason.new}.
Irie and Miyake \cite{Irie.Miyake}, Funahashi \cite{Funahashi1989}, Jones \cite{Jones1992}, and Barron \cite{Barron1993} 
used Fourier analysis to show the approximation property in a constructive way.
K\r{u}rkov\'{a} \cite{Kurkova2012} applied Barron's error bound to evaluate the complexity of neural networks.
Refer to Kainen et al. \cite{kainen.survey} for more details.

In the late 1990s, Cand\`{e}s \cite{Candes.HA, Candes.PhD}, Rubin \cite{Rubin.calderon}, and Murata \cite{Murata1996}
independently \redt{proposed} the so-called ridgelet transform, which has since been investigated by a number 
of authors \cite{Donoho.kplane, Donoho.ridgelet, Starck2010, Rubin.ridgelet, Kostadinova2014, Kostadinova2015}.

\subsection{Variations of Ridgelet Transform}
A ridgelet transform $\rid_\psi$, along with its reconstruction property, is determined by four classes \rtwo{of functions}:
domain $\spX(\RR^m)$, range $\spY(\YY^{m+1})$, ridgelet $\spZ(\RR)$, and dual ridgelet $\spW(\RR)$.
\begin{align}
\xymatrix{
\spX(\RR^m) \ni f \ar@/^2pc/[rr]^*+{\psi \in \spZ(\RR)}|*+{\rid_\psi} && \intrep \in \spY(\YY^{m+1}) \ar@/^2pc/[ll]|*+{\drid_\eta}^*+{\eta \in \spW(\RR)}
}
\end{align}

The following ladder relations by Schwartz \cite{Schwartz.new} are fundamental for describing the variations of the ridgelet transform:
\begin{gather}
\begin{array}{lccccccccccc}
(\mbox{functions})& \spcomp &\subset &\spSch &\subset &\spL{1} &\subset &\spL{p} &\subset &\spM &\subset &\spsmooth\\
& \cap &&\cap &&\cap &&\cap &&\cap &&\cap\\
(\mbox{distributions})& \spED &\subset &\spCD &\subset &\spLD{1} &\subset &\spLD{p} &\subset &\spTD &\subset &\spSD \\
    \multicolumn{1}{c}{} & \multicolumn{5}{@{}l@{}}{%
      \raisebox{.5\normalbaselineskip}{%
      \rlap{$\underbrace{\hphantom{\mbox{$p\lor{}q$\hspace*{\dimexpr11\arraycolsep+\arrayrulewidth}$p\land{}q$}}}_{\mbox{integrable}}$}}%
    }
    & \multicolumn{1}{c}{} &
\multicolumn{5}{@{}l@{}}{%
      \raisebox{.5\normalbaselineskip}{%
      \rlap{$\underbrace{\hphantom{\mbox{$p\lor{}q$\hspace*{\dimexpr13\arraycolsep+\arrayrulewidth}$p\land{}q$}}}_{\mbox{not always bounded}}$}}%
    }
\end{array},
\end{gather}
where the meaning of symbols are given below in \reftab{tab:class}.

The integral transform $\intrep$ by Murata \cite{Murata1996} coincides with the case for $\spZ \subset \spcomp$ and $\spW \subset \spsmooth \cap L^1$.
Cand\`{e}s \cite{Candes.HA, Candes.PhD} proposed the ``ridgelet transform'' for $\spZ = \spW \subset \spSch$.
Kostadinova et al. \cite{Kostadinova2014, Kostadinova2015} defined the ridgelet transform for the Lizorkin distributions $\spX = \spLizD$,
which is the broadest domain ever known, at the cost of restricting the choice of ridgelet functions to the Lizorkin functions $\spW = \spZ = \spLiz \subset \spSch$.

\subsection{Our Goal}
Although many researchers have investigated the ridgelet transform \cite{Donoho.kplane, Rubin.ridgelet, Kostadinova2014, Kostadinova2015},
in all the settings $\spZ$ does {\em not} directly admit some fundamental activation functions, namely the sigmoidal function and the ReLU.
One of the challenges we faced is to define the ridgelet transform for $\spW = \spLizD$, which admits the sigmoidal function and the ReLU.

\section{Preliminaries}

\subsection{Notations}
Throughout this paper, we consider approximating $f : \RR^m \to \CC$ by a neural network $g$ with hidden parameters $(\ba,b)$.
Following Kostadinova et al. \cite{Kostadinova2014, Kostadinova2015}, we denote by $\YY^{m+1} := \RR^m \times \RR$ the space of parameters $(\ba,b)$.
As already denoted, we symbolize the domain of a ridgelet transform as $\spX(\RR^m)$, the range as $\spY(\YY^{m+1})$, the space of ridgelets as $\spZ(\RR)$, and the space of dual ridgelets as $\spW(\RR)$.

We denote by $\Sph^{m-1}$ the $(m\!-\!1)$-sphere $\{ \bu \in \RR^m \mid \| \bu \| = 1 \}$;
by $\RR_+$ the open half-line $\{ \alpha \in \RR \mid \alpha > 0\}$;
by $\Half$ the open half-space $ \RR_+ \times \RR $.
We denote by $\NN$ and $\NN_0$ the sets of natural numbers excluding $0$ and including $0$, respectively.

We denote by $\refl{\cdot}$ the reflection $ \refl{f}(\xx) := f(-\xx) $;
\rsho{by $\overline{\cdot}$ the complex conjugate; 
by $a \lesssim b$ that there exists a constant $C \geq 0$ such that $a \leq C b$.}

\subsection{Class of Functions and Distributions}

Following Schwartz, we denote the classes of functions and distributions as in \reftab{tab:class}.
For Schwartz's distributions, we refer to Schwartz \cite{Schwartz.new} and Tr\`{e}ves \cite{Treves.new};
for Lebesgue spaces, Rudin \cite{Rudin.FA.new}, Brezis \cite{Brezis.new} and Yosida \cite{Yosida.new};
for Lizorkin distributions, Yuan et al. \cite{Yuan} and Holschneider \cite{Holschneider.new}.
\begin{table}[h!]
  \centering
    \caption{Classes of functions and distributions, and corresponding dual spaces.}
  \begin{tabular}{ll@{\hspace{3em}}ll}
    \toprule
    space & $\spany(\RR^k)$ & dual space & $\spanyD(\RR^k)$ \\
    \midrule
    polynomials of all degree & $\sppoly(\RR^k)$ & -- & \\
    smooth functions & $\spsmooth(\RR^k)$ & compactly supported distributions& $\spED(\RR^k)$ \\
    rapidly decreasing functions & $\spSch(\RR^k)$ & tempered distributions& $\spTD(\RR^k)$ \\
    compactly supported smooth functions & $\spcomp(\RR^k)$ & Schwartz distributions & $\spSD(\RR^k)$ \\
    $L^p$ of Sobolev order $\infty$ $(1 \leq p < \infty)$ & $\spL{p}(\RR^k)$ & Schwartz dists. $(1/p + 1/q=1)$ & $\spLD{q}(\RR^k)$ \\
    \rone{completion of $\spcomp(\RR^k)$ in $\spL{\infty}(\RR^k)$} & \rone{$\spBdot(\RR^k)$} & 
    \rone{Schwartz dists. $(p=1)$}
& $ \spLD{1}(\RR^k)$ \\
    slowly increasing functions & $\spM(\RR^k) $ & -- &  \\
    -- & & rapidly decreasing distributions & $\spCD(\RR^k) $  \\
    Lizorkin functions & $\spLiz(\RR^k)$ & Lizorkin distributions & $\spLizD(\RR^k)$ \\
    \bottomrule
  \end{tabular} \label{tab:class}
\end{table}

The space $\spLiz(\RR^k)$ of Lizorkin functions is a closed subspace of $\spSch(\RR^k)$ that consists of elements such that all moments vanish.
That is, $\spLiz(\RR^k) := \{ \phi \in \spSch(\RR^k) \mid \int_{\RR^k} \xx^{\boldsymbol{\alpha}} \phi(\xx) \dd \xx = 0 \mbox{ for any } \boldsymbol{\alpha} \in \NN_0^k \}$.
The dual space $\spLizD(\RR^k)$, known as the Lizorkin distribution space,
is homeomorphic to the quotient space of $\spTD(\RR^k)$ by the space of all polynomials $\sppoly(\RR^k)$.
That is, $\spLizD(\RR^k) \cong \spTD(\RR^k)/\sppoly (\RR^k)$. \rsho{Refer to Yuan et al. \cite[Prop.~8.1]{Yuan} for more details.}
In this work we identify and treat every polynomial as zero in the Lizorkin distribution.
That is, for \rone{$p \in \spTD(\RR^k)$}, if $p \in \sppoly(\RR^k)$ then $p \equiv 0$ in \rone{$\spLizD(\RR^k)$}.

\rsho{For $\Sph^{m-1}$, we work on the two subspaces $\spcomp(\Sph^{m-1}) \subset \spcomp(\RR^{m})$ and $\spED(\Sph^{m-1}) \subset \spED(\RR^m)$. In addition, we identify $\spcomp = \spSch = \spM = \spsmooth$
 and $\spED = \spCD = \spTD = \spLD{p} = \spSD$.}

\rsho{For $\Half$, let $\spsmooth(\Half) \subset \spsmooth(\RR^2)$ and $\spcomp(\Half) \subset \spcomp(\RR^2)$.
For $\TT \in \spsmooth(\Half)$, write
\begin{align}
\hsch_{s,t}^{k,\ell} \TT(\alpha,\beta) := \left( \alpha + 1/\alpha \right)^s (1+\beta^2)^{t/2} \partial_\alpha^k \partial_\beta^\ell \TT(\alpha,\beta), \quad s,t, k,\ell \in \NN_0.
\end{align}
The space $\spSch(\Half)$ 
consists of $\TT \in \spsmooth(\Half)$ such that for any $s,t,k,\ell \in \NN_0$,
the seminorm below is finite
\begin{align}
\sup_{(\alpha,\beta) \in \Half} \big|\hsch_{s,t}^{k,\ell} \TT(\alpha,\beta) \big| < \infty.
\end{align}
The space $\spM(\Half)$ consists of $\TT \in \spsmooth(\Half)$ such that for any $k,\ell \in \NN_0$ there exist $s,t \in \NN_0$ such that
\begin{align}
\big| \hsch^{k,\ell}_{0,0} \TT(\alpha,\beta) \big| \lesssim (\alpha + 1/\alpha)^s (1+\beta^2)^{t/2}.
\end{align}
The space $\spSD(\Half)$ consists of all bounded linear functionals $\Phi$ on  $\spcomp(\Half)$ such that
for every compact set $\KK \subset \Half$, there exists $N \in \NN_0$ such that
\begin{align}
\Bigg| \int_\KK \TT(\alpha,\beta) \Phi(\alpha,\beta) \frac{\dd \alpha \dd \beta}{\alpha} \Bigg| \lesssim \sum_{k,\ell \leq N} \sup_{(\alpha,\beta) \in \Half} | \hsch_{0,0}^{k,\ell} \TT (\alpha,\beta) |, \quad \forall \TT \in \spcomp(\KK),
\end{align}
where the integral is understood as the action of $\Phi$.
The space $\spTD(\Half)$ consists of $\Phi \in \spSch(\Half)$ for which there exists $N \in \NN_0$ such that
\begin{align}
\Bigg| \int_\Half \TT(\alpha,\beta) \Phi(\alpha,\beta) \frac{\dd \alpha \dd \beta}{\alpha} \Bigg|
\lesssim
\sum_{s,t,k,\ell \leq N}
\sup_{(\alpha, \beta) \in \Half} \big| \hsch_{s,t}^{k,\ell} \TT (\alpha,\beta)\big|, \quad 
\forall \TT \in \spSch(\Half).
\end{align}}

\subsection{Convolution of Distributions}
\reftab{tab:conv} lists the convergent convolutions of distributions and their ranges by Schwartz \cite{Schwartz.new}.
\begin{table}[h!]
  \centering
  \caption{Range of convolution (excerpt from Schwartz \cite{Schwartz.new})}
  \begin{tabular}{llll}
	\toprule
	case & $\spany_1$ & $\spany_2$ & $\spany_1 * \spany_2$ \\
	\midrule
	regularization & $\spcomp$ & $\spSD, \spLD{p}, \spED$ & $\spsmooth, L^p, \spcomp$ \\
	compactly supported distribution & $\spED$ & $\spED, \spsmooth, \rsho{\spSD} $ & $\spED, \spsmooth, \rsho{\spSD}$ \\
	regularization & $\spSch$ & $\spSch, \spTD$ & $\spSch, \spM$ \\
	Schwartz convolutor & $\spCD$ & $\spSch, \spCD, \spLD{p}, \spTD$ & $\spSch, \spCD, \spLD{p}, \spTD$ \\
	Young's inequality & $L^p$ & $L^q$ & $L^r \ (1/r=1/p+1/q-1)$ \\
	Young's inequality & $\spLD{p}$ & $\spL{q}, \spLD{q}$ & \rone{$\spLD{r} \ (1/r=1/p+1/q-1)$} \\
	\bottomrule
  \end{tabular} \label{tab:conv}
\end{table}

In general a convolution of distributions may neither commute $\phi * \psi \neq \psi * \phi$ nor associate $\phi * (\psi * \eta) \neq (\phi * \psi) * \eta$.
According to Schwartz \cite[Ch.6 Th.7, Ch.7 Th.7]{Schwartz.new}, both $\spSD * \spED * \spED * \cdots$ and $\spTD * \spCD * \spCD * \cdots$ are commutative and associative.

\subsection{Fourier Analysis} 
The Fourier transform $\widehat{\cdot}$ of $f:\RR^m \to \CC$ and the inverse Fourier transform $\widecheck{\cdot}$ of $F:\RR^m \to \CC$ are given by
\begin{gather}
\widehat{f}(\xxi) := \int_{\RR^m} f(\xx) e^{- i \xx \cdot \xxi} \dd \xx, \quad \xxi \in \RR^m \\
\quad  \widecheck{F}(\xx):= \frac{1}{(2 \pi)^m}\int_{\RR^m} F(\xxi) e^{i \xx \cdot \xxi} \dd \xxi, \quad \xx \in \RR^m.
\end{gather}

The Hilbert transform $\hil$ of $f : \RR \to \CC$ is given by
\begin{gather}
\hil f(s) := \frac{i}{\pi} \mathrm{p.v.} \riemann \frac{f(t)}{s-t} \dd t, \quad s \in \RR
\end{gather}
where $\mathrm{p.v.} \riemann$ denotes the principal value.
\rsho{We set the coefficients above to satisfy}
\begin{gather}
\widehat{\hil f}(\omega) = \sign \, \omega \cdot  \widehat{f}(\omega), \\
\hil^2 f(s) = f(s).
\end{gather}

\subsection{Radon Transform} 
The Radon transform $\rad$ of $f:\RR^m \to \CC$ and the dual Radon transform $\drad$ of $\Phi : \Sph^{m-1} \times \RR \to \CC$
are given by
\begin{gather}
\rad f(\bu, p) := \int_{(\RR \bu)^\perp} f(p \bu + \yy) \dd \yy, \quad (\bu,p) \in \Sph^{m-1} \times \RR \label{eq:radon} \\
\drad \Phi(\xx) := \int_{\Sph^{m-1}} \Phi(\bu, \bu \cdot \xx) \dd \bu, \quad \xx \in \RR^m \label{eq:dradon}
\end{gather}
where $(\RR \bu)^\perp := \{ \yy \in \RR^m \mid \yy \cdot \bu = 0 \}$ denotes the orthogonal complement of a line $\RR \bu \subset \RR^m$; and $\dd \yy$ denotes the Lebesgue measure on $(\RR \bu)^\perp$;
and $\dd \bu$ denotes the \rone{surface} measure on $\Sph^{m-1}$.

We use the following fundamental results (\cite{Helgason.new, Hertle1983}) for $f \in L^1(\RR^m)$ without proof:
{\em Radon's inversion formula}
\begin{align}
\drad \bp^{m-1} \rad f &= 2 (2 \pi)^{m-1} f,
\end{align}
where the {\em backprojection filter} $\bp^m$ is defined in \eqref{eq:bp};
the {\em Fourier slice theorem}
\begin{align}
\widehat{f}(\omega \bu) = \lebesgue \rad f(\bu,p) e^{-i p \omega }\dd p, \quad (\bu,\omega ) \in \Sph^{m-1} \times \RR
\end{align}
where
the left-hand side is the $m$-dimensional Fourier transform,
whereas the right-hand side is the one-dimensional Fourier transform of the Radon transform;
and a corollary of {\em Fubini's theorem}
\begin{align}
\lebesgue \rad f(\bu, p) \dd p = \int_{\RR^m} f(\xx) \dd \xx, \quad \mbox{a.e. } \bu \in \Sph^{m-1}.
\end{align}

\subsection{Backprojection filter}
For a function $\Phi(\bu,p)$, we define the {\em backprojection filter} $\bp^m$ as
\begin{align} \label{eq:bp}
\bp^{m} \Phi (\bu, p) :=
\left\{
\begin{array}{ll}
\partial_p^{m} \Phi (\bu, p), & \mbox{$m$ even} \\
\hil_p \partial_p^{m} \Phi (\bu, p), &\mbox{$m$ odd}.
\end{array}
\right.
\end{align}
where $\hil_p$ and $\partial_p$ denote the Hilbert transform and the partial differentiation with respect to $p$, respectively.
It is designed as a one-dimensional Fourier multiplier with respect to $p \to \omega$ such that
\begin{align}
\widehat{ \bp^m \Phi }(\bu, \omega) = i^m |\omega|^m \widehat{\Phi}(\bu, \omega).
\end{align}

\section{Classical Ridgelet Transform}
\subsection{An Overview}
The {\em ridgelet transform} $\rid_\psi f$ of $f:\RR^m \to \CC$
with respect to $\psi:\RR \to \CC$ is formally given by
\begin{align}
\rid_\psi f(\ba, b) &:= \int_{\RR^m} f(\xx) \overline{\psi( \ba \cdot \xx - b )} \rsho{\| \ba \|^s} \dd \xx, \quad (\ba,b) \in \YY^{m+1} \rsho{\mbox{ and } s > 0}. \label{eq:eucrid}
\end{align}
\rsho{The factor $|\ba|^s$ is simply posed for technical convenience. After the next section we set $s=1$, which simplifies some notations (e.g., \refthm{thm:existence}).
Murata \cite{Murata1996} originally posed $s=0$, which is suitable for the Euclidean formulation.
Other authors such as Cand\`{e}s \cite{Candes.PhD} used $s=1/2$, Rubin \cite{Rubin.calderon} used $s=m$, and Kostadinova et al. \cite{Kostadinova2014} used $s=1$ .}

When $f \in L^1(\RR^m)$ and $\psi \in L^\infty(\RR)$, by using H{\"o}lder's inequality,
the ridgelet transform is absolutely convergent at every $(\ba,b) \in \YY^{m+1}$.
\begin{align}
\int_{\RR^m} \big| f(\xx) \overline{\psi( \ba \cdot \xx - b )} \rsho{\| \ba \|^s} \big| \dd \xx &\leq \| f \|_{L^1(\RR^m)} \cdot  \| \psi \|_{L^\infty(\RR)} \rsho{\cdot \| \ba \|^s} < \infty.
\end{align}
\rsho{In particular when $s=0$, the estimate is independent of $\ba$ and thus $\rid_\psi f \in L^\infty(\YY^{m+1})$.
Furthermore, $\rid$ is a bounded bilinear operator $L^1(\RR^m) \times L^\infty(\RR) \to L^\infty(\YY^{m+1})$.}

The {\em dual ridgelet transform} $\drid_\eta \intrep$ of $\intrep : \YY^{m+1} \to \CC$ with respect to $\eta : \RR \to \CC$ is formally given by
\begin{align}
\drid_\eta \intrep(\xx) &:= \int_{\YY^{m+1}} \intrep(\ba,b) \eta( \ba \cdot \xx - b ) \rsho{\| \ba \|^{-s}} \dd \ba \dd b, \quad \xx \in \RR^m. \label{eq:drid}
\end{align}
The integral is absolutely convergent when $\eta \in L^\infty(\RR)$ and $\intrep \in L^1(\YY^{m+1}; \rsho{\| \ba \|^{-s}}\dd \ba \dd b)$ at every $\xx \in \RR^m$,
\begin{align}
\int_{\YY^{m+1}} \big| \intrep(\ba,b) \eta( \ba \cdot \xx - b ) \big| \rsho{\|\ba\|^{-s}} \dd \ba \dd b &\leq \| \intrep \|_{L^1(\YY^{m+1}; \rsho{\|\ba\|^{-s}} \dd \ba \dd b)} \cdot  \| \eta \|_{L^\infty(\RR)} < \infty,
\end{align}
\rsho{and thus $\drid$ is a bounded bilinear operator $L^1(\YY^{m+1};\| \ba \|^{-s} \dd \ba \dd b) \times L^\infty(\RR) \to L^\infty(\RR^m)$.}

Two functions $\psi$ and $\eta$ are said to be {\em admissible}
when 
\begin{align}
K_{\psi,\eta} := (2 \pi)^{m-1} \notsure \frac{\overline{\widehat{\psi}(\fz)} \widehat{\eta}(\fz)}{|\fz|^m} \dd \fz ,
\end{align}
is finite and not zero.
Provided that $\psi, \eta$, and $f$ belong to some good classes, and $\psi$ and $\eta$ are admissible,
then the {\em reconstruction formula}
\begin{align}
\drid_\eta \rid_\psi f = K_{\eta,\psi} f,
\end{align}
holds.

\subsection{Ridgelet Transform in Other Expressions}
It is convenient to write the ridgelet transform in ``polar'' coordinates as 
\begin{align}
\rid_\psi f(\bu, \alpha, \beta) &=
\int_{\RR^m} f(\xx) 
\overline{ \psi \left( \frac{\bu \cdot \xx - \beta}{\alpha} \right)} \rsho{\frac{1}{\alpha^s}}
 \dd \xx,
\end{align}
where ``polar'' variables are given by
\begin{gather}
\bu := \ba / \| \ba \|,  \quad
\alpha := 1 / \| \ba \|, \quad
\beta := b / \| \ba \|.
\end{gather}
Emphasizing the connection with wavelet analysis, we define the ``radius'' $\alpha$ as reciprocal.
Provided there is no likelihood of confusion, we use the same symbol $\params^{m+1}$ for the parameter space, regardless of whether it is parametrized by $(\ba,b) \in \RR^m \times \RR$ or $(\bu, \alpha, \beta) \in \Sph^{m-1} \times \RR_+ \times \RR$.

For a fixed $(\bu,\alpha,\beta) \in \YY^{m+1}$,
the ridgelet function
\begin{align}
\psi_{\bu,\alpha,\beta}(\xx)
&:= \psi \left( \frac{\bu \cdot \xx - \beta}{\alpha} \right) \rsho{\frac{1}{\alpha^s}}, \quad \xx \in \RR^m
\end{align}
behaves as a constant function on $(\RR \bu)^\perp$, and as a dilated and translated wavelet function on $\RR \bu$.
That is, by using the orthogonal decomposition $\xx = p \bu + \yy$ with $p \in \RR$ and $\yy \in (\RR \bu )^\perp$,
\begin{align}
\psi_{\bu,\alpha,\beta}(\xx)
= \psi \left( \frac{\bu \cdot ( p \bu + \yy ) - \beta}{\alpha} \right) \rsho{\frac{1}{\alpha^s}}
= \psi \left( \frac{p -\beta}{\alpha} \right) \rsho{\frac{1}{\alpha^s}} \otimes 1(\yy).
\end{align}

By using the decomposition above and Fubini's theorem, and assuming that the ridgelet transform is absolutely convergent,
we have the following equivalent expressions
\begin{align}
\rid_\psi f (\bu, \alpha, \beta)
&= \lebesgue \left( \int_{(\RR \bu)^\perp}f(p \bu + \yy) \dd \yy \right) \overline{\psi \left( \frac{p - \beta}{\alpha} \right)} \rsho{\frac{1}{\alpha^s}} \dd p \\
&= \rsho{ \lebesgue \rad f\left(\bu, p \right) \overline{\psi \left( \frac{p-\beta}{\alpha} \right)} \frac{1}{\alpha^s} \dd p } \\
&= \lebesgue \alpha^{\rsho{1-s}} \rad f\left(\bu, \alpha\pz+ \beta \right) \overline{\psi(\pz)} \dd\pz &\text{(weak form)} \label{eq:weakridge} \\
&= \left(\rad f (\bu, \cdot ) * \overline{\refl{\psi_{\alpha}}}\right) (\beta), \quad \psi_\alpha(p) := \psi \left( \frac{p}{\alpha} \right)\rsho{\frac{1}{\alpha^s}} &\text{(convolution form)} \label{eq:convridge} \\
&= \frac{1}{2 \pi} \lebesgue \widehat{f}(\omega \bu) \overline{\widehat{\psi}(\alpha \omega)} \alpha^{\rsho{1-s}} e^{i \omega \beta} \dd \omega, &\text{(Fourier slice th. \cite{Kostadinova2014})} \label{eq:fstridge}
\end{align}
\rsho{ where $\rad$ denotes the Radon transform \eqref{eq:radon};
the Fourier form follows by applying the identity $\mathcal{F}^{-1}_\omega \mathcal{F}_p = \mathrm{Id}$ to the convolution form.}
These reformulations reflect a well-known claim \cite{Starck2010, Kostadinova2014} that ridgelet analysis is wavelet analysis in the Radon domain.

\subsection{Dual Ridgelet Transform in Other Expressions}
Provided the dual ridgelet transform \refeq{eq:drid} is absolutely convergent,
some changes of variables lead to other expressions.
\begin{align}
\drid_\eta \intrep(\xx)
&= \rsho{ \int_{\RR^m} \int_\RR \intrep(\ba,b) \eta(\ba \cdot \xx - b) \| \ba \|^{-s} \dd b \, \dd \ba } \\
&= \rsho{ \int_0^\infty \int_{\Sph^{m-1}} \int_{\RR} \intrep (r \bu,b) \eta(r \bu \cdot \xx - b) \dd b \, \dd \bu \, r^{m-s-1} \dd r} \\
&=\int_{\Sph^{m-1}} \int_0^\infty \lebesgue
\intrep \left( \frac{\bu}{\alpha}, \frac{\beta}{\alpha} \right) 
\eta \left( \frac{ \bu \cdot \xx - \beta}{\alpha} \right) 
\frac{\dd \beta \dd \alpha \dd \bu}{\rsho{\alpha^{m-s+2}}} &\text{(polar expression)} \\
&= \int_{\Sph^{m-1}} \int_0^\infty \lebesgue
     \intrep \left( \bu, \alpha, \bu \cdot \xx -\alpha \pz \right)
     \eta(\pz) \frac{\dd \pz \dd \alpha \dd \bu}{\rsho{\alpha^{m-s+1}}}, &\text{(weak form)} \label{eq:weakdridge}
\end{align}
\rsho{where every integral is understood to be an iterated integral;
the second equation follows by substituting $(r, \bu) \gets (\| \ba \|, \ba / \| \ba \|)$ and using the coarea formula for polar coordinates;
the third equation follows by substituting $(\alpha,\beta) \gets (1/r, b/r)$ and using Fubini's theorem;
in the fourth equation with a slight abuse of notation, we write $\intrep(\bu, \alpha, \beta) := \intrep(\bu / \alpha, \beta / \alpha)$.}

\rsho{Furthermore, write $\eta_\alpha(p) := \eta(p/\alpha)/\alpha^t$.
Recall that the dual Radon transform $\drad$ is given by \eqref{eq:dradon} and the Mellin transform $\mellin$ \cite{Holschneider.new} is given by
$\mellin f(z) := \int_0^\infty f (\alpha) \alpha^{z-1} \dd \alpha, \ z \in \CC$.
Then,
\begin{align}
\drid_\eta \TT(\xx) = \drad\left[ \mellin [\TT(\bu,\alpha,\cdot) * \eta_\alpha ](s+t-m-1) \right](\xx).
\end{align}
Note that the composition of the Mellin transform and the convolution is the dual wavelet transform \cite{Holschneider.new}.
Thus, the dual ridgelet transform is the composition of the dual Radon transform and the dual wavelet transform.}
\section{Ridgelet Transform with respect to Distributions}

Using the weak expressions \refeq{eq:weakridge} and \refeq{eq:weakdridge}, we define the ridgelet transform with respect to distributions. \rsho{Henceforth, we focus on the case for which the index $s$ in \eqref{eq:eucrid} equals $1$.}

\subsection{Definition and Well-Definedness}
\begin{dfn}[Ridgelet Transform with respect to Distributions]
The ridgelet transform $\rid_\psi f$ of a function $f \in \spX(\RR^m)$
with respect to a distribution $\psi \in \spZ(\RR)$ is given by
\begin{align}
\rid_\psi f(\bu, \alpha, \beta)
&:= \lebesgue \rad f \left( \bu, \alpha \pz + \beta\right) \overline{\psi(\pz)} \dd \pz, \quad (\bu,\alpha,\beta) \in \YY^{m+1}
\end{align}
\rsho{where $\int_\RR \cdot \, \overline{\psi(z)} \dd z$ is understood as the action of a distribution $\psi$.}
\end{dfn}

\rone{Obviously,}
this ``weak'' definition coincides with the ordinary strong one
when $\psi$ coincides with a locally integrable function $(\sploc{1})$.
\rsho{With a slight abuse of notation, the weak definition coincides with the convolution form}
\begin{align}
\rsho{\rid_\psi f(\bu,\alpha,\beta)
= \left(\rad f (\bu, \cdot ) * \overline{\refl{\psi_{\alpha}}}\right) (\beta), \quad (\bu,\alpha,\beta) \in \YY^{m+1}}
\end{align}
\rsho{ where $\psi_\alpha(p) := \psi \left( p / \alpha \right) /\alpha $;
the convolution $\cdot * \cdot $, dilation $\cdot _\alpha$, reflection $\widetilde{\cdot}$, and complex conjugation $\overline{\cdot}$ are understood as operations for Schwartz distributions.}

\begin{thm}[Balancing Theorem] \label{thm:existence}
The ridgelet transform $\rid : \spX(\RR^m) \times \spZ(\RR) \to \spY(\YY^{m+1})$ is well defined as a bilinear map
when $\spX$ and $\spZ$ are chosen from \reftab{tab:weakridge}.
\end{thm}

\begin{table}[h!]
  \centering
  \caption{Combinations of classes for which the ridgelet transform is well defined as a bilinear map.
    The first and third columns list domains $\spX(\RR^m)$ of $f$ and $\spZ(\RR)$ of $\psi$, respectively.
    The second column lists the range of the Radon transform $\rad f(\bu,p)$ for which we reused the same symbol $\spX$ as it coincides.
    The fourth, fifth, and sixth columns list the range of the ridgelet transform with respect to $\beta$, $(\alpha,\beta)$, and $(\bu,\alpha,\beta)$, respectively.
  }
  \begin{tabular}{llllll}
    \toprule
    $f(\xx)$ & $\rad f(\bu,p)$ & $\psi(\pz)$ &
\multicolumn{3}{c}{
$\rid_\psi f(\bu,\alpha,\beta)$}
\\
\cmidrule(r){4-6}
    $\spX(\RR^m)$ & $\spX(\Sph^{m-1} \times \RR)$ & $\spZ(\RR)$ & $\spB(\RR)$ & $\spA(\Half)$ & $\spY(\YY^{m+1})$\\
    \midrule
     $\spcomp$ &  $\spcomp$ &  $\spSD$ & $\spsmooth$ & $\spsmooth$ & \rsho{$\spsmooth$} \\
     $\spED$   &  $\spED$ & $\rsho{\spSD}$ & $\rsho{\spSD}$ & \rsho{$\spSD$} & \rsho{$\spSD$} \\
     $\spSch$  &  $\spSch$ &  $\spTD$ & $\spM$ & \rsho{$\spM$} & \rsho{$\spM$} \\ 
     $\spCD$   &  $\spCD$ & $\spTD$ & $\spTD$ & \rsho{$\spTD$} & \rsho{$\spTD$} \\
     $L^1$ &  $L^1$ & $L^p \cap C^0$ & $L^p \cap C^0$ & \rsho{$\spTD$} & \rsho{$\spTD$} \\
     $\spLD{1}$ &  $\spLD{1}$ & $\spLD{p}$ & $\spLD{p}$ & \rsho{$\spTD$} & \rsho{$\spTD$} \\ 
    \bottomrule
  \end{tabular} \label{tab:weakridge}
\end{table}

\rsho{The proof is provided in \refapp{app:proof.existence}.}
Note that each $\spZ$ is (almost) the largest in the sense that the convolution $\spB = \spX * \spZ$ converges.
Thus, \reftab{tab:weakridge} suggests that there is a trade-off relation between $\spX$ and $\spZ$, that is, as $\spX$ increases, $\spZ$ decreases and vice versa.

Extension of the ridgelet transform of non-integrable functions requires
more sophisticated approaches,
because a direct computation of the Radon transform may diverge.
For instance, Kostadinova et al. \cite{Kostadinova2014} \rone{extend} $\spX = \spSch_0'$
by using a duality technique. \rsho{In \refsec{sec:L2} we extend the ridgelet transform to $L^2(\RR^m)$, by using the bounded extension procedure.}

\begin{prop}[Continuity of the Ridgelet Transform $L^1(\RR^m) \to L^\infty(\YY^{m+1})$] \label{prop:conti.L1}
\rsho{Fix $\psi \in \spSch(\RR)$. The ridgelet transform $\rid_\psi : L^1(\RR^m) \to L^\infty(\YY^{m+1})$ is bounded.}
\end{prop}
\begin{proof}
Fix an arbitrary $f \in L^1(\RR^m)$ and $\psi \in \spSch(\RR)$.
Recall that this case is absolutely convergent. By using the convolution form,
\begin{align}
\mathop{\mathrm{ess \ sup}}_{(\bu,\alpha,\beta)}
\Big| \left( \rad f(\bu,\cdot) * \overline{\widetilde{\psi_\alpha}} \right) (\beta) \Big|
&\leq \| f \|_{L^1(\RR^m)} \cdot \mathop{\mathrm{ess \ sup}}_{(\alpha,\beta)} | \psi_\alpha (\beta) | \\
&\leq \| f \|_{L^1(\RR^m)} \cdot \mathop{\mathrm{ess \ sup}}_{(r,\beta)} | r \cdot \psi(r \beta) | < \infty,
\end{align}
where the first inequality follows by using Young's inequality and applying $\int_\RR | \rad f(\bu,p) | \dd p = \| f \|_1$;
the second inequality follows by changing the variable $r \gets 1/\alpha$, and the resultant is finite because $\psi$ decays rapidly.
\end{proof}

The ridgelet transform $\rid_\psi$ is injective when $\psi$ is {\em admissible},
because if $\psi$ is admissible then the reconstruction formula holds and thus $\rid_\psi$ has the inverse.
However, $\rid_\psi$ is {\em not} always injective.
For instance, take a Laplacian $f := \Delta g$ of some function $g \in \spSch(\RR^m)$
and a polynomial $\psi(z) = z + 1$, \rtwo{which satisfies $\psi^{(2)} \equiv 0$.}
According to \reftab{tab:weakridge}, $\rid_\psi f$ exists as a smooth function because $f \in \spSch(\RR^m)$ and $\psi \in \spTD(\RR)$.
In this case $\rid_\psi f = 0$, which means $\rid_\psi$ is not injective.
That is,
\begin{align}
\rid_\psi f (\bu, \alpha,\beta)
&= \left( \rad \Delta g(\bu,\cdot) * \overline{\refl{\psi_\alpha}}\right)(\beta) \\
&= \left( \partial^{2} \rad g(\bu,\cdot) * \overline{\refl{\psi_\alpha}}\right)(\beta) \\
&= \left( \rad g(\bu,\cdot) * \partial^{2} \overline{\refl{\psi_\alpha}}\right)(\beta) \\
&= \left( \rad g(\bu,\cdot) * 0 \right)(\beta) \\
&= 0,
\end{align}
\rtwo{where the second equality follows by the intertwining relation $\rad \Delta g(\bu,p) = \partial_p^2 \rad g(\bu,p)$ \cite{Helgason.new}.}
\rone{Clearly}
the non-injectivity stems from the choice of $\psi$.
In fact, as we see in the next section, 
\rtwo{no polynomial can be admissible and thus $\rid_\psi$ is not injective for any polynomial $\psi$.}

\subsection{Dual Ridgelet Transform with respect to Distributions}
\begin{dfn}[Dual Ridgelet Transform with respect to Distributions]
The dual ridgelet transform $\drid_\eta \intrep$ of $\intrep \in \spY(\YY^{m+1})$ with respect to $\eta \in \spW(\RR)$ is given by
\begin{align}
\drid_\eta \intrep(\xx)
&= \lim_{\substack{\delta \to \infty \\ \eps \to 0}}\int_{\Sph^{m-1}} \int_\eps^\delta \lebesgue \intrep \left( \bu, \alpha, \bu \cdot \xx -\alpha \pz \right) \eta(\pz) \frac{\dd\pz\dd \alpha \dd \bu}{\alpha^{m}},
\quad \xx \in \RR^{m}
\end{align}
\rsho{where $\int_\RR \cdot \eta(z) \dd z$ is understood as the action of a distribution $\eta$.}
\end{dfn}

If the dual ridgelet transform $\drid_\eta$ exists, 
then it coincides with the dual operator \cite{Yosida.new} of the ridgelet transform $\rid_\eta$.
\begin{thm} \label{thm:dual}
Let $\spX$ and $\spZ$ be chosen from \reftab{tab:weakridge}. Fix $\psi \in \spZ$.
Assume that $\rid_\psi : \spX(\RR^m) \to \spY(\YY^{m+1})$ is injective and
that $\drid_{\psi}:\spY'(\YY^{m+1}) \to \spX'(\RR^{m})$ exists.
Then $\drid_\psi$ is the dual operator $(\rid_\psi)':\spY'(\YY^{m+1}) \to \spX'(\RR^m)$ of $\rid_\psi$.
\end{thm}
\begin{proof}
By assumption $\rid_\psi$ is densely defined on $\spX(\RR^m)$
and injective.
Therefore, by a classical result on the existence of the dual operator \cite[VII. 1. Th. 1, pp.193]{Yosida.new},
 there uniquely exists a dual operator $(\rid_\psi)' : \spY'(\YY^{m+1}) \to \spX'(\RR^m)$.
On the other hand, for $f \in \spX(\RR^m)$ and $\intrep \in \spY(\YY^{m+1})$,
\begin{align}
\bigdual{\rid_\psi f, \intrep}_{\YY^{m+1}}
= \int_{\RR^m \times \YY^{m+1}} f(\xx) \overline{\psi ( \ba \cdot \xx - b) \intrep (\ba,b)} \dd \xx \dd \ba \dd b
= \bigdual{f, \drid_{\psi} \intrep}_{\RR^m}.
\end{align}
By the uniqueness of the dual operator, we can conclude $(\rid_\psi)' = \drid_{\psi}$.
\end{proof}

\section{Reconstruction Formula for Weak Ridgelet Transform}
In this section we discuss the admissibility condition and the reconstruction formula,
 not only in the Fourier domain as many authors did \cite{Candes.HA, Candes.PhD, Murata1996, Kostadinova2014, Kostadinova2015},
but also in the real domain and in the Radon domain. Both domains are key to the constructive formulation.
In \refsec{sec:ac} we derive a constructive admissibility condition.
In \refsec{sec:reconst} we show two reconstruction formulas.
The first of these formulas is obtained by using the Fourier slice theorem and the other by using the Radon transform.
In \refsec{sec:L2} we will \rsho{extend the ridgelet transform to $L^2$.}

\subsection{Admissibility Condition} \label{sec:ac}
\begin{dfn}[Admissibility Condition]
A pair $(\psi, \eta) \in \spSch(\RR) \times \spTD(\RR)$ is said to be admissible when 
\rsho{there exists a neighborhood $\neighbour \subset \RR$ of $0$ such that $\rone{\widehat{\eta} \in \sploc{1}(\neighbour \setminus \{ 0 \}), }$ and the integral
\begin{gather}
K_{\psi, \eta} := (2 \pi)^{m-1} \left( \int_{\neighbour \setminus \{ 0\}} + \int_{\RR \setminus \neighbour} \right) \frac{\overline{\widehat{\psi}(\fz)}\widehat{\eta}(\fz)}{|\fz|^m}\dd \fz, \label{eq:defK}
\end{gather}
\rone{converges and is not zero}, 
where $\int_{\neighbour \setminus \{ 0 \}}$ and $\int_{\RR \setminus \neighbour}$ are understood as 
Lebesgue's integral and the action of $\widehat{\eta}$, respectively.}
\end{dfn}

Using the Fourier transform in $\spW$ requires us to assume that $\spW \subset \spTD$.

\rsho{The second integral $\int_{\RR \setminus \neighbour}$ is always finite 
because $|\zeta|^{-m} \in \spM(\RR \setminus \neighbour)$ and thus $|\zeta|^{-m} \overline{\widehat{\psi}(\zeta)}$ decays rapidly; therefore, by definition the action of a tempered distribution $\widehat{\eta}$ always converges.
The convergence of the first integral $\int_{\neighbour \setminus \{ 0 \}}$ does not depend on the choice of $\neighbour$ because 
for every two neighborhoods $\neighbour$ and $\neighbour'$ of $0$, the residual $\int_{\neighbour \setminus \neighbour'}$ is always finite.
Hence, the convergence of $K_{\psi,\eta}$ does not depend on the choice of $\neighbour$.}

\rsho{The removal of $0$ from the integral is essential because a product of two singular distributions, which is indeterminate in general, can occur at $0$. See examples below.
In \refapp{app:proof.formula}, we have to treat $|\zeta|^{-m}$ as a locally integrable function, rather than simply a regularized distribution such as Hadamard's finite part.
If the integrand coincides with a function at $0$, then obviously $\int_{\RR \setminus \{ 0 \}} = \int_\RR$.}

\rsho{If $\widehat{\eta}$ is supported in the singleton $\{ 0 \}$ then $\eta$ cannot be admissible because $K_{\psi,\eta} = 0$ for any $\psi \in \spSch(\RR)$.
According to Rudin \cite[Ex. 7.16]{Rudin.FA.new}, it happens if and only if $\eta$ is a polynomial.}
Therefore, it is natural to take $\spW = \spTD / \sppoly \cong \spLizD$ rather than $\spW = \spTD$.
That is, in $\spLizD(\RR)$,  we identify a polynomial \rsho{$\eta \in \sppoly(\RR)$ as $0 \in \spTD(\RR)$.
The integral $K_{\psi,\eta}$ is well-defined for $\spLizD(\RR)$. Namely $K_{\psi,\eta}$ is invariant under the addition of a polynomial $Q$ to $\eta$
\begin{align}
K_{\psi, \eta} = K_{\psi, \eta + Q}.
\end{align}
}

\begin{eg}[{Modification of Schwartz \cite[Ch.5 Th.6]{Schwartz.new}}]
Let $\eta(z) = z$ and $\psi(z) = \bp \gauss(z)$ with $\gauss(z) = \exp( -z^2/2 )$.
Then,
\begin{gather}
\widehat{\eta}(\fz) = \delta(\fz) \quad  \mbox{ and } \quad \overline{\widehat{\psi}(\fz)} = |\fz| \cdot \gauss(\fz).
\end{gather}
In this case the product of the two distributions is not associative
\begin{gather}
\int_{\RR} \mathrm{p.v.} \ \frac{1}{|\fz|} \times \left( |\fz| \cdot \gauss(\fz) \times \delta(\fz)\right) \dd \fz =  0, \\
\int_{\RR} \left( \mathrm{p.v.} \ \frac{1}{|\fz|} \times |\fz| \cdot \gauss(\fz) \right) \times \delta(\fz) \dd \fz = \gauss(0) \neq 0.
\end{gather}
On the other hand \eqref{eq:defK} is well defined
\begin{gather}
K_{\psi,\eta} = 
 \int_{0 < |\zeta| < 1} \frac{|\fz| \cdot \gauss(\fz) \times 0}{|\fz|} \dd \zeta  + \int_{1 \leq |\zeta|} \frac{|\fz| \cdot \gauss(\fz)}{|\fz|} \delta(\fz) \dd \zeta = 0.
\end{gather}
\end{eg}

\begin{eg}
\rsho{Let $\eta(z) = z_+^0 + (2 \pi)^{-1} \exp iz$ and $\psi(z) = \bp \gauss(z)$.
Then,
\begin{gather}
\widehat{\eta}(\fz) = \frac{i}{\zeta} + \delta(\fz) + \delta(\fz-1) \quad  \mbox{ and } \quad \overline{\widehat{\psi}(\fz)} = |\fz| \cdot \gauss(\fz).
\end{gather}
The product of the two distributions is not associative
\begin{gather}
\int_{\RR} \mathrm{p.v.} \ \frac{1}{|\fz|} \times \left( |\fz| \cdot \gauss(\fz) \times \left( \frac{i}{\zeta} + \delta(\fz) + \delta(\fz-1) \right)\right) \dd \fz =  \gauss(1), \\
\int_{\RR} \left( \mathrm{p.v.} \ \frac{1}{|\fz|} \times |\fz| \cdot \gauss(\fz) \right) \times \left( \frac{i}{\zeta} + \delta(\fz) + \delta(\fz-1)\right) \dd \fz = \gauss(0) + \gauss(1) \neq 0.
\end{gather}
On the other hand, \eqref{eq:defK} is well defined
\begin{gather}
K_{\psi,\eta} = 
 \int_{0 < |\zeta| < 1} \frac{|\fz| \cdot \gauss(\fz) \times  i \zeta^{-1} }{|\fz|} \dd \zeta  + \int_{1 \leq |\zeta|} \frac{|\fz| \cdot \gauss(\fz)}{|\fz|} \left( \mathrm{p.v.} \ \frac{i}{\zeta} + \delta(\fz) + \delta(\fz-1) \right) \dd \zeta = \infty + G(1).
\end{gather}}
\end{eg}

Observe that formally the integrand $\widehat{u}(\fz) := \overline{\widehat{\psi}(\fz)}\widehat{\eta}(\fz) |\fz|^{-m}$ 
is a solution of $|\fz|^m \widehat{u}(\fz) = \overline{\widehat{\psi}(\fz)} \widehat{\eta}(\fz)$.
By taking the Fourier inversion, we have $\bp^m u = \overline{\widetilde{\psi}} * \eta$.
\rsho{To be exact, $\widehat{\eta}$ may contain a point mass at the origin, such as Dirac's $\delta$.}

\begin{thm}(Structure Theorem for Admissible Pairs) \label{thm:eq.ac}
\rsho{Let $(\psi, \eta) \in \spSch(\RR) \times \spTD(\RR)$.
Assume that there exists $k \in \NN_0$ such that
\begin{align}
\widehat{\eta}(\zeta) = \sum_{j=0}^k c_j \delta^{(j)}(\zeta), \quad \zeta \in \{ 0 \}.
\end{align}
Assume that there exists a neighborhood $\neighbour$ of $0$ such that $\widehat{\eta} \in C^0(\neighbour \setminus \{ 0 \})$.
Then $\psi$ and $\eta$ are admissible if and only if there exists $u \in \spM(\RR)$ such that
\begin{gather}
\bp^m u = \overline{\refl{\psi}} * \left( \eta - \sum_{j=0}^k c_j z^j \right) \quad \mbox{ and } \quad 
\int_{\RR \setminus \{ 0 \}} \widehat{u}(\zeta) \dd \zeta \neq 0,
\end{gather}
where $\bp$ is the backprojection filter defined in \eqref{eq:bp}.
In addition, 
$\lim_{\zeta \to +0} |\widehat{u}(\zeta)| < \infty$ and  
$\lim_{\zeta \to -0} |\widehat{u}(\zeta)| < \infty$.}
\end{thm}
\rsho{The proof is provided in \ref{app:proof.eq.ac}. Note that the continuity implies local integrability.
If $\psi$ has $\ell$ vanishing moments with $\ell \geq k$, namely $\int_\RR \psi(z) z^j \dd z = 0$ for $j \leq \ell$,
then the condition reduces to
\begin{align}
\bp^m u = \overline{\widetilde{\psi}} * \eta, \quad 
\quad \Bigg| \int_\RR u(z) \dd z \Bigg| < \infty
\quad \mbox{ and }
\quad \int_\RR \widehat{u}(\zeta) \dd \zeta \neq 0.
\end{align}}

As a consequence of \rsho{\refthm{thm:eq.ac}}, we can construct admissible pairs as below.
\begin{cor}[Construction of Admissible Pairs] \label{cor:const.ap}
Given $\eta \in \spLizD(\RR)$. Assume that there exists a neighborhood $\neighbour$ of $0$ and $k \in \NN_0$ such that
$\zeta^k \cdot \widehat{\eta}(\zeta) \in C^0(\neighbour)$.
Take $\psi_0 \in \spSch(\RR)$ such that
\begin{align}
\int_\RR \zeta^k \, \overline{\widehat{\psi_0}(\zeta)} \widehat{\eta}(\zeta) \dd \zeta \neq 0.
\end{align}
Then
\begin{align}
\psi := \bp^m \psi_0^{(k)},
\end{align}
is admissible with $\eta$.
\end{cor}
\rsho{The proof is obvious because $u := \overline{\widetilde{\psi^{(k)}_0}} * \eta$ satisfies the conditions in \refthm{thm:eq.ac}.}

\subsection{Reconstruction Formula} \label{sec:reconst}
\begin{thm}[Reconstruction Formula] \label{thm:formula}
Let $f \in L^1(\RR^m)$ satisfy $\widehat{f} \in L^1(\RR^m)$ and let $(\psi, \eta) \in \spSch(\RR) \times \spLizD(\RR)$ be admissible.
Then the reconstruction formula
\begin{align}
\drid_\eta \rid_\psi f(\xx) = K_{\psi,\eta} f(\xx),
\end{align}
holds for almost every $\xx \in \RR^m$.
\rsho{The equality holds for every point where $f$ is continuous.}
\end{thm}

\rsho{The proof is provided in \ref{app:proof.formula}.} The admissibility condition can be easily inverted to $(\psi, \eta) \in \spLizD \times \spSch$.
However, extensions to $\spLizD \times \spLizD$ and $\spSch \times \spSD$ may {\em not} \rtwo{be} easy.
\rtwo{This is because} the multiplication $\spLizD \cdot \spLizD$ is not always commutative, nor associative,
and the Fourier transform is not always defined over $\spSD$ \cite{Schwartz.new}.

The following theorem is another suggestive reconstruction formula that
 implies wavelet analysis in the Radon domain works as a backprojection filter.
In other words, the admissibility condition requires $(\psi, \eta)$ to construct the filter $\bp^m$.
Note that similar techniques are obtained for ``wavelet measures'' by Rubin \cite{Rubin.ridgelet, Rubin.calderon}.
\begin{thm}[Reconstruction Formula via Radon Transform] \label{thm:formula.radon}
Let $f \in \rsho{L^1}(\RR^m)$ \rsho{be sufficiently smooth} and $(\psi,\eta) \in \spSch(\RR) \times \spTD(\RR)$ be admissible.
Assume that there exists a real-valued smooth and integrable function $u$ such that
\begin{gather}
\bp^{m} u = \overline{\refl{\psi}} * \eta \quad \mbox{ and } \quad \int_\RR \widehat{u}(\zeta) \dd \zeta = -1. \label{eq:radon.ac}
\end{gather}
Then, 
\begin{gather}
\drid_\eta \rid_\psi f(\xx) = \drad \bp^{m-1} \rad f(\xx) = 2 (2 \pi)^{m-1} f(\xx),
\end{gather}
holds for almost every $\xx \in \RR^m$.
\end{thm}

\rsho{The proof is provided in \ref{app:proof.formula.radon}. Note that here we imposed a stronger condition on $u$ than the $u \in L^1(\RR \setminus \{ 0 \})$ we imposed in \refthm{thm:eq.ac}.}

Recall intertwining relations (\cite[Lem.2.1, Th.3.1, Th.3.7]{Helgason.new})
\begin{gather}
(- \Delta)^\frac{m-1}{2} \drad = \bp^{m-1} \rad, \quad \mbox{ and } \quad \rad (- \Delta)^\frac{m-1}{2} = \drad \bp^{m-1}.
\end{gather}
Therefore, we have the following.
\begin{cor} \label{cor:radon.d}
\begin{gather}
\drid_\eta \rid_\psi = \drad \bp^{m-1} \rad = (- \Delta)^\frac{m-1}{2} \drad \rad = \drad \rad (- \Delta)^\frac{m-1}{2}.
\end{gather}
\end{cor}

\subsection{Extension to $L^2$} \label{sec:L2}
\rsho{By $\ip{\cdot , \cdot}$ and $\| \cdot \|_2$, with a slight abuse of notation, we denote the inner product of $L^2(\RR^m)$ and $L^2(\YY^{m+1})$.
Here we endow $\YY^{m+1}$ with a fixed measure $\alpha^{-m} \dd \alpha \dd \beta \dd \bu$, and omit writing it explicitly as $L^2(\YY^{m+1}; \dots)$.
We say that $\psi$ is {\em self-admissible} if $\psi$ is admissible in itself, i.e. the pair $(\psi,\psi)$ is admissible.}
The following relation is immediate by the duality.
\begin{thm}[Parseval's Relation and Plancherel's Identity] \label{thm:parseval}
Let $(\psi, \eta) \in \spSch \times \spTD$ be admissible with, for simplicity, $K_{\psi,\eta} = 1$.
\rsho{For $f, g \in L^1 \cap L^2(\RR^m)$,}
\begin{align}
&\ip{ \rid_\psi f, \rid_\eta g } = \ip{\drid_\eta \rid_\psi f, g} = \ip{ f, g }.  &\text{Parseval's Relation} \\
\intertext{In particular, if $\psi$ is self-admissible, then}
&\| \rid_\psi f \|_2 = \| f \|_2. &\text{Plancherel's identity}
\end{align}
\end{thm}
\rsho{Recall \refprop{prop:conti.L1} that the ridgelet transform is a bounded linear operator on $L^1(\RR^m)$.
If $\psi \in \spSch(\RR)$ is self-admissible, then we can extend the ridgelet transform to $L^2(\RR^m)$, by following the bounded extension procedure \cite[2.2.4]{Grafakos.classic}.
That is, for $f \in L^2(\RR^m)$, take a sequence $f_n \in L^1 \cap L^2(\RR^m)$ such that $f_n \to f$ in $L^2$.
Then by Plancherel's identity,
\begin{align}
\| f_n - f_m \|_2 = \| \rid_\psi f_n - \rid_\psi f_m \|_2, \quad \forall n,m \in \NN.
\end{align}
The right-hand side is a Cauchy sequence in $L^2(\YY^{m+1})$ as $n,m \to \infty$.
By the completeness, there uniquely exists the limit $\TT_\infty \in L^2(\YY^{m+1})$ of $\rid_\psi f_n$.
We regard $\TT_\infty$ as the ridgelet transform of $f$ and define $\rid_\psi f := \TT_\infty$.
\begin{thm}[Bounded Extension of Ridgelet Transform on $L^2$] \label{thm:L2}
Let $\psi \in \spSch(\RR)$ be self-admissible with $K_{\psi,\psi=1}$.
The ridgelet transform on $L^1 \cap L^2(\RR^m)$ admits a unique bounded extension to $L^2(\RR^m)$, with satisfying $\| \rid_\psi f \|_2 = \| f \|_2$.
\end{thm}}

\rsho{We say that $(\psi,\eta)$ and $(\psi^\star,\eta^\star)$ are {\em equivalent},
if two admissible pairs $(\psi,\eta)$ and $(\psi^\star,\eta^\star)$ define the same convolution $\overline{\widetilde{\psi}} * \eta = \overline{\widetilde{\psi^\star}} * \eta^\star$ in common. If $(\psi,\eta)$ and $(\psi^\star,\eta^\star)$ are equivalent, then obviously
\begin{align}
\ip{ \rid_\psi f, \rid_\eta g } = \ip{ \rid_{\psi^\star} f, \rid_{\eta^\star} g }.
\end{align}
We say that an admissible pair $(\psi,\eta)$ is {\em admissibly decomposable}, when there exist self-admissible pairs $(\psi^\star,\psi^\star)$ and $(\eta^\star,\eta^\star)$ such that 
$(\psi^\star,\eta^\star)$ is equivalent to $(\psi,\eta)$. If $(\psi,\eta)$ is admissibly decomposable with $(\psi^\star,\eta^\star)$, then by the Schwartz inequality
\begin{align}
\ip{ \rid_\psi f, \rid_\eta g } \leq \| \rid_{\psi^\star} f \|_2 \|\rid_{\eta^\star} g \|_2.
\end{align}
\begin{thm}[Reconstruction Formula in $L^2$] \label{thm:formula.L2}
Let $f \in L^2(\RR^m)$ and $(\psi, \eta) \in \spSch \times \spTD$ be admissibly decomposable with $K_{\psi,\eta} = 1$.
Then,
\begin{align}
\drid_\eta \rid_\psi f \to f, \quad \mathrm{in } \  L^2.
\end{align}
\end{thm}
The proof is provided in \refapp{app:proof.formula.L2}.
Even when $\psi$ is not self-admissible and thus $\rid_\psi$ cannot be defined on $L^2(\RR^m)$,
the reconstruction operator $\drid_\eta \rid_\psi$ can be defined with the aid of $\eta$.}

\section{Neural Network with Unbounded Activation Functions}

In this section we instantiate 
the universal approximation property for the variants of neural networks.
Recall that a neural network coincides with the dual ridgelet transform \rsho{of a function.
Henceforth, we rephrase a dual ridgelet function as an activation function.}
According to the reconstruction formulas (\refthm{thm:formula}, \ref{thm:formula.radon}, and \ref{thm:formula.L2}),
we can determine whether a neural network with an activation function $\eta$ \rtwo{is a} universal approximator
by checking the admissibility of $\eta$.

\rsho{\reftab{tab:acts} lists some Lizorkin \rtwo{distributions for} potential activation functions.
In \refsec{sec:ex.liz} we verify that they belong to $\spLizD(\RR)$ and some of them 
belong to $\spM(\RR)$ and $\spSch(\RR)$, which are subspaces of $\spLizD(\RR)$.
In \refsec{sec:ex.K} we show that they are admissible with some ridgelet function $\psi \in \spSch(\RR)$;
therefore, each of their corresponding neural networks is a universal approximator.}

\subsection{Examples of Lizorkin Distributions} \label{sec:ex.liz}

We proved the class properties by using the following propositions.
\begin{prop}[Tempered Distribution $\spTD(\RR)$ {\cite[Ex. 2.3.5]{Grafakos.classic}}]
Let $g \in \sploc{1}(\RR)$. If $|g(z)| \lesssim (1+|z|)^k$ for some $k \in \NN_0$,
then $g \in \spTD(\RR)$.
\end{prop}

\begin{prop}[Slowly Increasing Function $\spM(\RR)$ {\cite[Def. 2.3.15]{Grafakos.classic}}] \label{prop:spM}
Let $g \in \spsmooth(\RR)$. 
If for any $\alpha \in \NN_0, \, |\partial^\alpha g(x)| \lesssim (1+|z|)^{k_\alpha} $ for some $k_\alpha \in \NN_0$,
then $g \in \spM(\RR)$.
\end{prop}

\begin{eg}
Truncated power functions $z_+^k \, (k \in \NN_0)$, which contain the ReLU $z_+$ and the step function $z_+^0$, belong to $\spLizD(\RR)$.
\end{eg}
\begin{proof}
For any $\ell \in \NN_0$ there exists a constant $C_\ell$ such that $|\partial^\ell (z_+^k)| \leq C_\ell (1+|z|)^{k-\ell}$. Hence, $z_+^k \in \spLizD(\RR)$.
\end{proof}

\begin{eg} \label{eg:sig}
The sigmoidal function $\sig(z)$ and the softplus $\sps(z)$ belong to $\spM(\RR)$.
The derivatives $\sig^{(k)}(z) \, (k \in \NN)$ belong to $\spSch(\RR)$.
Hyperbolic tangent $\tanh(z)$ belongs to $\spM(\RR)$.
\end{eg}
The proof is provided in \refapp{app:proof.sig}

\begin{eg}[{\cite[Ex.2.2.2]{Grafakos.classic}}]
RBF $\rbf(z)$ and their derivatives $\rbf^{(k)}(z)$ belong to $\spSch(\RR)$.
\end{eg}

\begin{eg}[{\cite[Ex.2.3.5]{Grafakos.classic}}]
Dirac's $\delta(z)$ and their derivatives $\delta^{(k)}(z)$ belong to $\spTD(\RR)$.
\end{eg}

\subsection{$K_{\psi,\eta}$ when $\psi$ is a derivative of the Gaussian} \label{sec:ex.K}

Given \rsho{an activation function} $\eta \in \spLizD(\RR)$,
according to \refcor{cor:const.ap}
we can construct an admissible ridgelet function $\psi \in \spSch(\RR)$
by letting
\begin{gather}
\psi := \bp^m \psi_0,
\end{gather}
where $\psi_0 \in \spSch(\RR)$ satisfies
\begin{gather}
\bigdual{\widehat{\eta},\widehat{\psi_0}} := \int_{\RR \setminus \{ 0 \}} \overline{\widehat{\psi_0}(\zeta)} \widehat{\eta}(\zeta) \dd \zeta  \neq 0, \, \pm \infty.
\end{gather}

Here we consider the case when $\psi_0$ is given by
\begin{align}
\psi_0 = \gauss^{(\ell)}, \quad 
\end{align}
for some $\ell \in \NN_0$, where $\gauss$ denotes the Gaussian $\gauss(z) := \exp(-z^2/2)$.

The Fourier transform of the Gaussian is given by $\widehat{\gauss}(\fz) = \exp(-\fz^2/2) = \sqrt{2 \pi} \, \gauss(\zeta)$.
The Hilbert transform of the Gaussian, which we encounter by computing $\psi = \bp^m \gauss$ when $m$ is odd, is given by
\begin{gather}
\hil \gauss(z) = \frac{2 i}{\sqrt{\pi}} \dawson \left( \frac{z}{\sqrt{2}} \right),
\end{gather}
where $\dawson(z)$ is the Dawson function $\dawson(z) := \exp ( -z^2 ) \int_0^z \exp (w^2) \dd w$.

\begin{eg}
$z_+^k \ (k \in \NN_0)$ is admissible with $\psi = \bp^m \gauss^{(\ell+k+1)} \ (\ell \in \NN_0)$ \rtwo{iff} $\ell$ is even. If odd, then $K_{\psi,\eta} = 0$.
\end{eg}
\begin{proof}
It follows from the fact that, according to Gel'fand and Shilov \cite[\S~9.3]{GelfandShilov}, 
\begin{align}
\widehat{ z_+^k }(\fz) = \frac{k!}{(i \zeta)^{k+1}} + \pi i^k \delta^{(k)}(\zeta), \quad k \in \NN_0.  &\qedhere
\end{align}
\end{proof}

\begin{eg}$\eta(z) = \delta^{(k)}(z)  \ (k \in \NN_0)$ is admissible with $\psi = \bp^{m} \gauss$ \rtwo{iff} $k$ is even. If odd, then $K_{\psi,\eta} = 0$.
\end{eg}
In contrast to polynomial functions, Dirac's $\delta$ can be an admissible activation function.

\begin{eg}
$\eta(z) = \rbf^{(k)}(z) \ (k \in \NN_0)$ is admissible with $\psi = \bp^m \gauss$ \rtwo{iff} $k$ is even. If odd, then $K_{\psi,\eta} = 0$.
\end{eg}

\begin{eg} \label{eg:adm.sig}
$\eta(z) = \sig^{(k)}(z) \ (k \in \NN_0)$ is admissible with $\psi = \bp^m \gauss$ \rtwo{iff} $k$ is odd. If odd, then $K_{\psi,\eta} = 0$.
$\sig^{(-1)}$ is admissible with $\psi=\bp^m \gauss''$.
\end{eg}
The proof is provided in \refapp{app:proof.sig}.

\section{Numerical Examples of Reconstruction}
We performed some numerical experiments on reconstructing a one-dimensional signal and a two-dimensional image,
with reference to our theoretical diagnoses for admissibility in the previous section.
\reftab{tab:admissible} lists the diagnoses of $(\bp^m\psi_0, \eta)$ we employ in this section.
\rsho{The symbols '$+$,' '$0$,' and '$\infty$' in each cell indicate that $K_{\psi,\eta}$ of the corresponding $(\psi, \eta)$ converges to a non-zero constant ($+$), converges to zero ($0$), and diverges ($\infty$). Hence, by \refthm{thm:formula}, if the cell $(\psi,\eta)$ indicates '$+$' then a neural network with an activation function $\eta$ is a universal approximator.}
\begin{table}[h!]
  \centering
    \caption{\rsho{Theoretical diagnoses for admissibility of $\psi = \bp^m \psi_0$ and $\eta$.
        '$+$' indicates that $(\psi,\eta)$ is admissible.
        '$0$' and '$\infty$' indicate that $K_{\psi,\eta}$ vanishes and diverges, respectively,
        and thus $(\psi,\eta)$ is not admissible.}}
  \begin{tabular}{llccc}
  \toprule
  activation function & $\eta$ & $ \psi = \bp^m \gauss $ & $ \psi = \bp^m \gauss' $ & $ \psi = \bp^m \gauss'' $ \\
  \midrule
  derivative of sigmoidal ft. & $\sig'$ & $+$ & $\rsho{0}$ & $+$ \\
  sigmoidal function & $\sig$ & $\rsho{\infty}$ & $+$ & $\rsho{0}$ \\
  softplus & $\sps$ & $\rsho{\infty}$ & $\rsho{\infty}$ & $+$ \\
  \midrule
  Dirac's $\delta$ & $\delta$ & $+$ & $\rsho{0}$ & $+$ \\
  unit step function & $z_+^0$ & $\rsho{\infty}$ & $+$ & $\rsho{0}$ \\
  ReLU & $z_+$ & $\rsho{\infty}$ & $\rsho{\infty}$ & $+$ \\
  \midrule
  linear function & $z$ & $\rsho{0}$ & $\rsho{0}$ & $\rsho{0}$ \\
  \midrule
  RBF & $\rbf$ & $+$ & $\rsho{0}$ & $+$ \\
  \bottomrule
  \end{tabular} \label{tab:admissible}
\end{table}

\subsection{Sinusoidal Curve}
We studied a one-dimensional signal $f(x) = \sin 2 \pi x$ defined on $x \in [-1,1]$.
The ridgelet functions \rsho{functions} $\psi = \bp \psi_0$ were chosen from derivatives of the Gaussian $\psi_0 = \gauss^{(\ell)}, \ (\ell=0,1,2)$.
The activation \rsho{functions} $\eta$ were chosen from among
the softplus $\sps$, the sigmoidal function $\sig$ and its derivative $\sig'$, 
the ReLU $z_+$, unit step function $z_+^0$, and Dirac's $\delta$.
In addition, we examined the case when the activation function is simply a linear function: $\eta(z) = z$,
which cannot be admissible because the Fourier transform of polynomials is supported at the origin in the Fourier domain.

The signal was sampled from $[-1, 1]$ with $\Delta x = 1/100$.
We computed the reconstruction formula
\begin{align}
\int_{\RR} \int_\RR \rid_\psi f(a,b) \eta(ax-b) \frac{\dd a \dd b}{\rsho{|a|}},
\end{align}
by simply discretizing \rsho{$(a,b) \in [ -30, 30 ] \times [-30, 30]$} by $\Delta a = \Delta b = 1/10$.
\rsho{That is,
\begin{align}
&\rid_\psi f(a,b) \approx \sum_{n=0}^N f(x_n) \overline{ \psi ( a \cdot x_n - b ) } |a| \Delta x, \quad x_n = x_0 + n \Delta x \\
&\drid_\eta \rid f(x) \approx \sum_{(i,j)=(0,0)}^{I,J} \rid_\psi f(a_i, b_j) \eta ( a_i \cdot x - b_j ) \frac{\Delta a \Delta b}{|a_i|}, \quad a_i = a_0 + i \Delta a, \ b_j = b_0 + j \Delta b
\end{align}
where $x_0 = -1, \ a_0 = -30, \ b_0 = -30$, and $N=200, (I,J)=(600,600)$.}

\begin{figure}[h]
\centering
\begin{tabular}{cccc}
  $ \psi = \bp \gauss $ & $ \psi = \bp \gauss' $ & $ \psi = \bp \gauss'' $ & \\
  \raisebox{-.5\height}{\includegraphics[width=.3\linewidth]{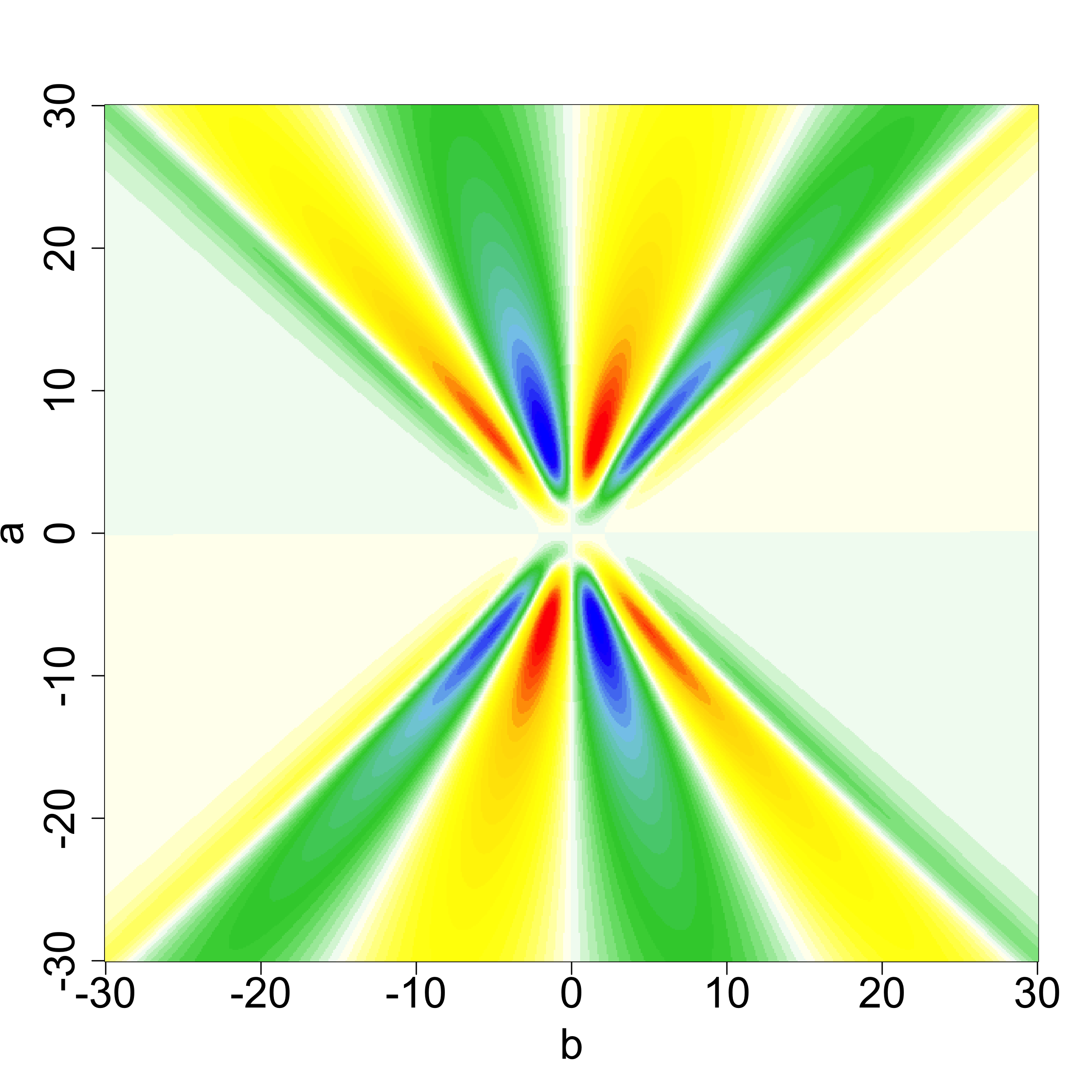}} &
  \raisebox{-.5\height}{\includegraphics[width=.3\linewidth]{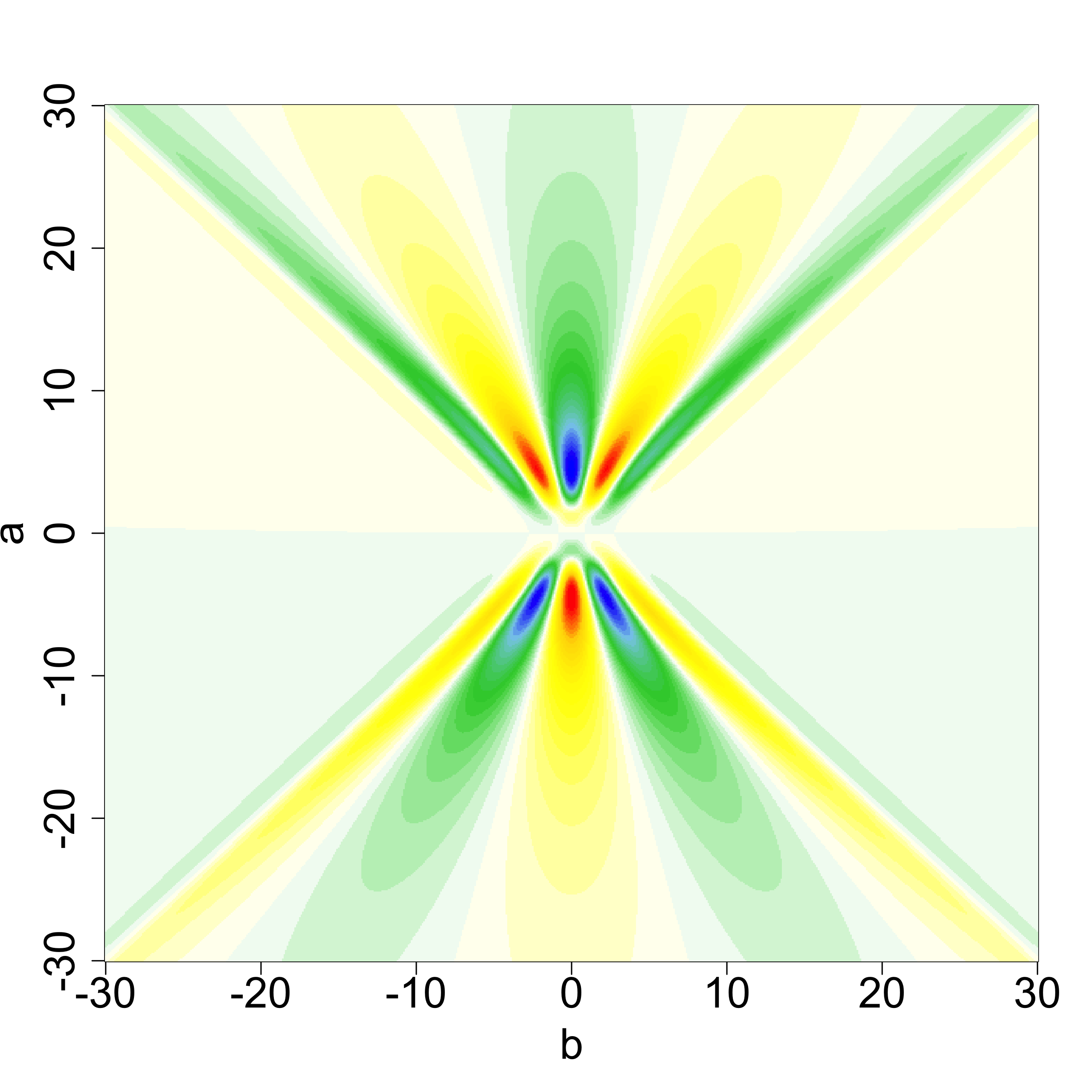}} &
  \raisebox{-.5\height}{\includegraphics[width=.3\linewidth]{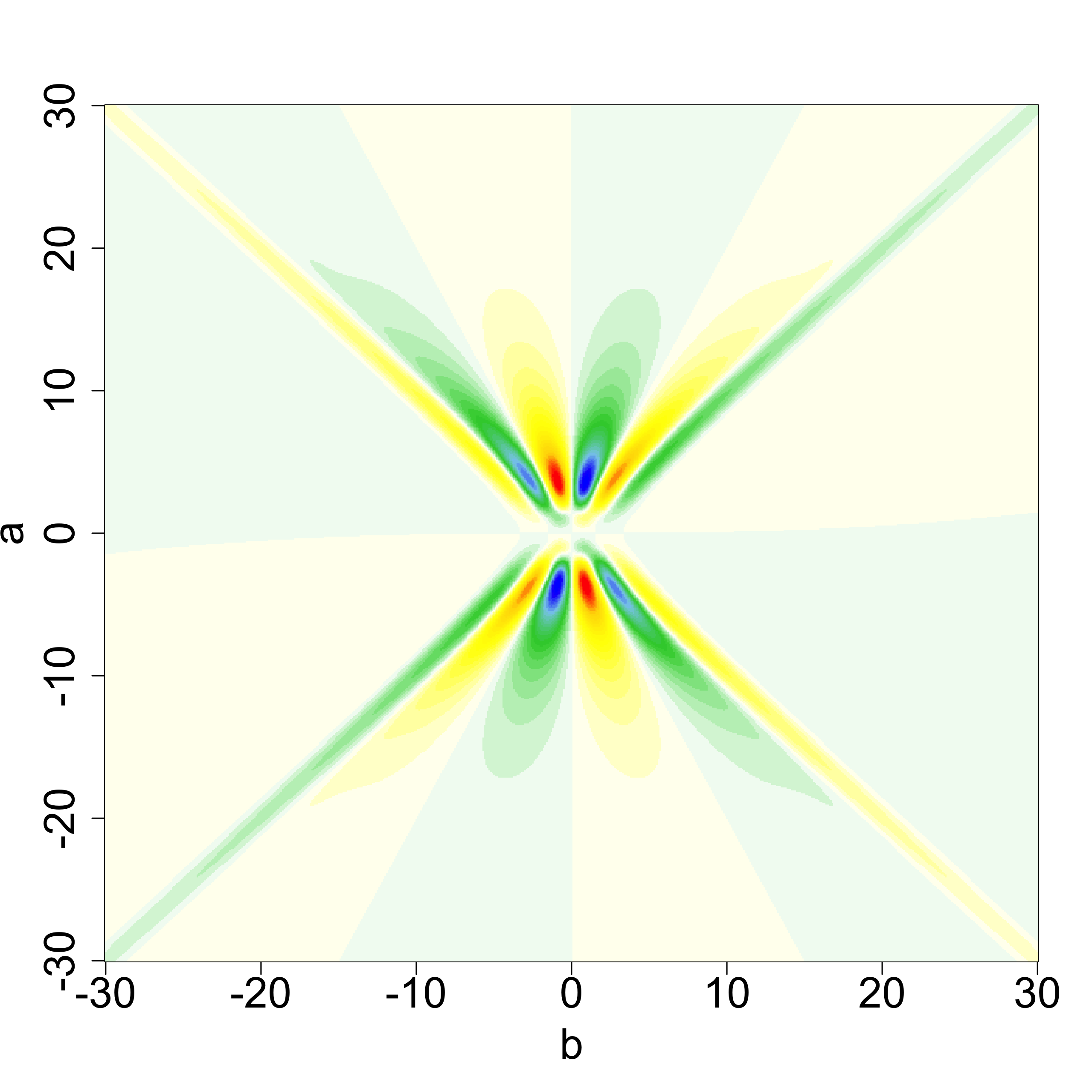}} &
  \raisebox{-.5\height}{\includegraphics[width=.016\linewidth]{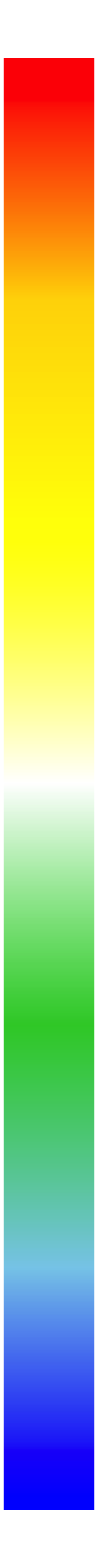}} \\
\end{tabular}
  \caption{Ridgelet transform $\rid_\psi f(a,b)$ of $f(x) = \sin 2 \pi x$ defined on $[-1,1]$ with respect to $\psi$.} \label{fig:sin_Tab}
\end{figure}

\begin{figure}[h]
\centering
\begin{tabular}{cccc}
  & $ \psi = \bp \gauss $ & $ \psi = \bp \gauss' $ & $ \psi = \bp \gauss'' $ \\
  \rotatebox{90}{$\eta=\sig'$} &
  \raisebox{-.5\height}{\includegraphics[width=.27\textwidth]{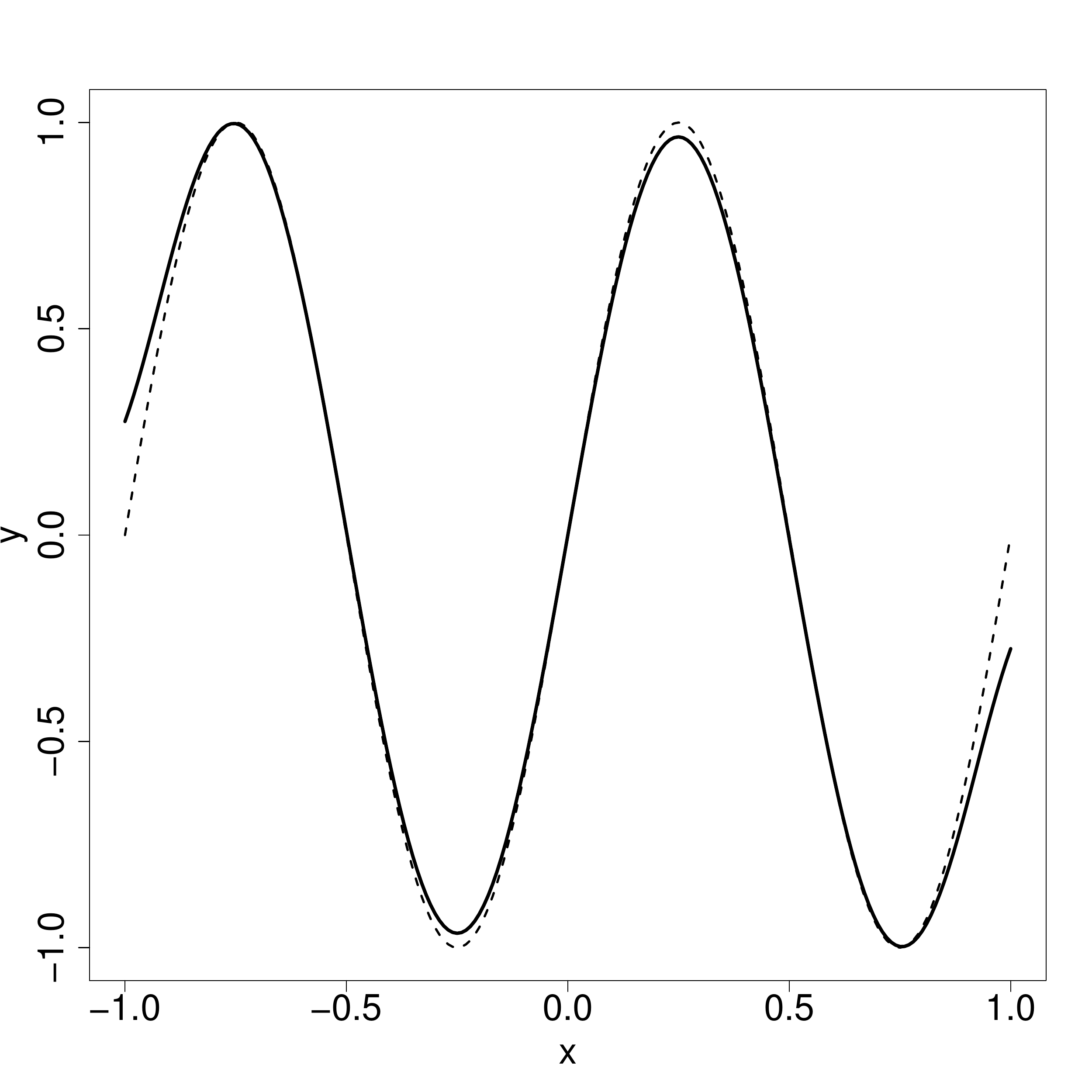}} &
  \raisebox{-.5\height}{\includegraphics[width=.27\textwidth]{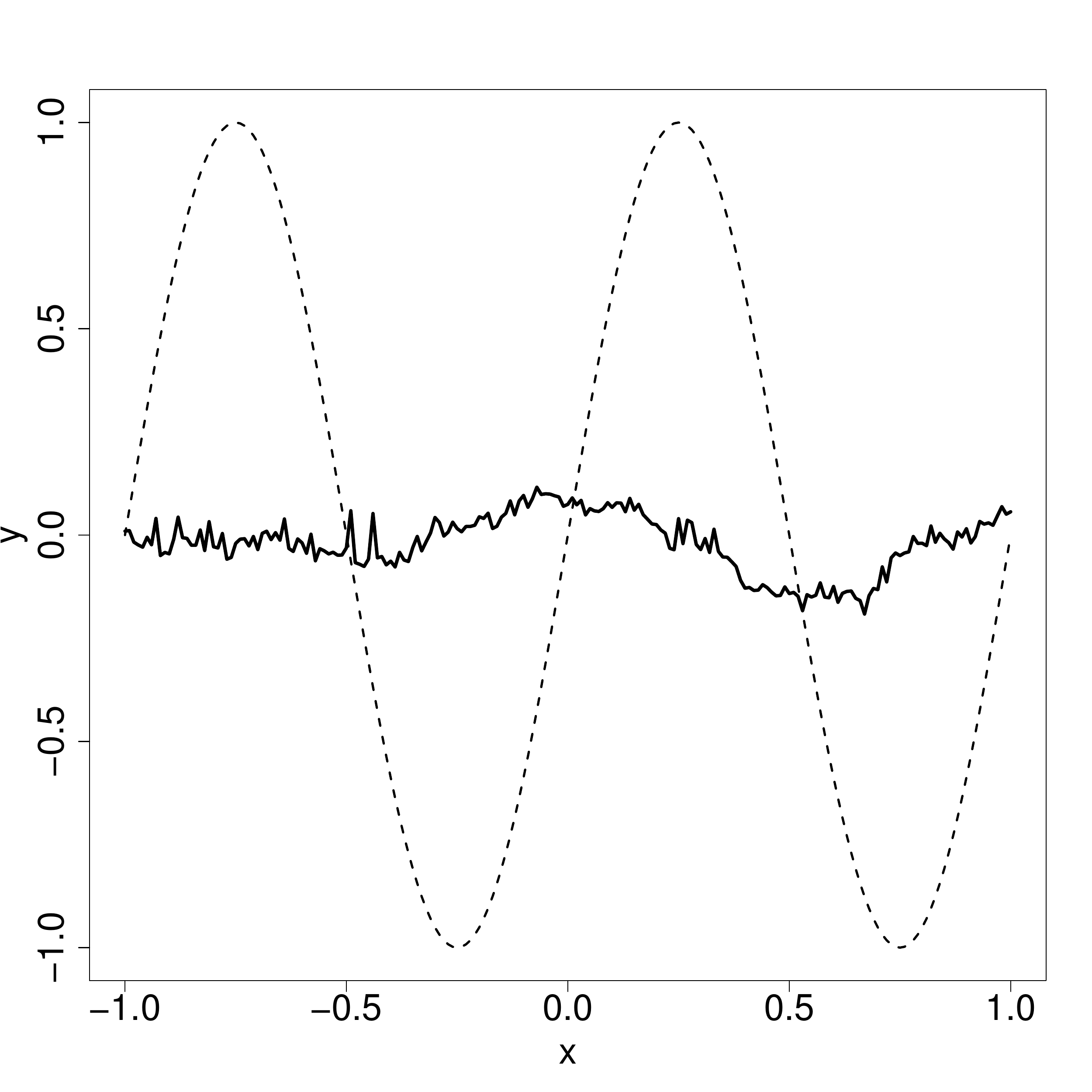}} &
  \raisebox{-.5\height}{\includegraphics[width=.27\textwidth]{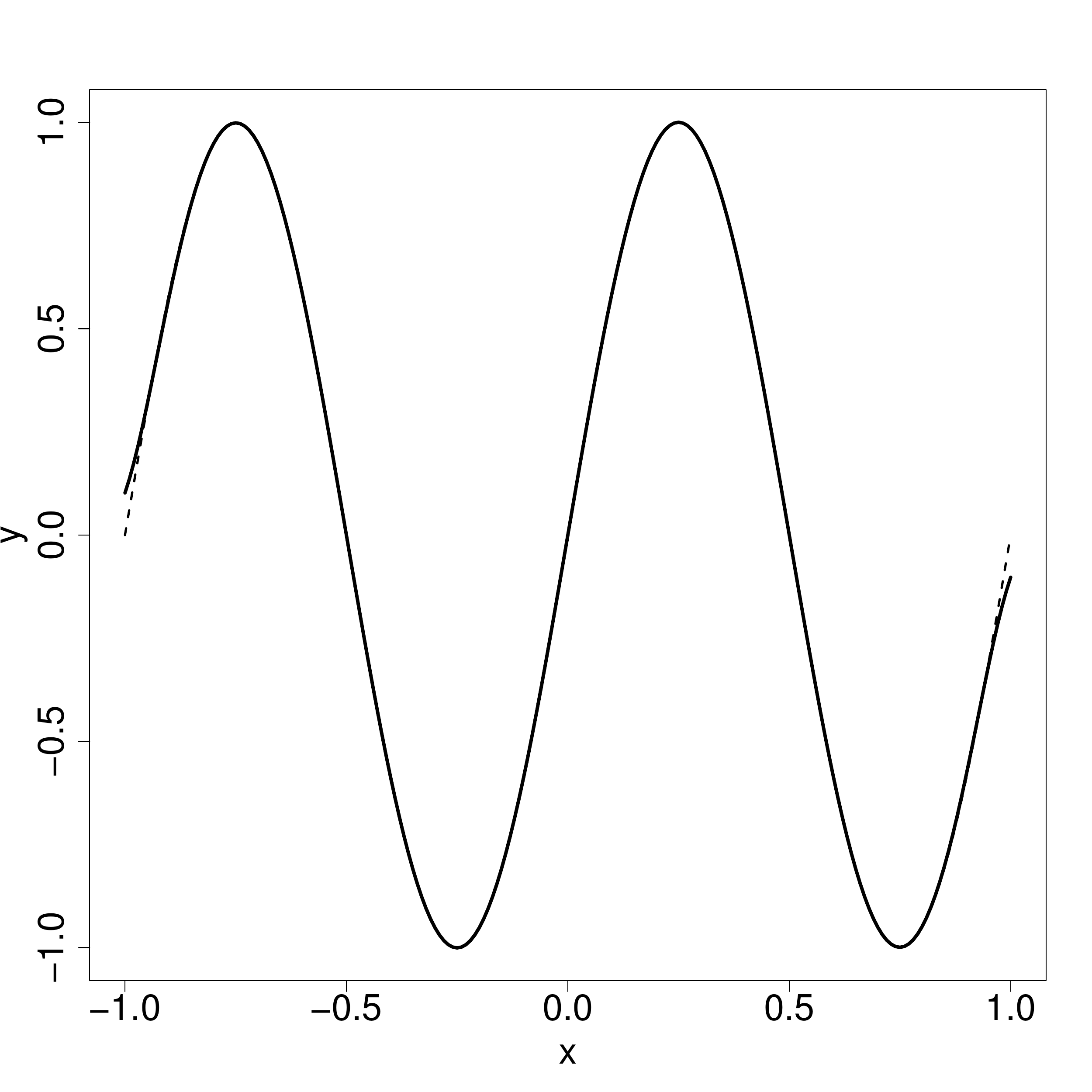}} \\
  \rotatebox{90}{$\eta=\sig$} &
  \raisebox{-.5\height}{\includegraphics[width=.27\textwidth]{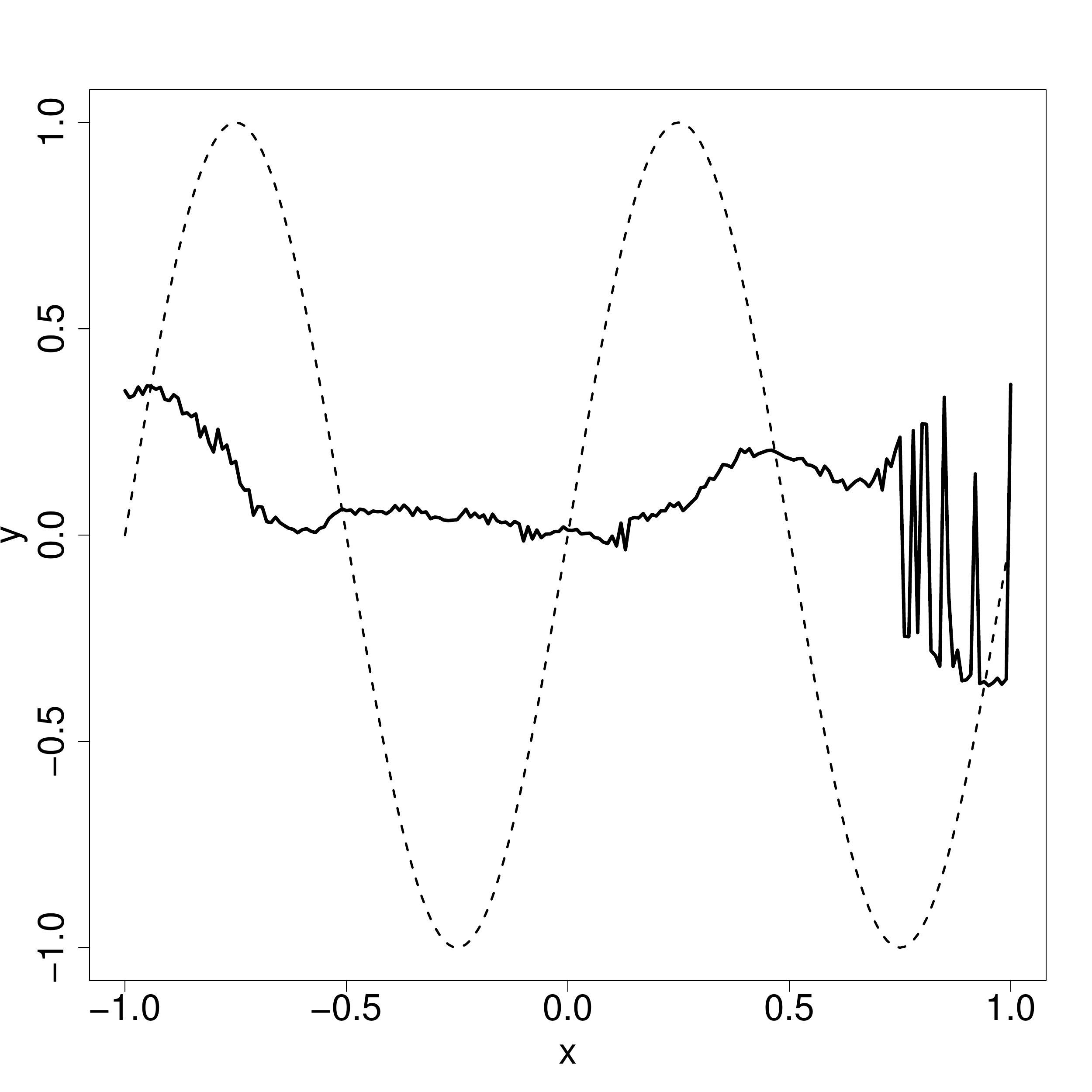}} &
  \raisebox{-.5\height}{\includegraphics[width=.27\textwidth]{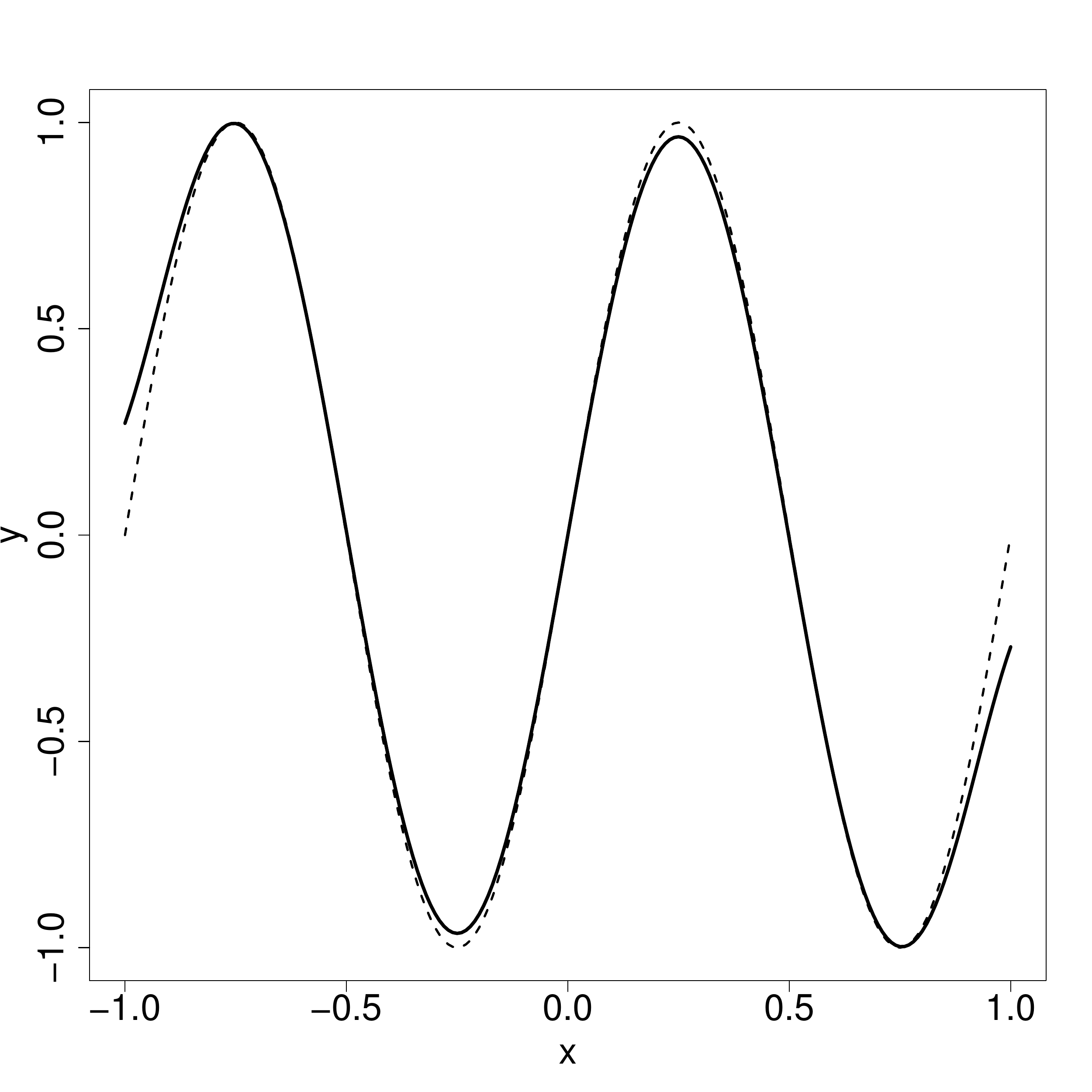}} &
  \raisebox{-.5\height}{\includegraphics[width=.27\textwidth]{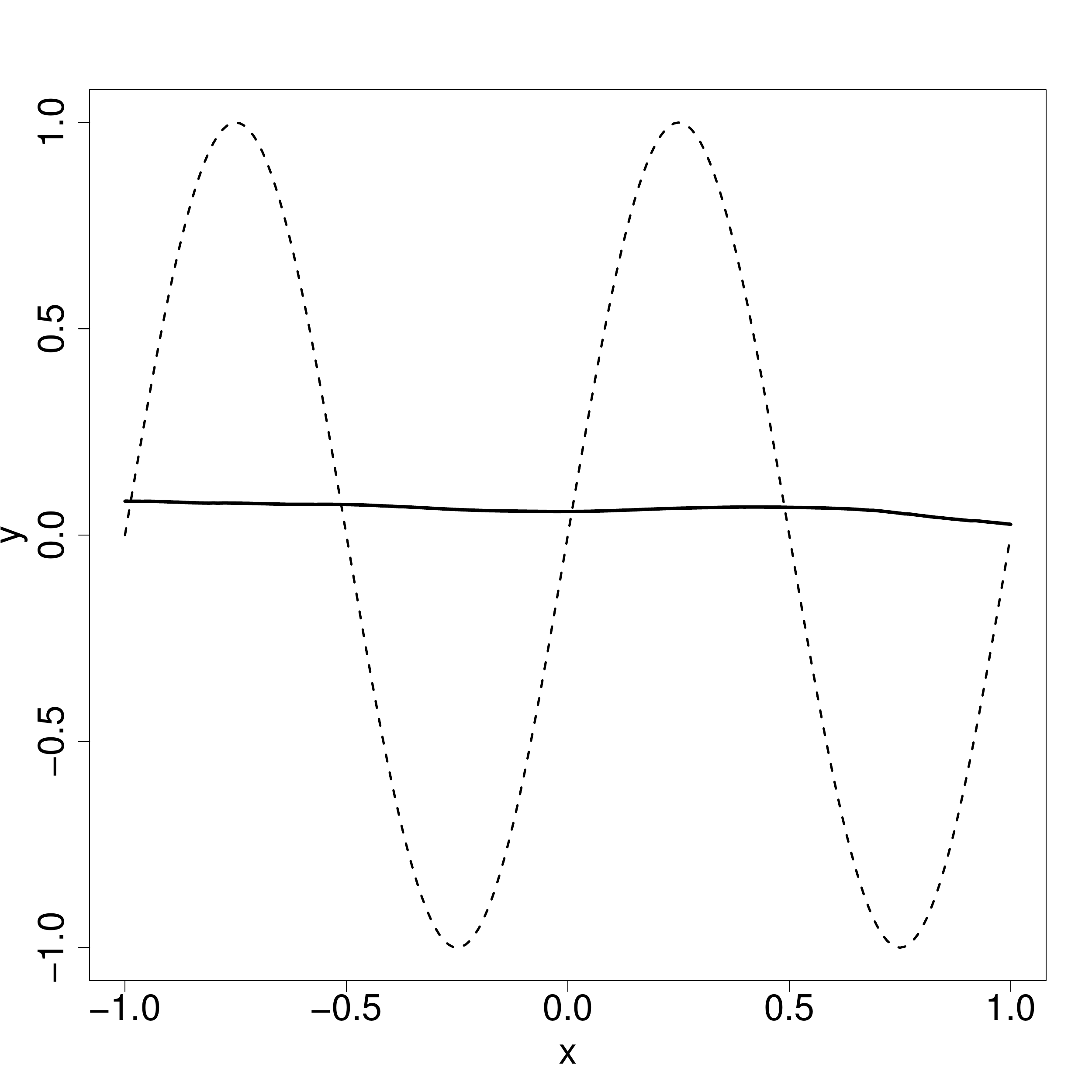}} \\
  \rotatebox{90}{$\eta=\sps$} &
  \raisebox{-.5\height}{\includegraphics[width=.27\textwidth]{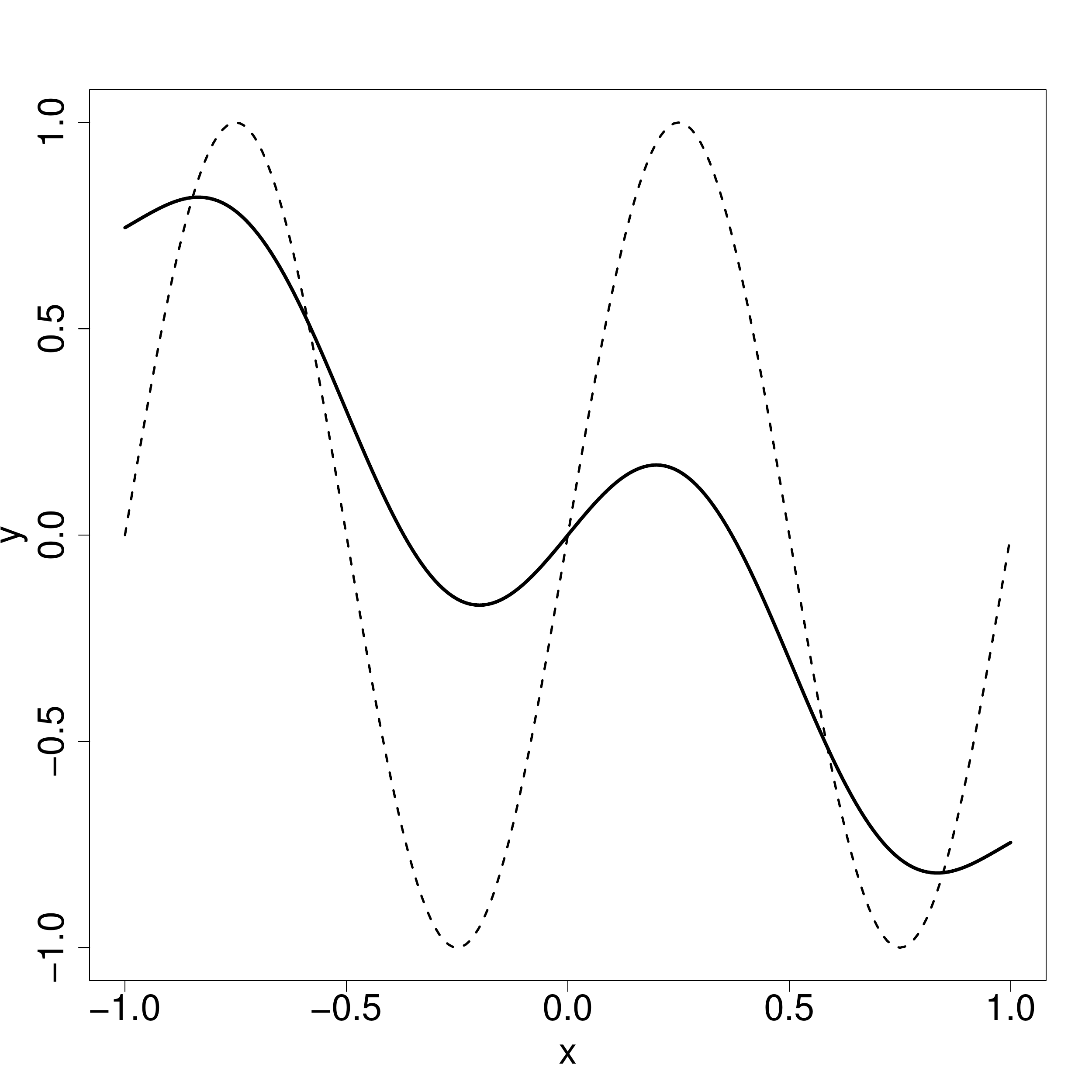}} &
  \raisebox{-.5\height}{\includegraphics[width=.27\textwidth]{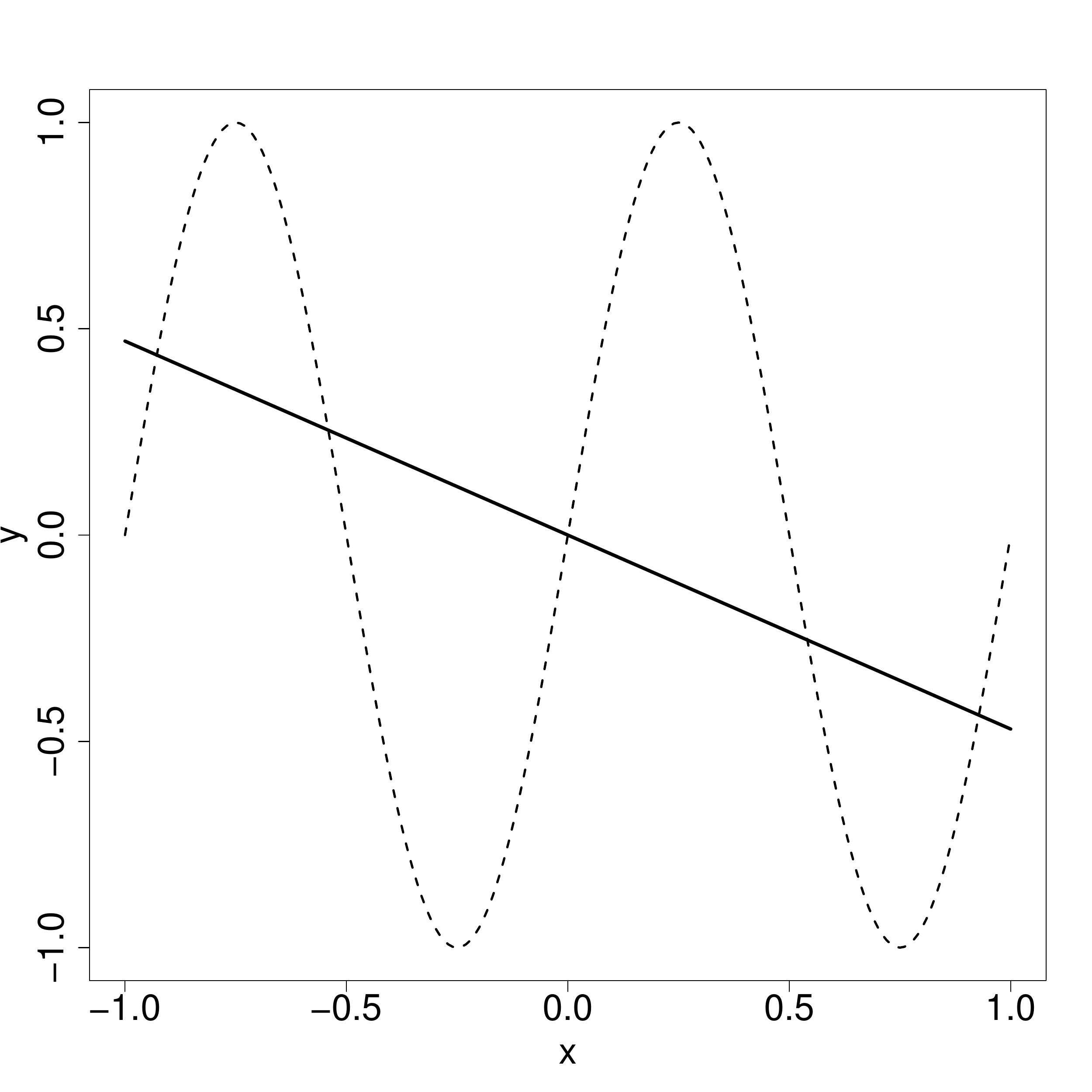}} &
  \raisebox{-.5\height}{\includegraphics[width=.27\textwidth]{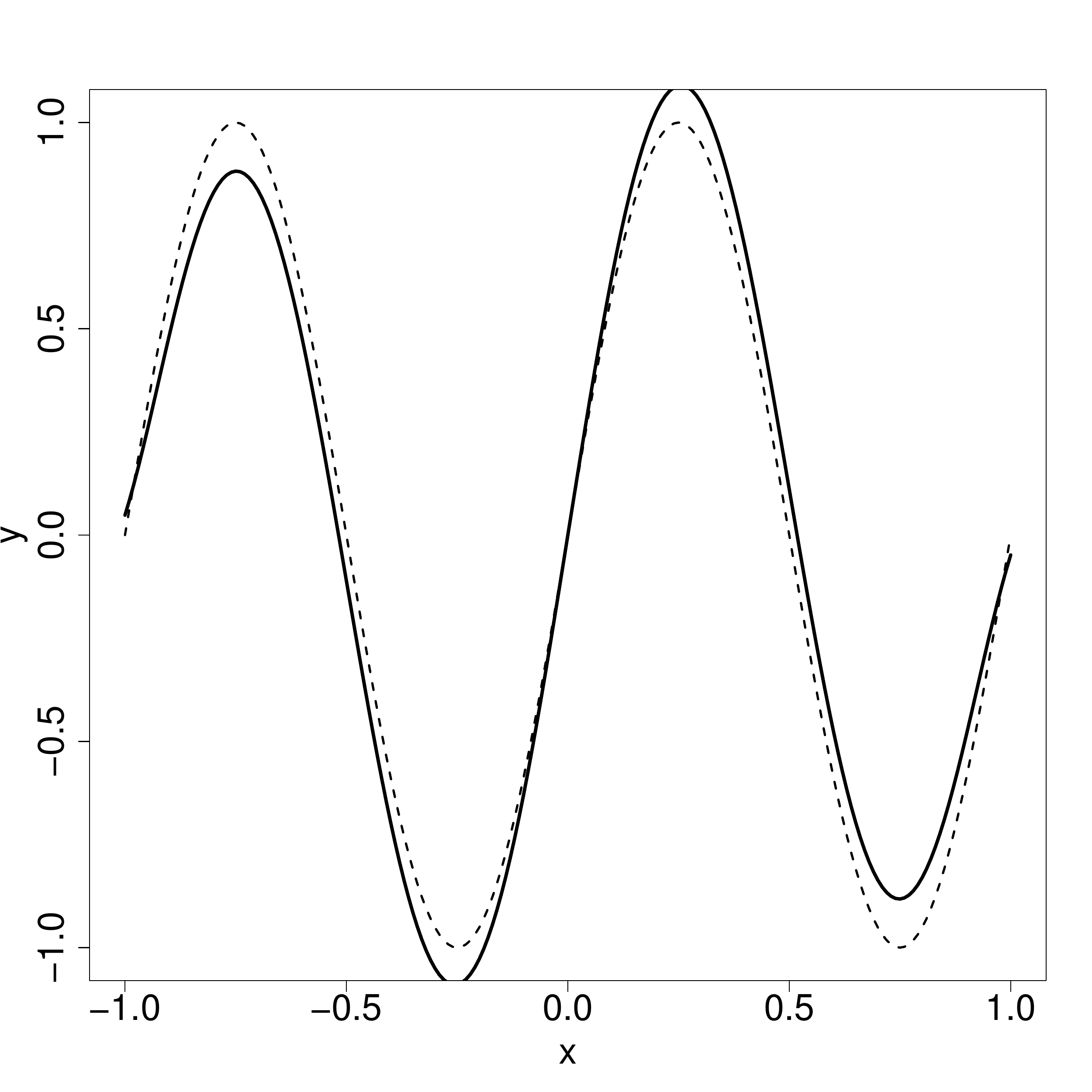}} \\
\end{tabular}
  \caption{Reconstruction with the derivative of sigmoidal function $\sig'$, sigmoidal function $\sig$, and softplus $\sps$.
The solid line is a plot of the reconstruction result; the dotted line plots the original signal.
	}
  \label{fig:sin_sigs}
\end{figure}

\begin{figure}[h!] 
\centering
\begin{tabular}{cccc}
  & $ \psi = \bp \gauss $ & $ \psi = \bp \gauss' $ & $ \psi = \bp \gauss'' $ \\
  \rotatebox{90}{$\eta = \delta$} &
  \raisebox{-.5\height}{\includegraphics[width=.27\textwidth]{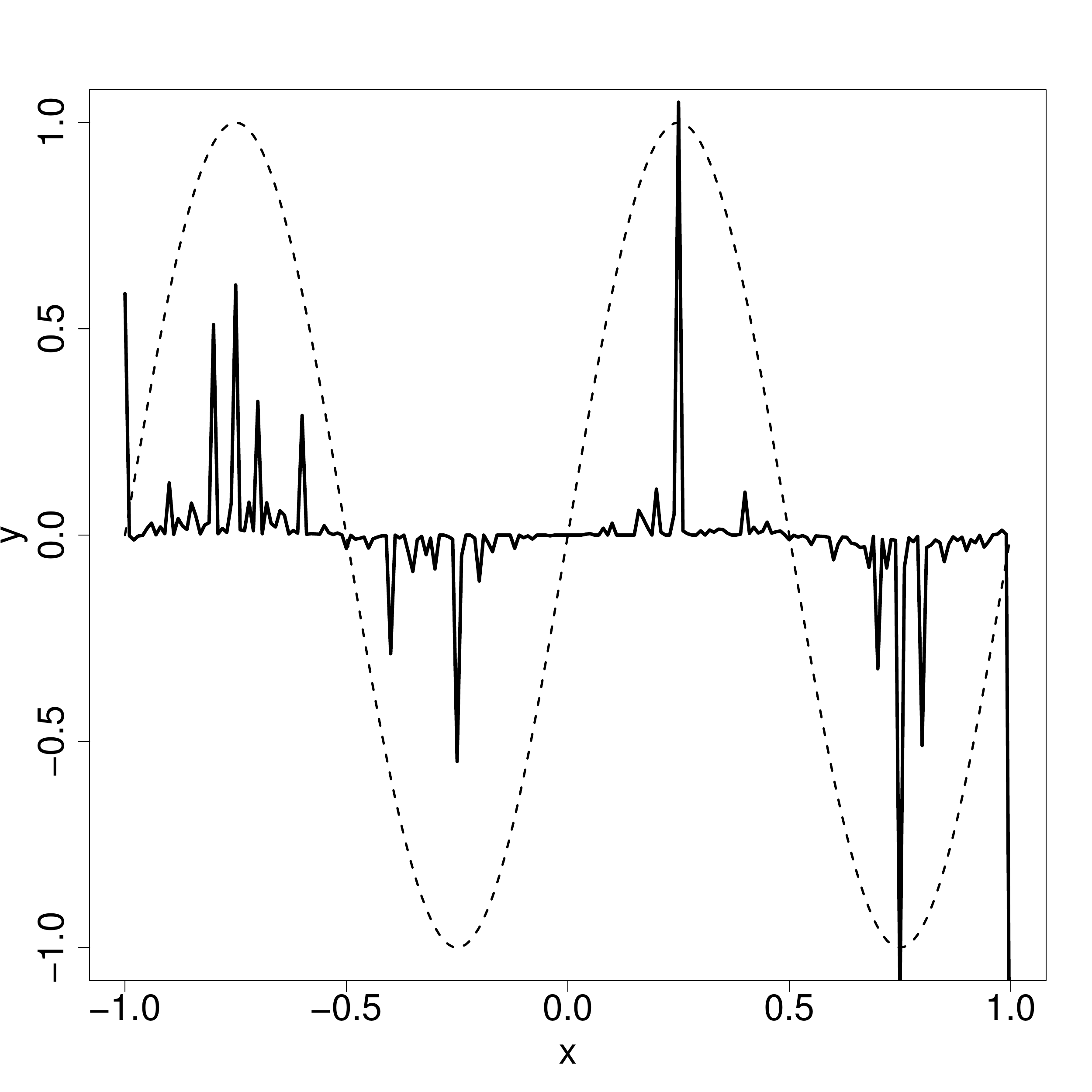}} &
  \raisebox{-.5\height}{\includegraphics[width=.27\textwidth]{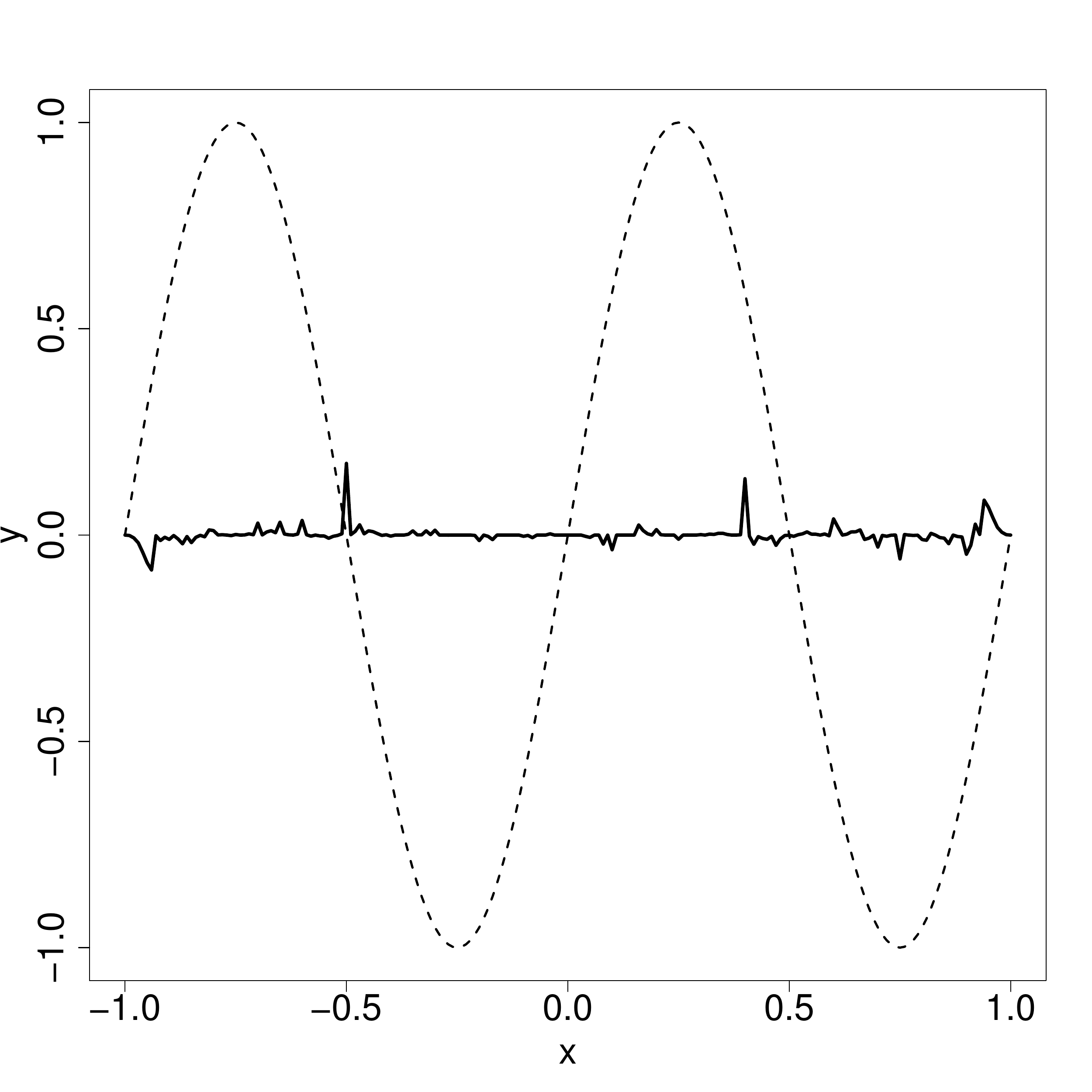}} &
  \raisebox{-.5\height}{\includegraphics[width=.27\textwidth]{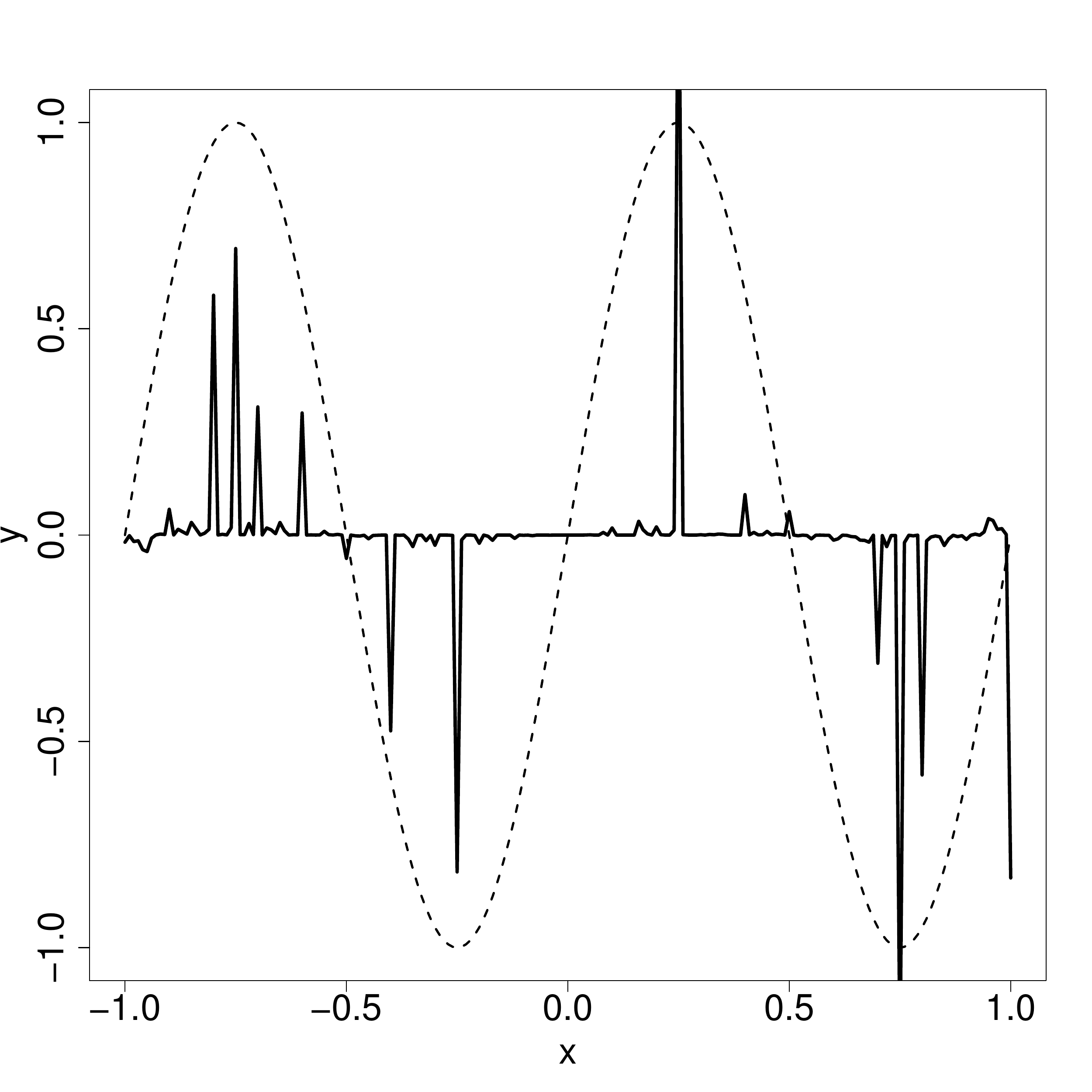}} \\
  \rotatebox{90}{$\eta=z_+^0$} &
  \raisebox{-.5\height}{\includegraphics[width=.27\textwidth]{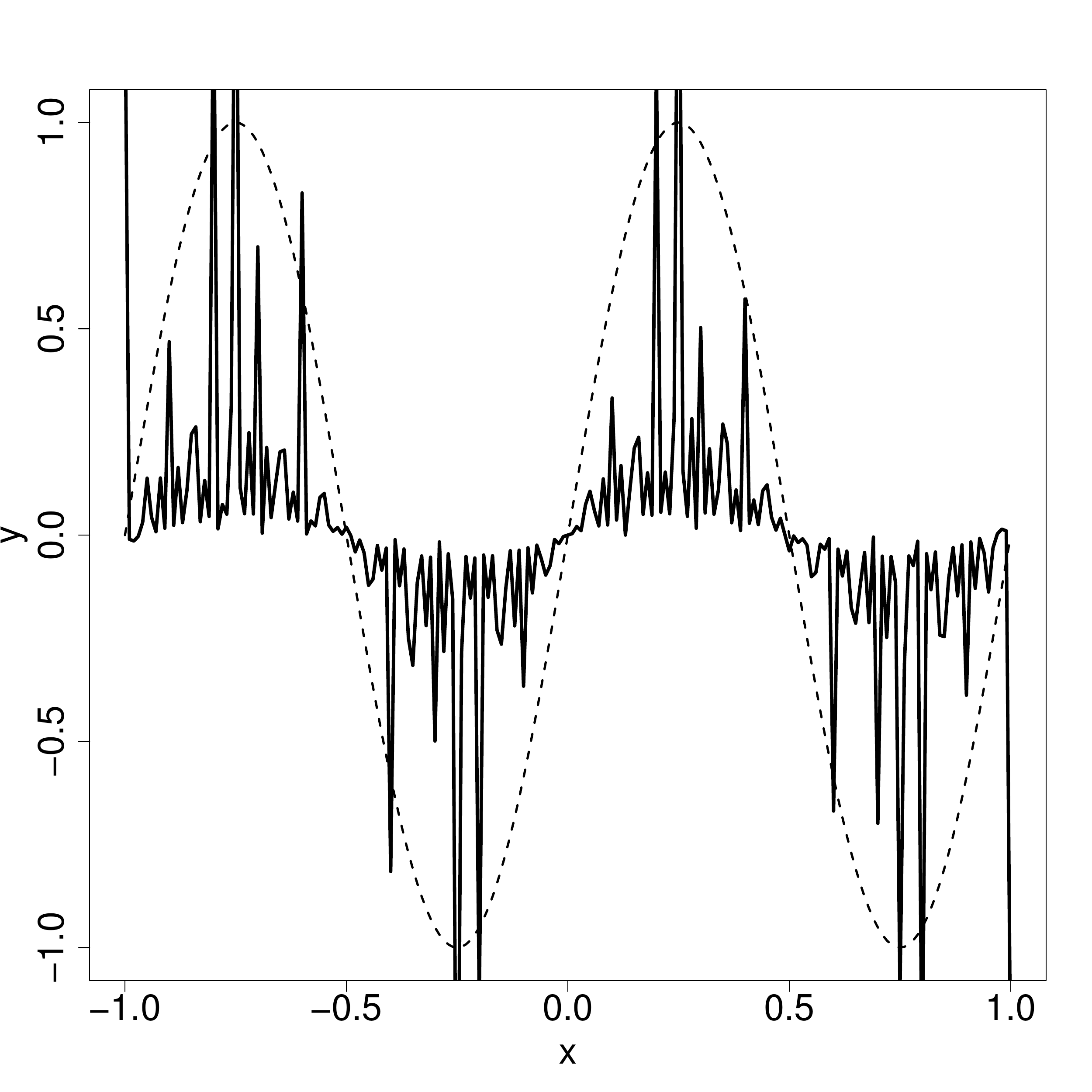}} &
  \raisebox{-.5\height}{\includegraphics[width=.27\textwidth]{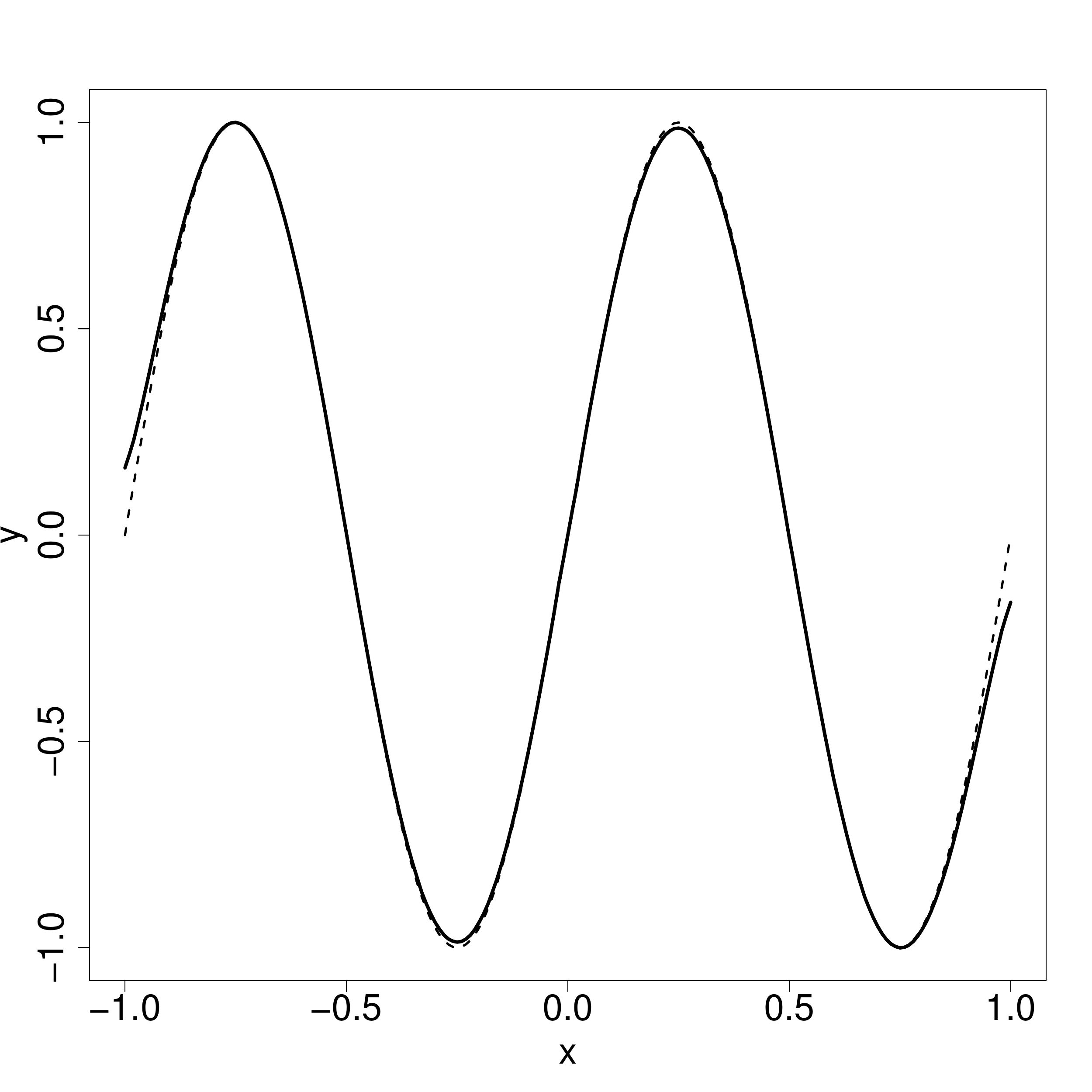}} &
  \raisebox{-.5\height}{\includegraphics[width=.27\textwidth]{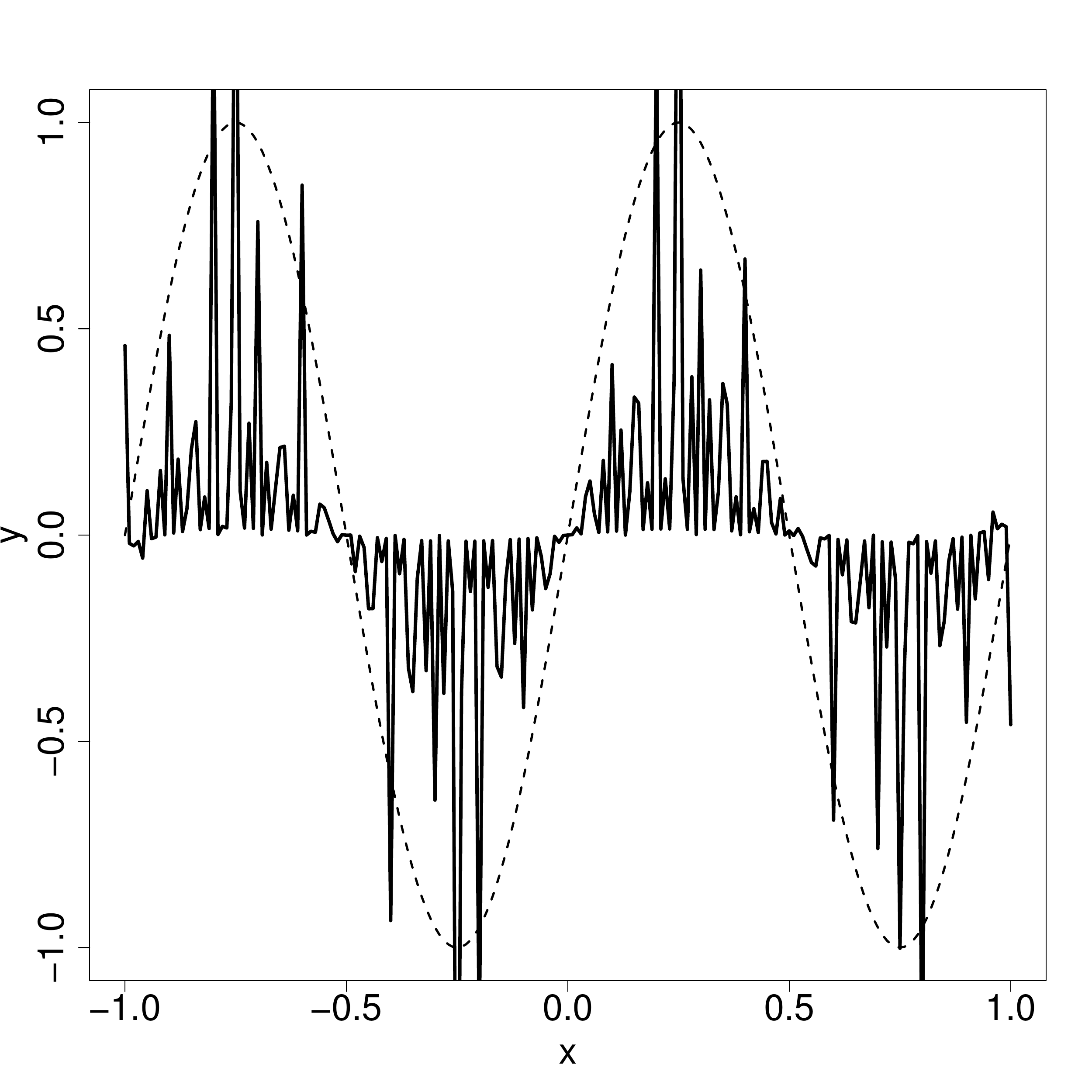}} \\
  \rotatebox{90}{$\eta=z_+$} &
  \raisebox{-.5\height}{\includegraphics[width=.27\textwidth]{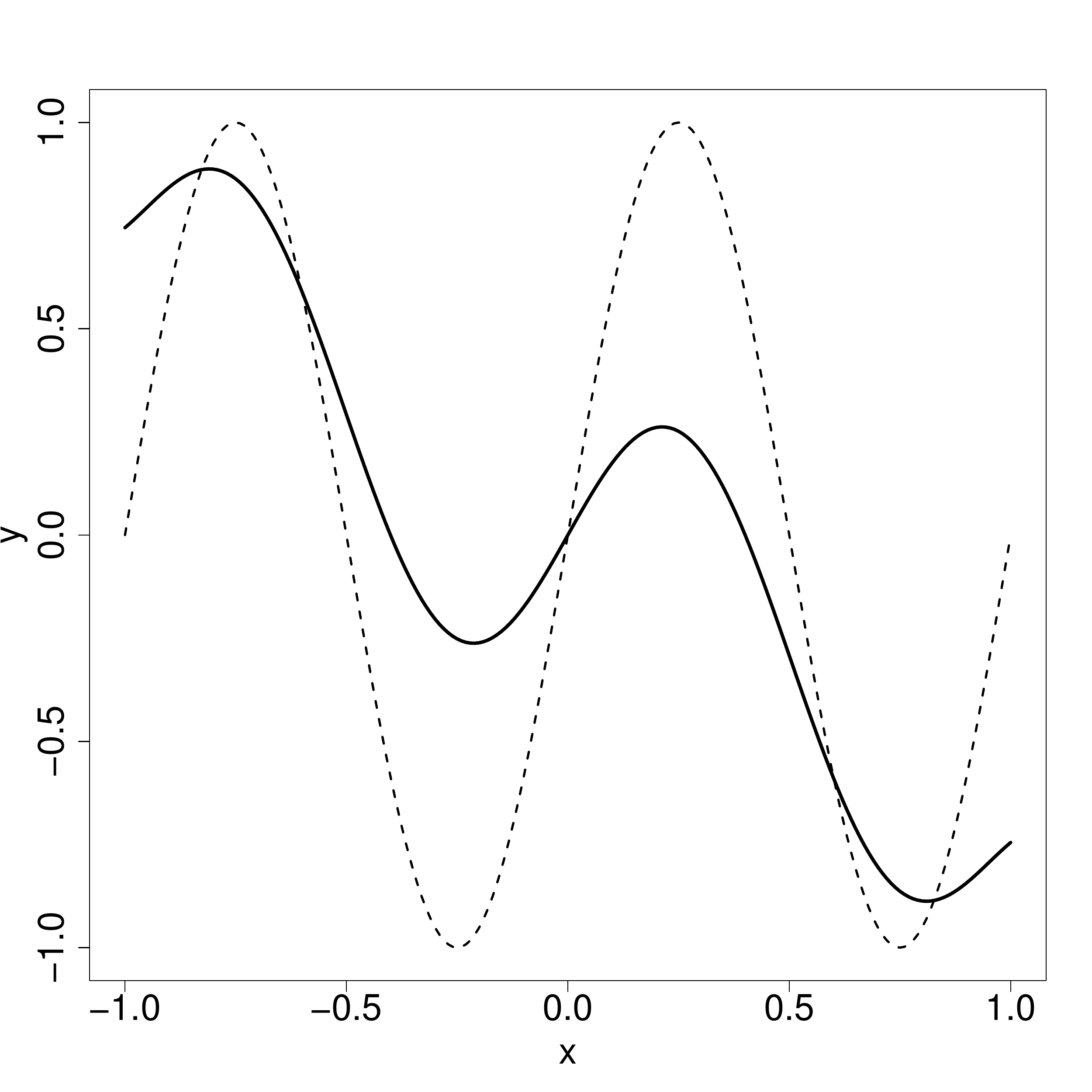}} &
  \raisebox{-.5\height}{\includegraphics[width=.27\textwidth]{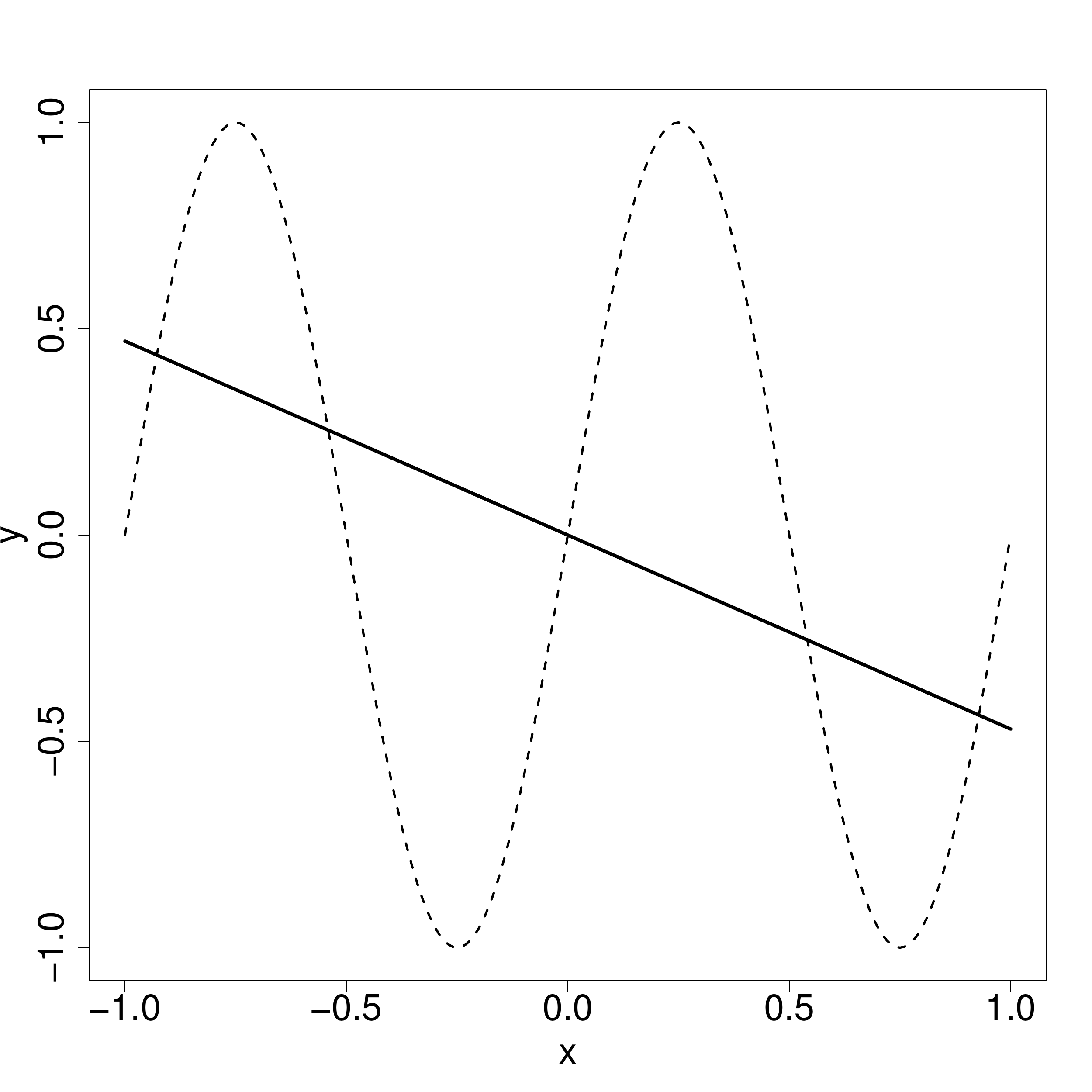}} &
  \raisebox{-.5\height}{\includegraphics[width=.27\textwidth]{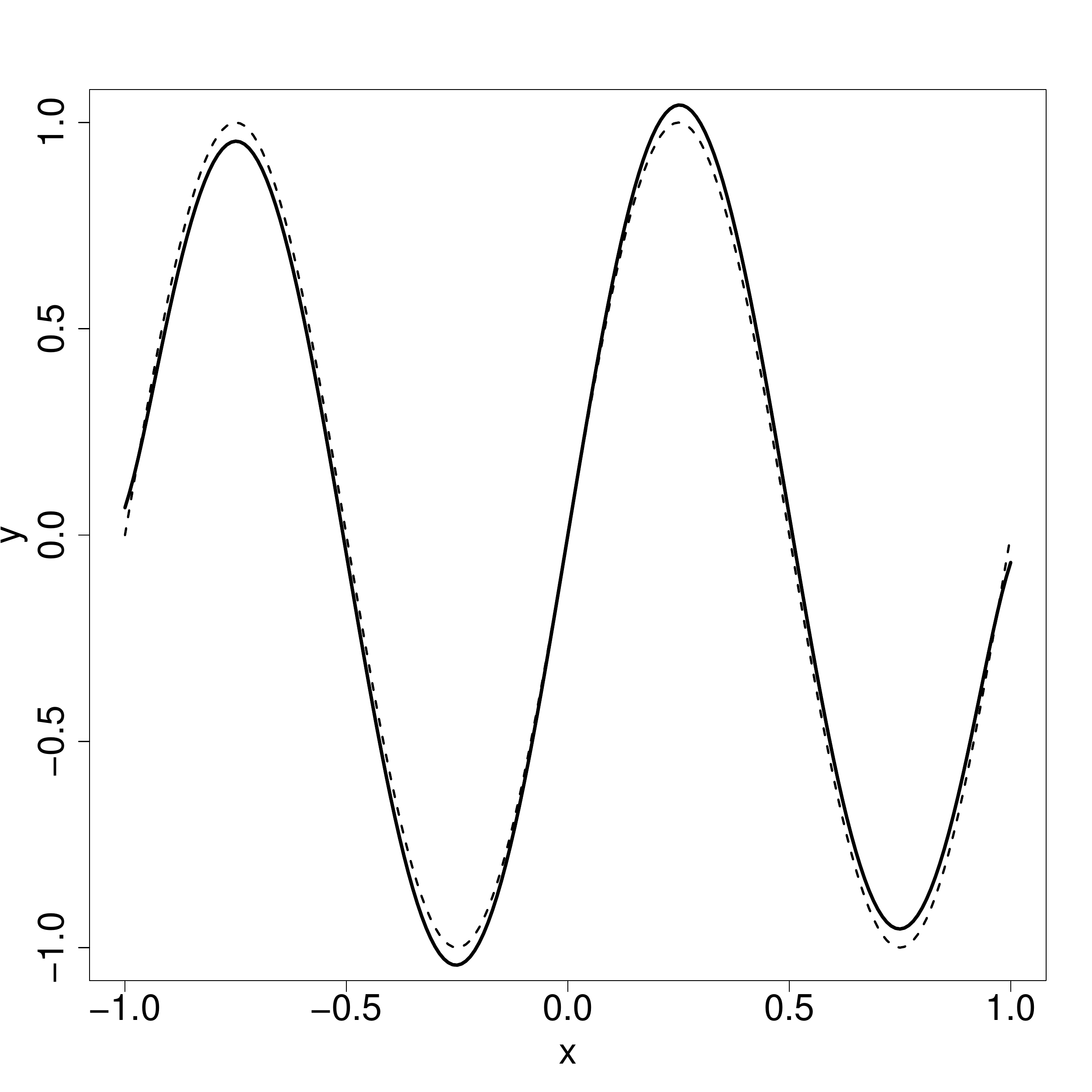}} \\
\end{tabular}
  \caption{Reconstruction with truncated power functions --- Dirac's $\delta$, unit step $z_+^0$, and ReLU $z_+$.
The solid line is a plot of the reconstruction result; the dotted line plots the original signal.
}
  \label{fig:sin_truncs}
\end{figure}
\begin{figure}[h!] 
\centering
\begin{tabular}{cccc}
  & $ \psi = \bp \gauss $ & $ \psi = \bp \gauss' $ & $ \psi = \bp \gauss'' $ \\
  \rotatebox{90}{$\eta(z)=z$} &
  \raisebox{-.5\height}{\includegraphics[width=.27\textwidth]{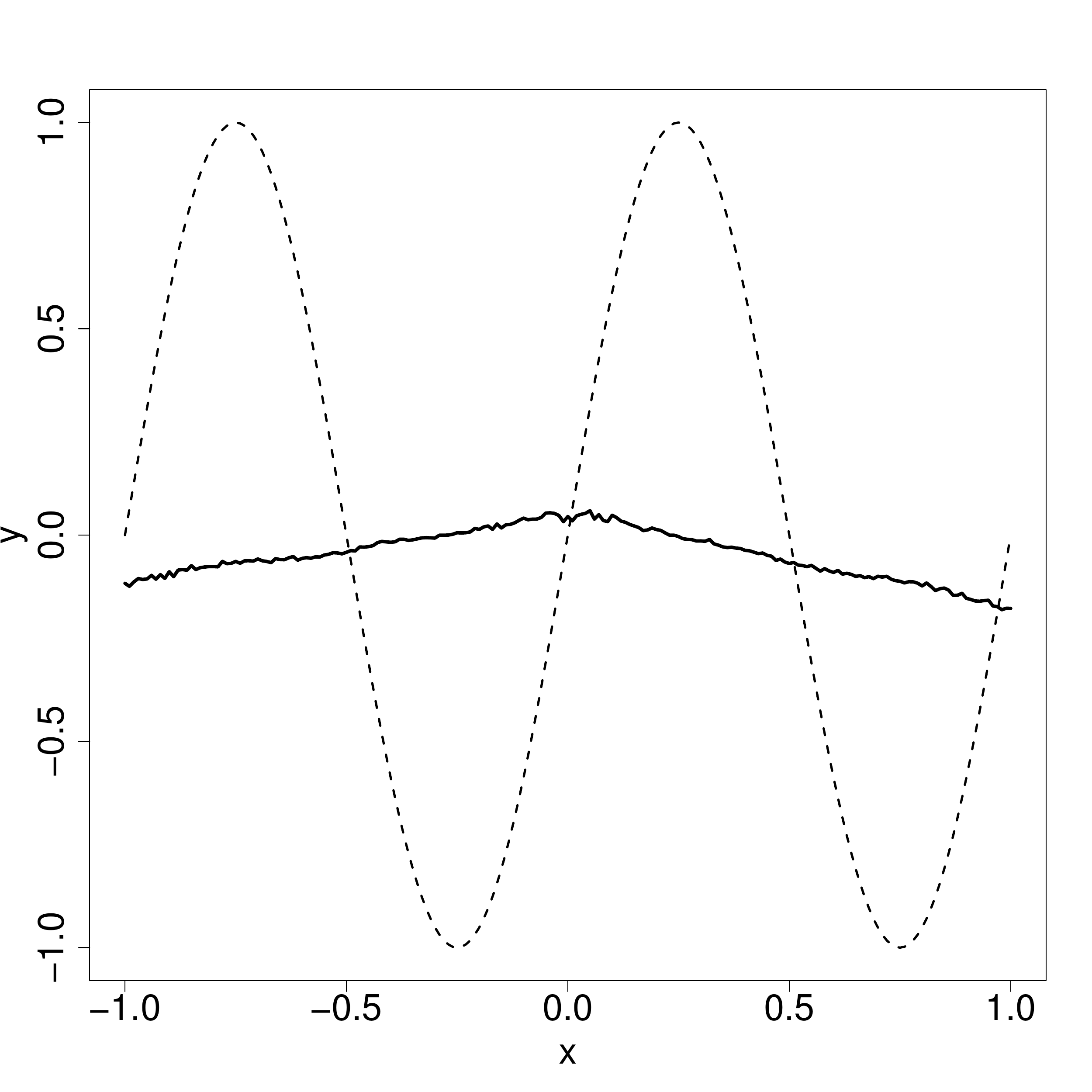}} &
  \raisebox{-.5\height}{\includegraphics[width=.27\textwidth]{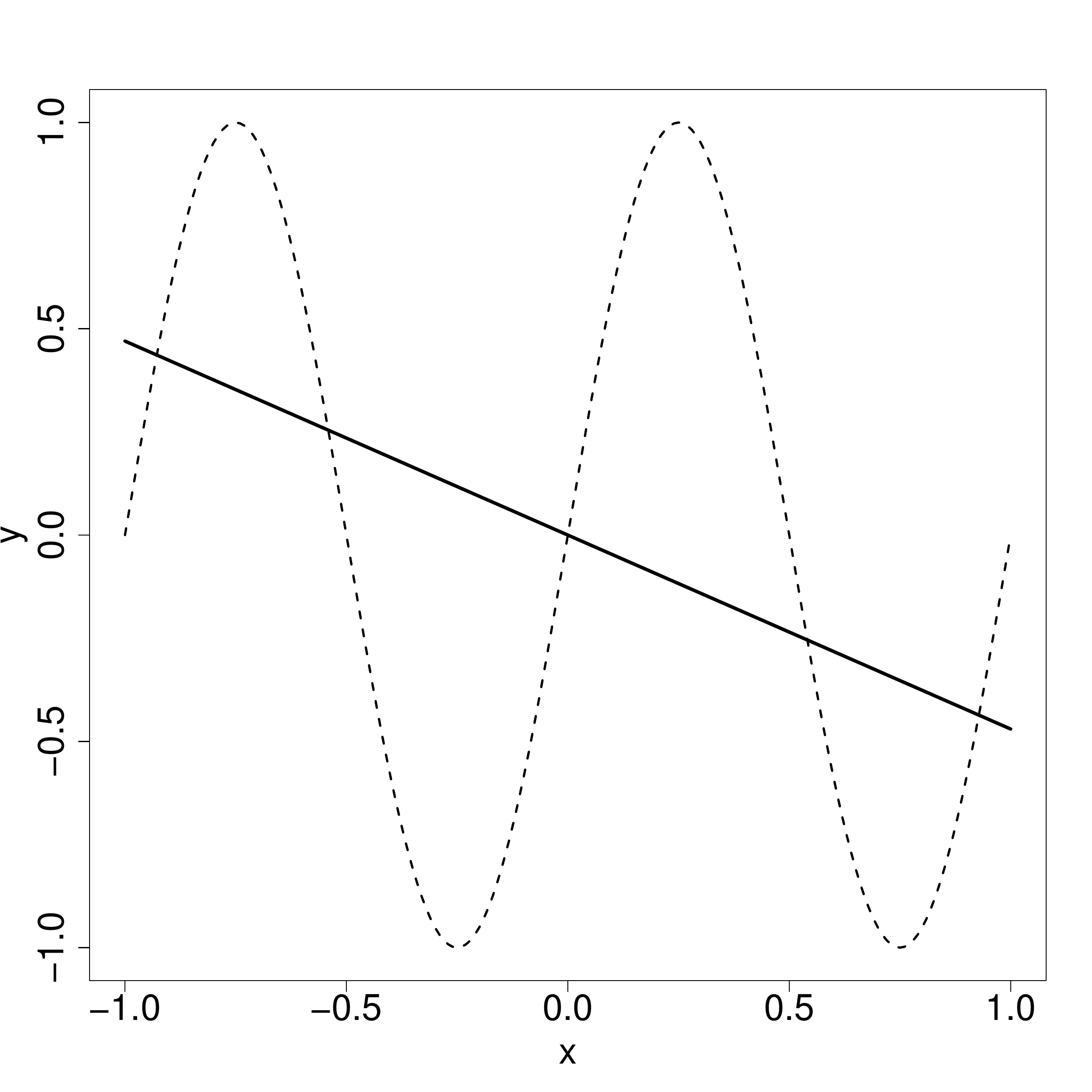}} &
  \raisebox{-.5\height}{\includegraphics[width=.27\textwidth]{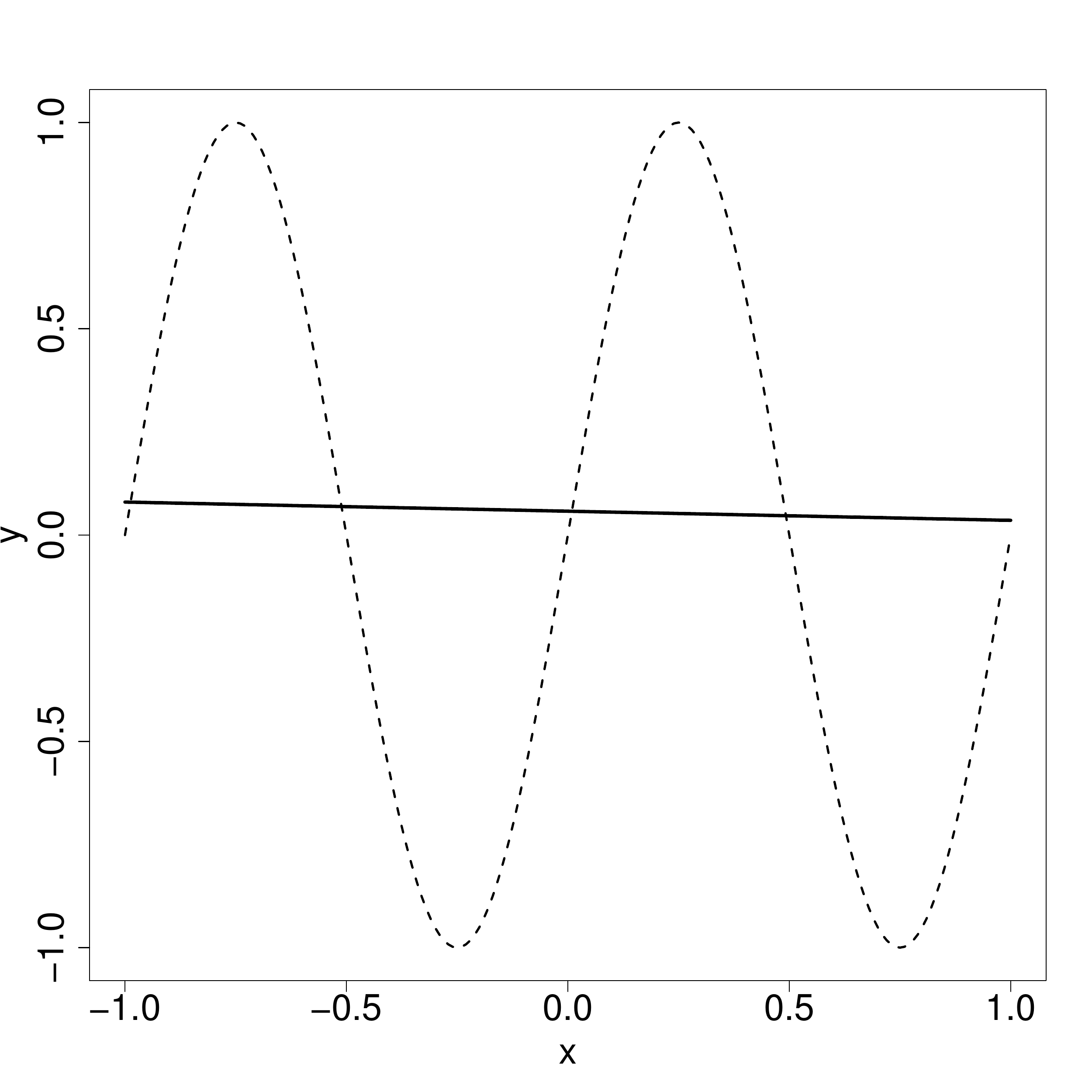}} \\
\end{tabular}
  \caption{Reconstruction with linear function $\eta(z) = z$.
The solid line is a plot the reconstruction result; the dotted line plots the original signal.
}
  \label{fig:sin_linear}
\end{figure}

\reffig{fig:sin_Tab} depicts the ridgelet transform $\rid_\psi f(a,b)$.
\rsho{As the order $\ell$ of $\psi = \bp \gauss^{(\ell)}$ increases, the localization of $\rid_\psi f$ increases.
As shown in \reffig{fig:sin_sigs}, every $\rid_\psi f$ can be reconstructed to $f$ with some admissible activation function $\eta$.
It is somewhat intriguing that the case $\psi = \bp \gauss''$ can be reconstructed with two different activation functions.}

Figures \ref{fig:sin_sigs}, \ref{fig:sin_truncs}, and \ref{fig:sin_linear} tile the results of reconstruction with sigmoidal functions, truncated power functions, and a linear function.
The solid line is a plot of the reconstruction result; the dotted line draws the original signal.
In each of the figures, the theoretical diagnoses and experimental results are almost consistent and reasonable.

In \reffig{fig:sin_sigs}, at the bottom left, the reconstruction signal with the softplus seems incompletely reconstructed,
\rsho{in spite of \reftab{tab:admissible} indicating '$\infty$'.
Recall that $\widehat{\sig^{(-1)}}(\fz)$ has a pole $\fz^{-2}$;
thus, we can understand this cell in terms of \rsho{$\sps * \bp \gauss$} working as an {\em integrator}, that is, a {\em low-pass filter}.}

In \reffig{fig:sin_truncs}, in the top row, all the reconstructions with Dirac's $\delta$ fail.
These results seem to contradict the theory.
However, it simply reflects the \rsho{implementation} difficulty of realizing Dirac's $\delta$,
because $\delta(z)$ is a ``function'' that is almost constantly zero, except for the origin.
Nevertheless, $z = ax-b$ rarely happens to be exactly zero, provided $a, b$, and $x$ are discretized.
This is the reason why this row fails.
At the bottom left, the ReLU seems to lack sharpness for reconstruction. Here we can again understand that \rsho{$z_+ * \bp \gauss$} worked as a low-pass filter.
It is worth noting that the unit step function and the ReLU provide a sharper reconstruction than the sigmoidal function and the softplus.

In \reffig{fig:sin_linear}, all the reconstructions with a linear function fail.
This is consistent with the theory that polynomials cannot be admissible as their Fourier transforms are singular at the origin.

\clearpage
\subsection{Shepp-Logan phantom}
We next studied a gray-scale image {\em Shepp-Logan phantom} \cite{LS.phantom}. 
The ridgelet functions $\psi = \bp^2 \psi_0$ were chosen from the $\sth{\ell}$ derivatives of the Gaussian $\psi_0 = \gauss^{(\ell)}, \ (\ell=0,1,2)$.
The activation functions $\eta$ were chosen from the RBF $\rbf$ (instead of Dirac's $\delta$), the unit step function $z_+^0$, and the ReLU $z_+$.

The original image was composed of $256 \times 256$ pixels. We treated it as a two-dimensional signal $f(\xx)$ defined on $[-1, 1]^2$.
We computed the reconstruction formula
\begin{align}
\int_{\RR} \int_{\RR^2} \rid_\psi f(\ba,b) \eta(\ba \cdot \xx-b) \frac{\dd \ba \dd b}{\rsho{\| \ba \|}},
\end{align}
by discretizing $(\ba,b) \in [ -300, 300 ]^2 \times [-30, 30]$ by $\Delta \ba = (1,1)$ and $\Delta b = 1$.

\reffig{fig:phantom} lists the results of the reconstruction.
As observed in the one-dimensional case, the results are fairly consistent with the theory.
Again, at the bottom left, the reconstructed image seems dim. Our understanding is that it was caused by low-pass filtering.

\begin{figure}[h!]
\centering
\begin{tabular}{cccc}
  & $ \psi = \bp^2 \gauss $ & $ \psi  = \bp^2 \gauss' $ & $ \psi = \bp^2 \gauss'' $ \\
  \rotatebox{90}{$\eta = \rbf$} &
  \raisebox{-.5\height}{\includegraphics[width=.28\linewidth]{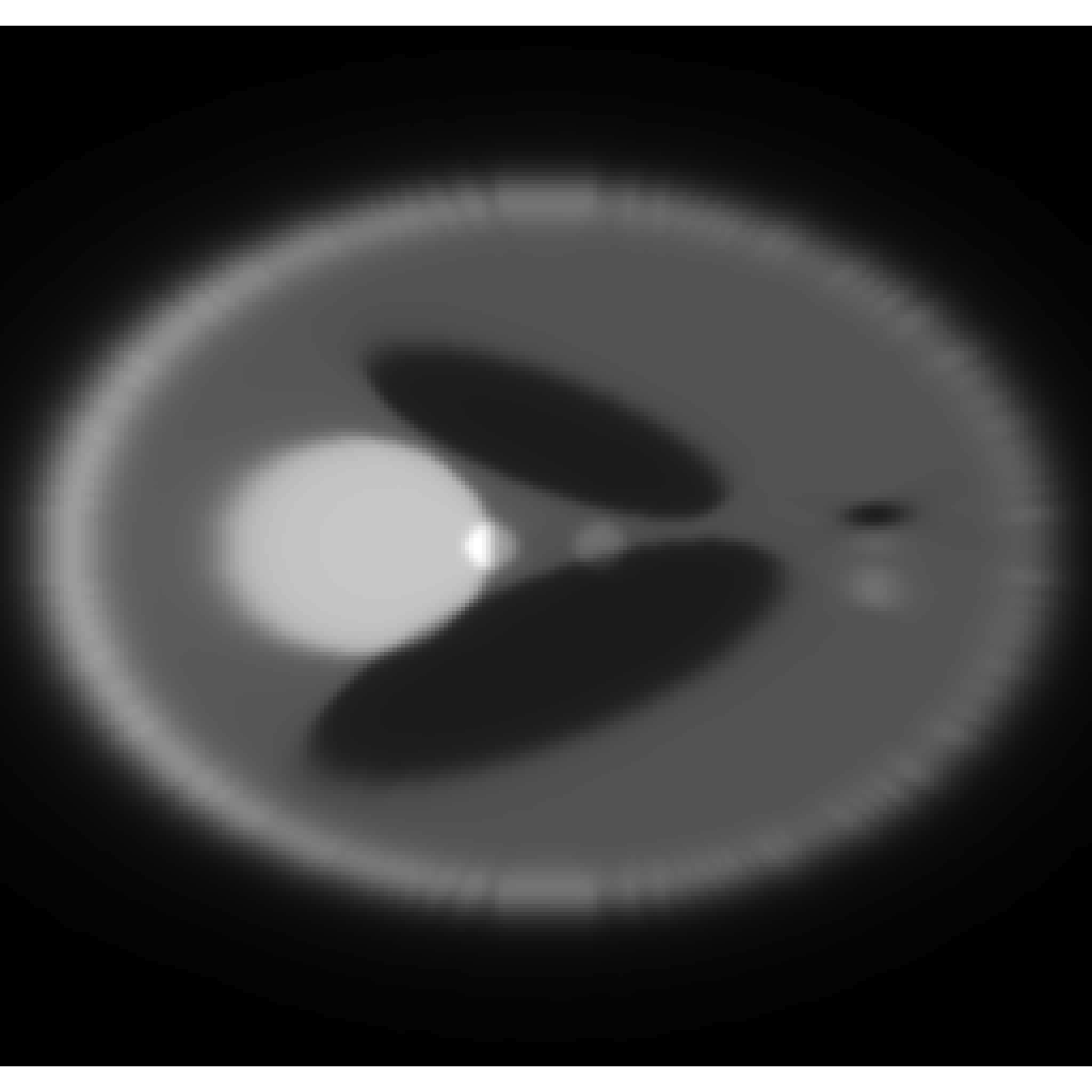}} &
  \raisebox{-.5\height}{\includegraphics[width=.28\linewidth]{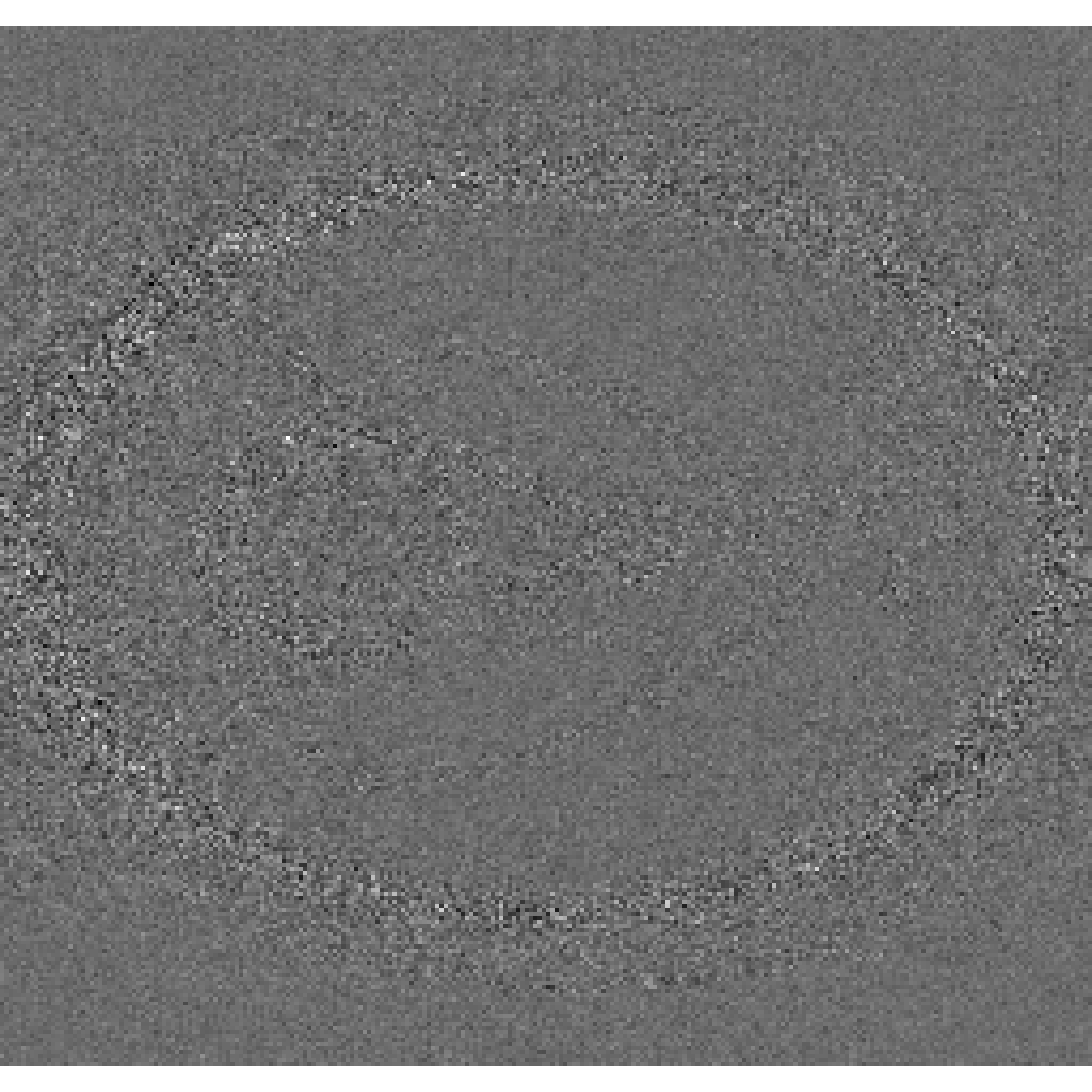}} &
  \raisebox{-.5\height}{\includegraphics[width=.28\linewidth]{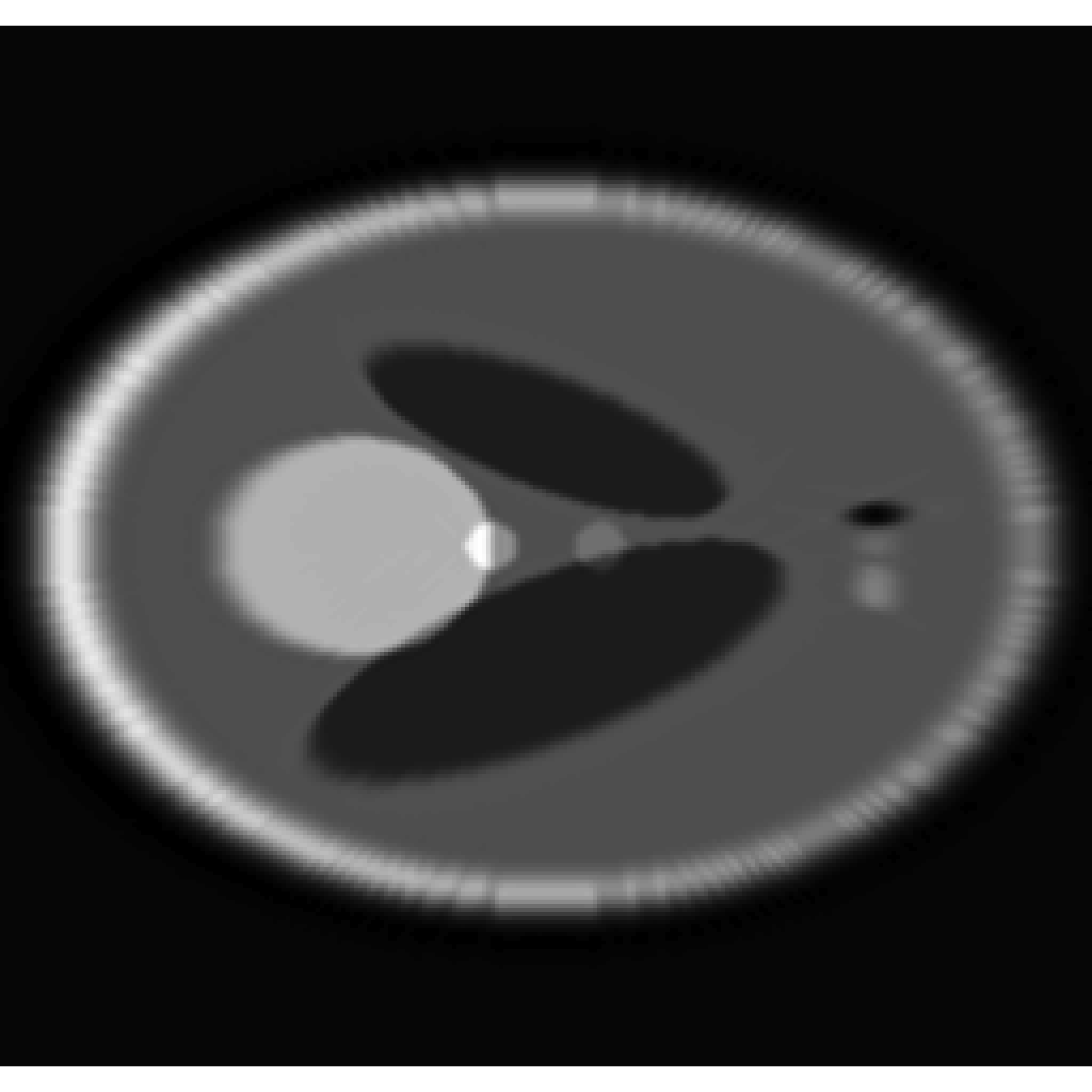}} \\
  \rotatebox{90}{$\eta = z_+^0$} &
  \raisebox{-.5\height}{\includegraphics[width=.28\textwidth]{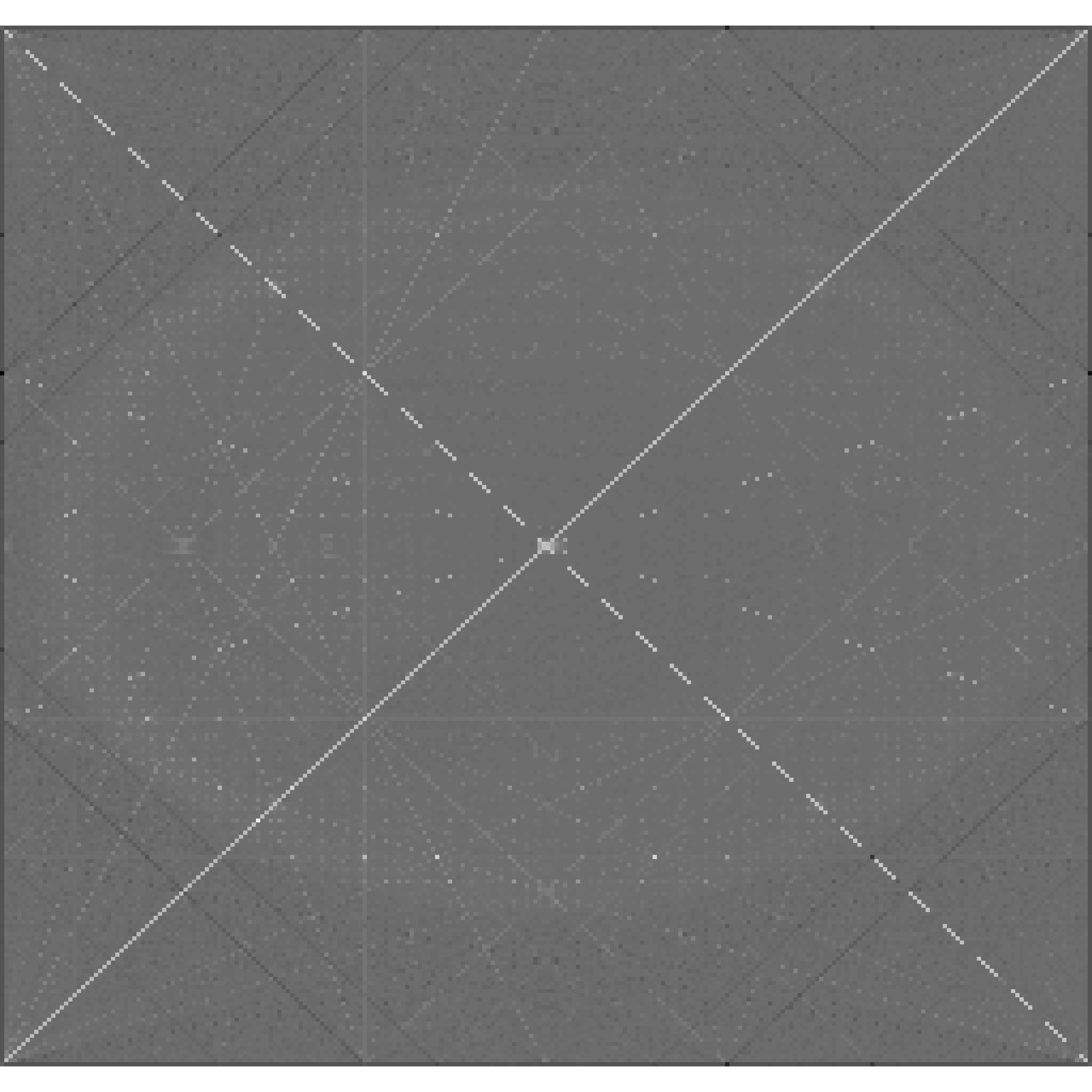}} &
  \raisebox{-.5\height}{\includegraphics[width=.28\textwidth]{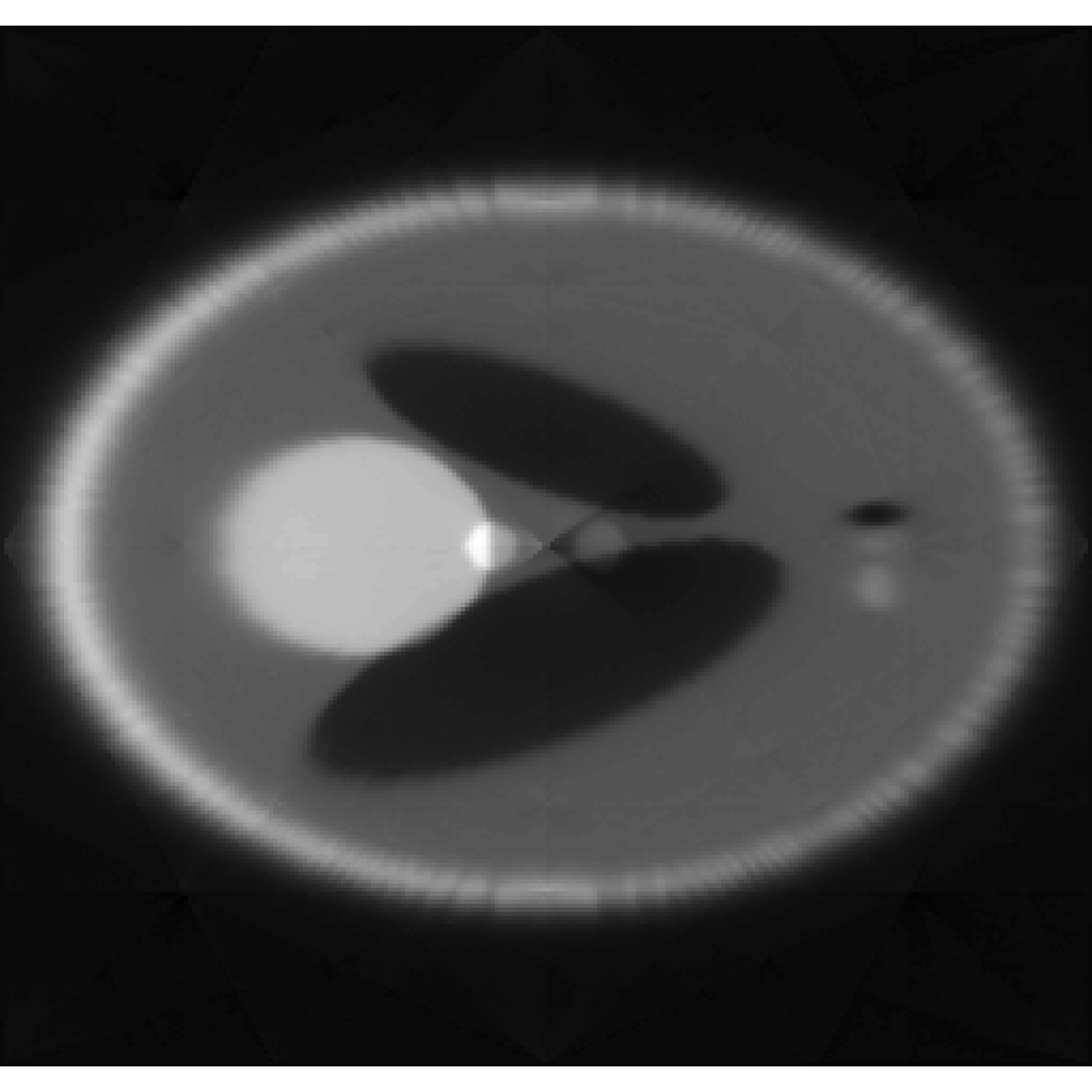}} &
  \raisebox{-.5\height}{\includegraphics[width=.28\textwidth]{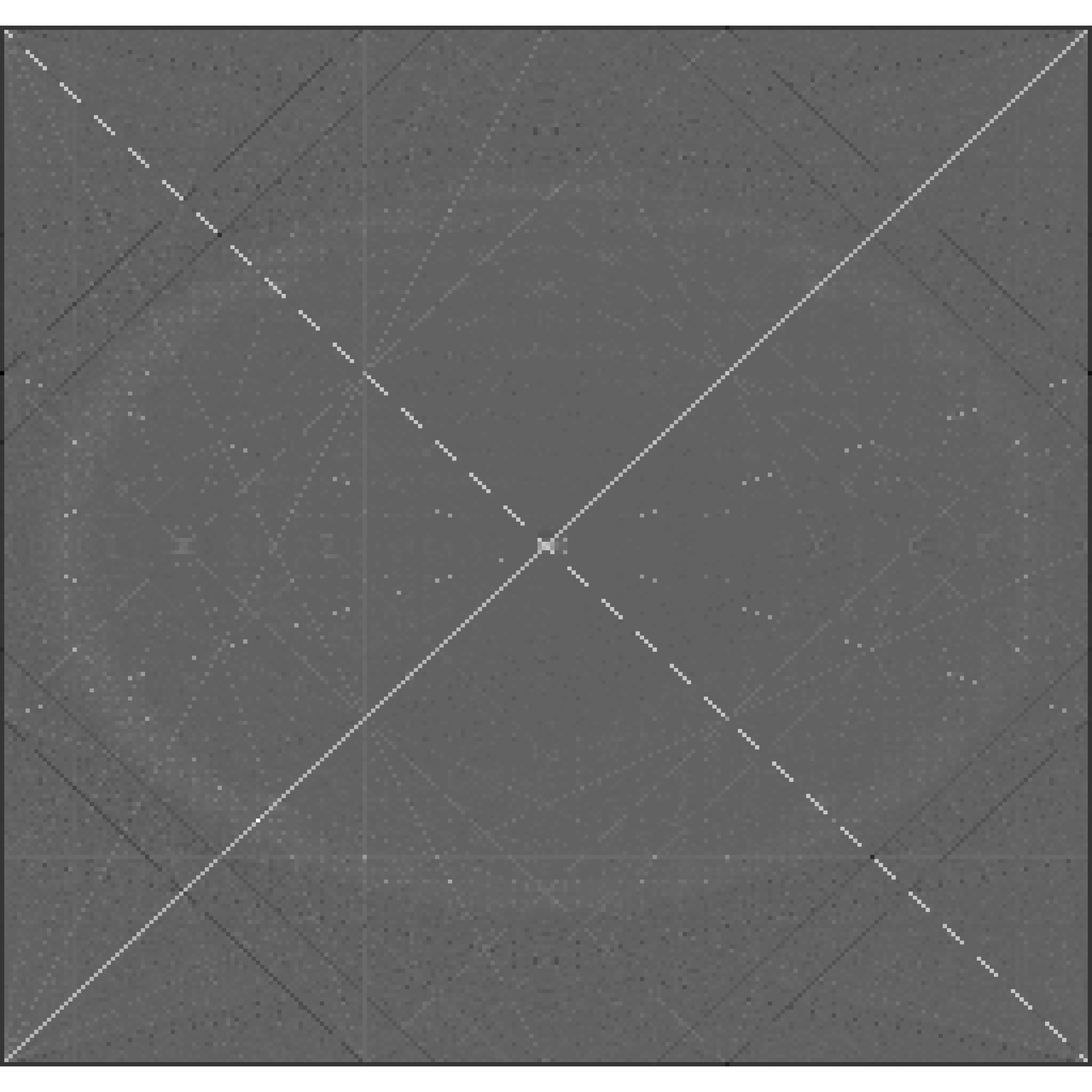}} \\
  \rotatebox{90}{$\eta = z_+$} &
  \raisebox{-.5\height}{\includegraphics[width=.28\textwidth]{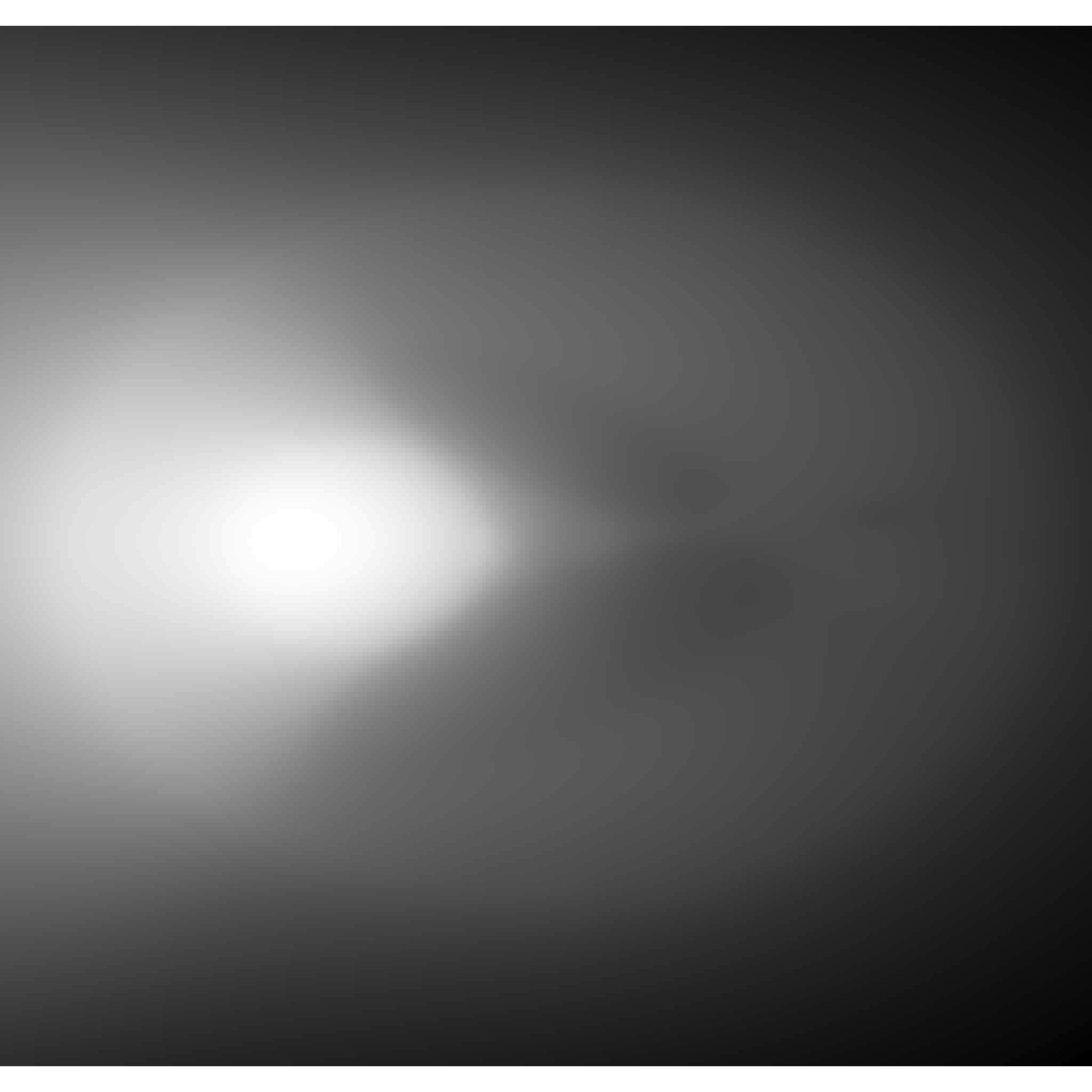}} &
  \raisebox{-.5\height}{\includegraphics[width=.28\textwidth]{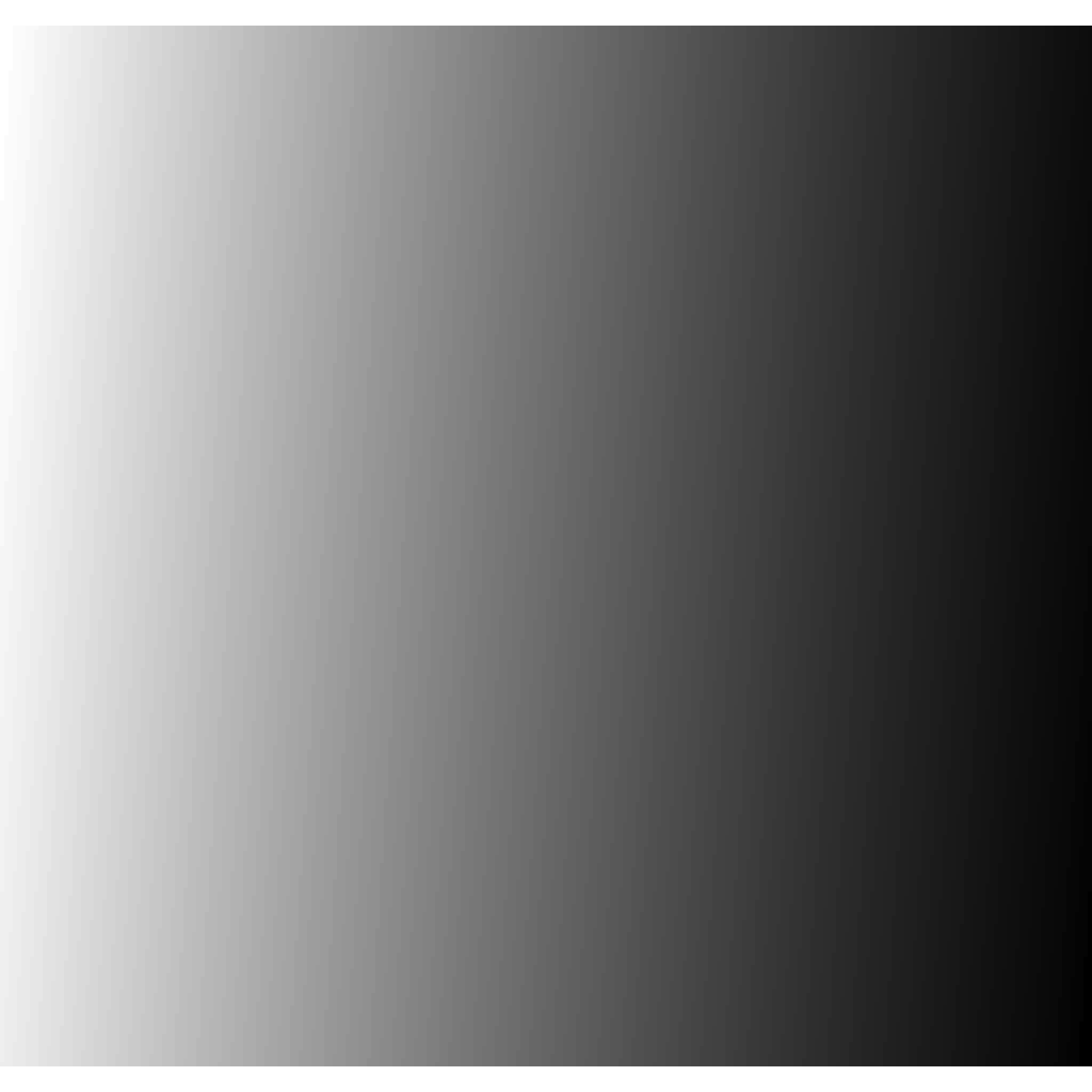}} &
  \raisebox{-.5\height}{\includegraphics[width=.28\textwidth]{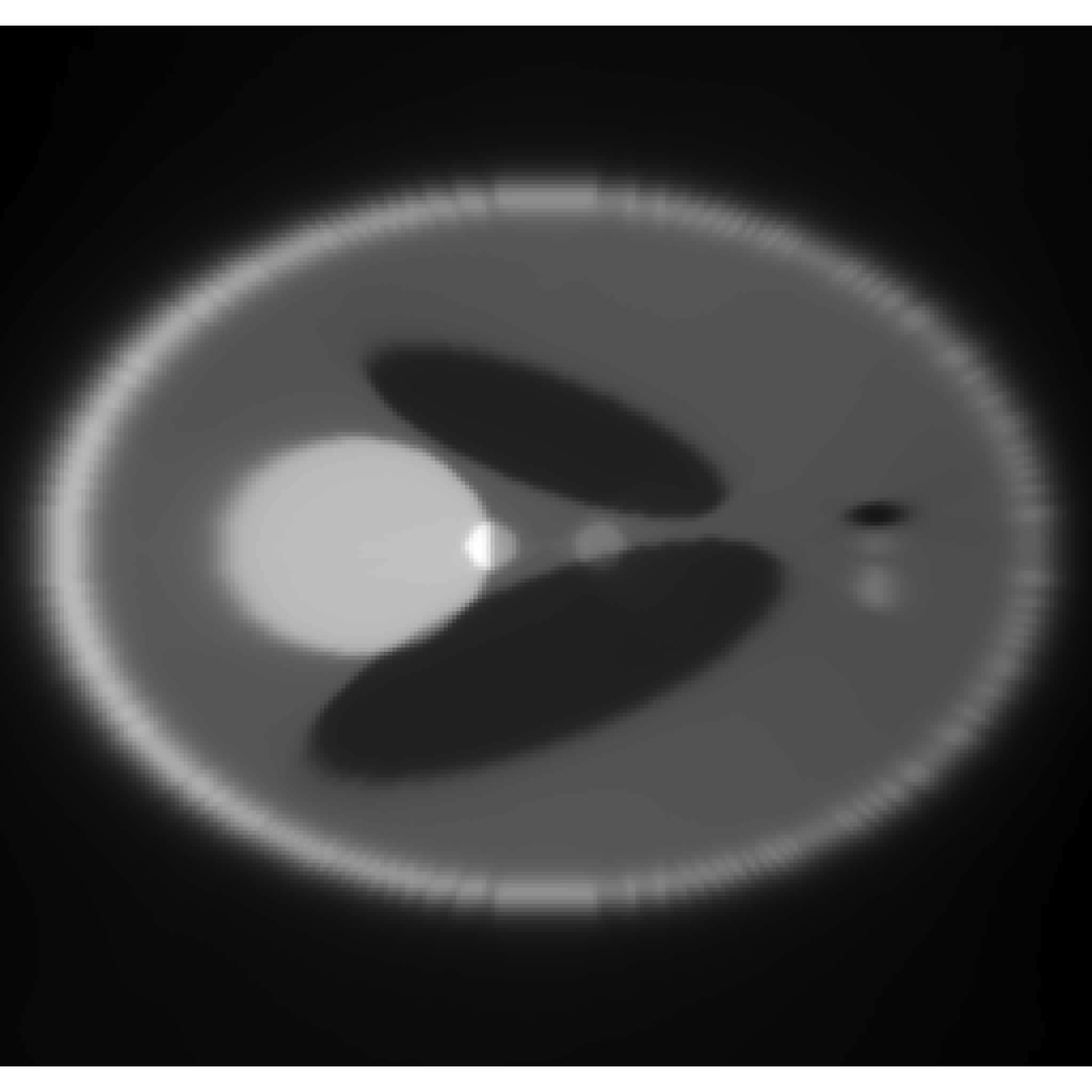}} \\
\end{tabular}
  \caption{Reconstruction with RBF $\rbf$, unit step $z_+^0$, and ReLU $z_+$.}
  \label{fig:phantom}
\end{figure}

\clearpage
\section{Concluding Remarks}
We have shown that \rtwo{neural networks with unbounded non-polynomial activation functions have} the universal approximation property.
Because the integral representation of the neural network coincides with the dual ridgelet transform,
our goal reduces to \rtwo{constructing} the ridgelet transform with respect to distributions.
Our results cover a wide range of activation functions: not only the traditional RBF, sigmoidal function, and unit step function,
but also truncated power functions $z_+^k$, which contain the ReLU and even Dirac's $\delta$.
In particular, we concluded that a neural network can approximate \rsho{$L^1 \cap C^0$ functions in the pointwise sense,
and $L^2$ functions in the $L^2$ sense,}
 when its activation ``function'' is a Lizorkin distribution ($\spLizD$) that is admissible.
The Lizorkin distribution is a tempered distribution ($\spTD$) that is not a polynomial.
As an important consequence,
 what a neural network learns is a ridgelet transform of the target function $f$.
In other words, during backpropagation the network indirectly searches for an admissible ridgelet function, \rtwo{by constructing} a backprojection filter.

Using the weak form expression of the ridgelet transform, we extensively defined the ridgelet transform with respect to distributions.
\refthm{thm:existence} guarantees the existence of the ridgelet transform with respect to distributions.
\reftab{tab:weakridge} suggests that for the convolution of distributions to converge, 
the class $\spX$ of domain and the class $\spZ$ of ridgelets should be balanced.
\rsho{\refprop{prop:conti.L1} states that $\rid_\psi : L^1(\RR^m) \to \spTD(\YY^{m+1})$ is a bounded linear operator.}
\refthm{thm:dual} states that the dual ridgelet transform coincides with a dual operator.
Provided the reconstruction formula holds, that is, when the ridgelets are admissible,
the ridgelet transform is injective and the dual ridgelet transform is surjective.

For an unbounded $\eta \in \spZ(\RR)$ to be admissible, it cannot be a polynomial and it can be associated with a backprojection filter.
If $\eta \in \spZ(\RR)$ is a polynomial then the product of distributions in the admissibility condition should be indeterminate.
Therefore, $\spZ(\RR)$ \rtwo{excludes} polynomials.
\rsho{\refthm{thm:eq.ac}} rephrases the admissibility condition \rsho{in the real domain.} 
As a direct consequence, \refcor{cor:const.ap} gives a constructive sufficiently admissible condition.

After investigating the construction of the admissibility condition, we showed that formulas can be reconstructed \rsho{on $L^1(\RR^m)$} in two ways.
\refthm{thm:formula} uses the Fourier slice theorem.
\refthm{thm:formula.radon} uses approximations to the identity and reduces to the inversion formula of the Radon transform.
\rsho{\refthm{thm:formula.radon} as well as \refcor{cor:radon.d} suggest} that the admissibility condition requires $(\psi,\eta)$ to construct a backprojection filter.

\rsho{In addition, we have extended the ridgelet transform on $L^1(\RR^m)$ to $L^2(\RR^m)$.}
\refthm{thm:parseval} states that Parseval's relation, which is a weak version of the reconstruction formula, holds on $L^1 \cap L^2(\RR^m)$.
\refthm{thm:L2} follows the bounded extension of $\rid_\psi$ from $L^1 \cap L^2(\RR^m)$ to $L^2(\RR^m)$.
\refthm{thm:formula.L2} gives the reconstruction formula in $L^2(\RR^m)$.

By showing \rtwo{that} $z_+^k$ and other activation functions belong to $\spLizD$, and \rtwo{that} they are admissible with some derivatives of the Gaussian, we \rtwo{proved} the universal approximation property of a neural network with an unbounded activation function.
Numerical examples were consistent \rtwo{with} our theoretical diagnoses on the admissibility.
In addition, we found that some non-admissible combinations worked as a low-pass filter; for example, $(\psi,\eta) = (\bp^{m}[\mathrm{Gaussian}], \mathrm{ReLU})$ and $(\psi,\eta) = (\bp^{m}[\mathrm{Gaussian}], \mathrm{softplus} $).

We plan to perform the following interesting investigations in future.
\begin{enumerate}
\item Given an activation function $\eta \in \spLizD(\RR)$, which is the ``best'' ridgelet function $\psi \in \spSch(\RR)$?
\begin{itemize}
\item[]
In fact, for a given activation function $\eta$, we have plenty of choices. By \refcor{cor:const.ap}, all elements of 
\begin{gather}
\spany_\eta := \left\{ \bp^m \psi_0 \ \middle| \ \psi_0 \in \spSch(\RR) \mbox{ such that } \dual{ \widehat{\eta},\widehat{\psi_0} } \mbox{ is finite and nonzero. }\right\},
\end{gather}
are admissible with $\eta$.
\end{itemize}
\item How are ridgelet functions related to {\em deep} neural networks?
\begin{itemize}
\item[]
Because ridgelet analysis is so fruitful, we aim to develop ``deep'' ridgelet analysis.
One of the essential leaps from shallow to deep is that the network output expands from scalar to vector
because a deep structure is a cascade of multi-input multi-output layers.
In this regard, we expect \rsho{\refcor{cor:radon.d}} to play a key role.
By using the intertwining relations, we can ``cascade'' the reconstruction operators as below
\begin{gather}
\drid_\eta \rid_\psi \drid_\eta \rid_\psi = \drad \bp^{k-1}\rad (- \Delta)^{m-1 -\frac{k+\ell}{2}}\drad \bp^{\ell-1}\rad. \quad (0 \leq k,\ell \leq m)
\end{gather}
This equation suggests that the cascade of ridgelet transforms coincides with a composite of backprojection filtering in the Radon domain and differentiation in the real domain.
We conjecture that this point of view can be expected to facilitate analysis of the deep structure.
\end{itemize}
\end{enumerate}

\subsection*{Acknowledge}
\rsho{The authors would like to thank the anonymous reviewers for fruitful comments and suggestions to improve the quality of the paper.}
\rsho{The authors would like to express their appreciation toward Dr.~Hideitsu Hino for his kind support with writing the paper.}
\rsho{This work was supported by JSPS KAKENHI Grand Number 15J07517.}

\appendix

\section{Proof of \refthm{thm:existence}} \label{app:proof.existence}
A ridgelet transform $\rid_\psi f(\bu,\alpha,\beta)$ is the convolution of a Radon transform $\rad f(\bu,p)$ and a dilated distribution $\psi_\alpha(p)$ in the sense of a Schwartz distribution.
That is,
\begin{align}
f(\xx) \mapsto \rad f(\bu,p) \mapsto \left(\rad f(\bu, \cdot ) * \overline{\refl{\psi_\alpha}}\right) (\beta) = \rid_\psi f(\bu, \alpha, \beta).
\end{align}
\rsho{We verify that the ridgelet transform is well defined in a stepwise manner.}
\rsho{Provided there is no danger of confusion, in the following steps we denote by $\spX$ the classes $\spcomp, \spED, \spSch, \spCD, L^1$, or $\spLD{1}$.}

\subhead{Step 1}{Class $\spX(\Sph^{m-1} \times \RR)$ of $\rad f(\bu,p)$}
Hertle's results found \cite[Th 4.6, Cor 4.8]{Hertle1983} that the Radon transform is the continuous injection
\begin{align}
\rad : \spX(\RR^m) \hookrightarrow \spX(\Sph^{m-1} \times \RR),
\end{align}
\rsho{where $\spX = \spcomp, \spED, \spSch, \spCD, L^1$, or $\spLD{1}$;
if $f \in \spX(\RR^m)$ then $\rad f \in \spX(\Sph^{m-1} \times \RR)$,
which determines the second column.}
Our possible choice of the domain $\spX$ is restricted to them.

\subhead{Step 2}{Class $\spB(\RR)$ of $\rid_\psi f (\bu, \alpha, \beta)$ with respect to $\beta$}
\rsho{
Fix $\alpha > 0$.
Recall that $\rid_\psi f(\bu, \alpha, \beta) = \left( \rad f (\bu, \cdot ) * \overline{\widetilde{\psi_\alpha}}\right)(\beta)$ in the sense of Schwartz distributions.
By the nuclearity of $\spX$ \cite[\S~51]{Treves.new}, the kernel theorem
\begin{align}
\rsho{ \spX(\Sph^{m-1} \times \RR) \cong \spX(\Sph^{m-1}) \ttprod \spX(\RR),} \label{eq:kernel}
\end{align}
holds.
Therefore, we can omit $\bu \in \Sph^{m-1}$ in the considerations for $(\alpha, \beta) \in \Half$.
According to Schwartz's results shown in \reftab{tab:conv},
for the convolution $g * \psi$ of $g \in \spX(\RR)$ and $\psi \in \spZ(\RR)$ to converge in $\spB(\RR)$,
we can assign the largest possible class $\spZ$ for each $\spX$ as in the third column. Note that for $\spX = L^1$ we even assumed the continuity $\spZ=L^p \cap C^0$, which is technically required in Step 3.
Obviously for $\spZ = \spSD, \spTD, L^p \cap C^0$, or $\spLD{p}$, if $\psi \in \spZ(\RR)$ then $\psi_\alpha \in \spZ(\RR)$.
Therefore, we can determine the fourth column by evaluating $\spX * \spZ$ according to \reftab{tab:conv}.
 }

\subhead{Step 3}{
Class \rsho{$\spA(\Half)$} of $\rid_\psi f (\bu, \alpha, \beta)$ with respect to \rsho{$(\alpha, \beta)$}}
Fix $\bu_0 \in \Sph^{m-1}$ and assume $f \in \spX(\RR^m)$.
Write $g(p) := \rad f(\bu_0, p)$ and
\begin{align}
\wav[\psi; g](\alpha, \beta) := \int_\RR g(\alpha z + \beta) \overline{\psi(z)} \dd z,
\end{align}
then $\rid_\psi f(\bu_0, \alpha, \beta) = \wav[\psi;g](\alpha, \beta)$ for every $(\alpha, \beta) \in \Half$.
By the kernel theorem, $g \in \spX(\RR)$.

\subhead{Case 3a}{$(\spX = \spcomp \mbox{ and } \spZ = \spSD \mbox{ then } \spB = \spsmooth \mbox{ and } \spA = \spsmooth)$}
We begin by considering the case in the first row.
Observe that 
\begin{gather}
\partial_\alpha \wav[\psi;g](\alpha, \beta)
= \partial_\alpha \int_\RR g(\alpha z + \beta) \overline{\psi(z)} \dd z 
= \int_\RR g'(\alpha z + \beta) \overline{ z \cdot \psi(z)} \dd z 
= \wav[z \cdot \psi; g'](\alpha, \beta), \label{eq:wavda}\\
\partial_\beta \wav[\psi; g](\alpha, \beta)
= \partial_\beta \int_\RR g(\alpha z + \beta) \overline{\psi(z)} \dd z 
= \int_\RR g'(\alpha z + \beta) \overline{\psi(z)} \dd z
= \wav[\psi; g'](\alpha, \beta), 
\label{eq:wavdb}\\
\intertext{and thus that for every $k,\ell \in \NN_0$,}
\partial^k_\alpha \partial_\beta^\ell \wav[\psi; g](\alpha, \beta) = \wav[z^k \cdot \psi; g^{(k+\ell)}](\alpha,\beta).
 \label{eq:wavdab}
\end{gather}
Obviously if $g \in \spcomp(\RR)$ and $\psi \in \spSD(\RR)$ then $g^{(k+\ell)} \in \spcomp(\RR)$ and $z^k \cdot \psi \in \spSD(\RR)$, respectively,
and thus $\partial_\alpha^k \partial_\beta^\ell \wav[\psi; g](\alpha, \beta)$
exists at every $(\alpha, \beta) \in \Half$.
Therefore, we can conclude that if $g \in \spcomp(\RR)$ and $\psi \in \spSD(\RR)$ then $\wav[\psi; g] \in \spsmooth(\Half)$.

\subhead{Case 3b}{$(\spX = \spED \mbox{ and } \spZ = \spSD \mbox{ then } \spB = \spSD \mbox{ and } \spA = \spSD)$}
Let $g \in \spED(\RR)$ and $\psi \in \spSD(\RR)$. We show that $\wav[\psi; g] \in \spSD(\Half)$, that is,
for every compact set $\KK \subset \Half$, there exists $N \in \NN_0$ such that
\begin{align}
\Bigg| \int_\KK \TT(\alpha,\beta) \wav[\psi;g](\alpha,\beta) \frac{\dd \alpha \dd \beta}{\alpha} \Bigg| \lesssim \sum_{k,\ell \leq N} \sup_{(\alpha,\beta) \in \Half} | \partial_\alpha^k \partial_\beta^\ell \TT (\alpha,\beta) |, \quad \forall \TT \in \spcomp(\KK).
\end{align}
Fix an arbitrary compact set $\KK \subset \Half$ and a smooth function $\TT \in \spcomp(\KK)$, which is supported in $\KK$.
Take two compact sets $\mathrm{A} \subset \RR_+$ and $\mathrm{B} \subset \RR$ such that $\KK \subset \mathrm{A} \times \mathrm{B}$.
By the assumption that $g \in \spED(\RR)$ and $\psi \in \spSD(\RR)$, there exist $k, \ell \in \NN_0$ such that
\begin{gather}
\Bigg| \int_\RR u(z) g(z) \dd z \Bigg| \lesssim \sup_{z \in \supp g} | u^{(k)}(z) |, \quad \forall u \in \spsmooth(\RR) \label{eq:case3bg} \\
\Bigg| \int_\RR v(z) \overline{\psi(z)}\dd z \Bigg| \lesssim \sup_{z \in \RR} | v^{(\ell)}(z) |, \quad \forall v \in \spcomp(B). \label{eq:case3bpsi}
\end{gather}
Observe that for every fixed $\alpha$, $\TT(\alpha, \cdot) * \widetilde{g} \in \spSD(\RR)$. Then, by applying \refeq{eq:case3bg} and \refeq{eq:case3bpsi} incrementally,
\begin{align}
\Bigg| \int_\RR \TT(\alpha,\beta) \int_\RR g(\alpha z + \beta) \overline{\psi(z)} \dd z \frac{\dd \alpha \dd \beta}{\alpha} \Bigg|
&\leq \int_0^\infty \Bigg| \int_{\RR} \int_\RR \TT(\alpha,\beta - \alpha z) \overline{\psi(z)} \dd z \cdot g(\beta) \dd \beta \Bigg| \frac{\dd \alpha}{\alpha} \\
&\lesssim \int_0^\infty \sup_{\beta \in \supp g} \Bigg| \int_\RR \partial_\beta^k \TT(\alpha,\beta - \alpha z) \psi (z) \dd z \Bigg| \frac{\dd \alpha}{\alpha} \\
&\lesssim \int_0^\infty \sup_{\beta \in \supp g} \sup_z \Big| \partial_\beta^{k+\ell} \TT(\alpha,\beta - \alpha z) \Big| \alpha^{\ell-1} \dd \alpha \\
&= \int_\mathrm{A} \sup_{\beta \in \mathrm{B}} \Big| \partial_\beta^{k+\ell} \TT(\alpha,\beta) \Big| \alpha^{\ell-1} \dd \alpha \\
& \leq \sup_{(\alpha,\beta) \in \KK} \Big| \partial_\beta^{k+\ell} \TT(\alpha,\beta) \Big| \cdot \int_\mathrm{A} \alpha^{\ell-1} \dd \alpha,
\end{align}
where the third inequality follows by 
repeatedly applying $\partial_z [\TT (\alpha, \beta - \alpha z)] = (-\alpha) \partial_\beta \TT(\alpha,\beta - \alpha z)$;
the fourth inequality follows by the compactness of the support of $\TT$. Thus, we conclude that $\wav[\psi;g] \in \spSD(\Half)$.

\subhead{Case 3c}{$(\spX = \spSch \mbox{ and } \spZ = \spTD \mbox{ then } \spB = \spM \mbox{ and } \spA = \spM)$}
Let $g \in \spSch(\RR)$ and $\psi \in \spTD(\RR)$.
Recall the case when $\spX = \spcomp$.
Obviously, for every $k,\ell \in \NN_0$,
$g^{(k+\ell)} \in \spSch(\RR)$ and $z^k \cdot \psi \in \spTD(\RR)$, respectively,
which implies $\wav[\psi; g] \in \spsmooth(\Half)$.
Now we even show that $\wav[\psi; g] \in \spM(\Half)$, that is, for every $k,\ell \in \NN_0$ there exist $s,t \in \NN_0$ such that
\begin{align}
\big| \partial_\alpha^k \partial_\beta^\ell \wav[\psi; g](\alpha,\beta) \big| \lesssim (\alpha + 1/\alpha)^s (1+\beta^2)^{t/2}.
\end{align}
Recall that by \eqref{eq:wavdab}, we can regard $\partial^k_\alpha \partial_\beta^\ell \wav[\psi; g](\alpha, \beta)$ as $\partial^0_\alpha \partial_\beta^0 \wav[\psi_0; g_0](\alpha,\beta)$,
by setting $g_0 := g^{(k+\ell)}\in \spSch(\RR)$ and $\psi_0 := z^k \cdot \psi \in \spTD(\RR)$.
Henceforth we focus on the case when $k=\ell=0$.
Since $\psi \in \spTD(\RR)$, there exists $N \in \NN_0$ such that
\begin{align}
\Bigg| \int_\RR u(z)\overline{\psi(z)} \dd z \Bigg| \lesssim \sum_{s,t \leq N} \sup_{z \in \RR}| z^s u^{(t)}(z) |, \quad \forall u \in \spSch(\RR).
\end{align}
By substituting $u(z) \gets g(\alpha z + \beta)$, we have
\begin{align}
\Bigg| \int_\RR g(\alpha z + \beta) \overline{\psi(z)} \dd z \Bigg|
& \lesssim \sum_{s,t \leq N} \sup_{z \in \RR}| z^s \partial_z^t g(\alpha z + \beta) | \\
& = \sum_{s,t \leq N} \sup_{p \in \RR} \Bigg| \left( \frac{p - \beta}{\alpha} \right)^s \alpha^t g^{(t)}(p) \Bigg| \\
& \lesssim \sum_{s,t \leq N} \alpha^{t-s} \beta^s \sup_{p \in \RR} | p^{s} g^{(t)}(p) | \\
& \lesssim \left( \alpha + 1/\alpha \right)^N (1+\beta^2)^{N/2},
\end{align}
where the second equation follows by substituting $p \gets \alpha z + \beta$;
the fourth inequality follows because every $\sup_p | p^s g^{t} (p)|$ is finite by assumption that $g \in \spSch(\RR)$. Therefore, we can conclude that if $g \in \spSch(\RR)$ and $\psi \in \spTD(\RR)$ then $\wav[\psi; g] \in \spM(\Half)$.

\subhead{Case 3d}{$(\spX = \spCD \mbox{ and } \spZ = \spTD \mbox{ then } \spB = \spTD \mbox{ and } \spA = \spTD)$}
Let $g \in \spCD(\RR)$ and $\psi \in \spTD(\RR)$. We show that $\wav[\psi;g] \in \spTD(\Half)$, that is,
there exists $N \in \NN_0$ depending only on $\psi$ and $g$ such that
\begin{align}
\Bigg| \int_\Half \TT(\alpha,\beta) \wav[\psi; g](\alpha,\beta) \frac{\dd \alpha \dd \beta}{\alpha} \Bigg|
\lesssim
\sum_{s,t,k,\ell \leq N}
\sup_{\alpha, \beta \in \Half} \big| \hsch_{s,t}^{k,\ell} \TT (\alpha,\beta)\big|, \quad 
\forall \TT \in \spSch(\Half) \label{eq:case3dtd}
\end{align}
where we defined
\begin{align}
\hsch_{s,t}^{k,\ell} \TT(\alpha,\beta) := \left( \alpha + 1/\alpha \right)^s (1+\beta^2)^{t/2} \partial_\alpha^k \partial_\beta^\ell \TT(\alpha,\beta).
\end{align}
Fix an arbitrary $\TT \in \spSch(\Half)$.
By the assumption that $\psi \in \spTD(\RR)$, there exist $s,t \in \NN_0$ such that
\begin{align}
\Bigg| \int_\RR u(z) \overline{\psi(z)}\dd z \Bigg| \lesssim \sup_{z} | z^t u^{(s)}(z) |, \quad \forall u \in \spSch(\RR).
\end{align}
Observe that for every fixed $\alpha$, $\TT(\alpha, \cdot) * \widetilde{g} \in \spSch(\RR)$. Then we can provide an estimate as below.
\begin{align}
\Bigg| \int_\Half \TT(\alpha,\beta) \int_\RR g(\alpha z + \beta) \overline{\psi(z)} \dd z \frac{\dd \alpha \dd \beta}{\alpha} \Bigg|
&\leq \int_0^\infty \Bigg| \int_{\RR} \int_\RR \TT(\alpha,\beta) g(\alpha z + \beta) \dd \beta \cdot \overline{\psi(z)} \dd z \Bigg| \frac{\dd \alpha}{\alpha} \\
&\lesssim \int_{\RR} \sup_z \Bigg| z^t \int_\RR \hsch_{s,0}^{0,0} \TT(\alpha,\beta) g^{(s)}(\alpha z + \beta) \dd \beta \Bigg| \frac{\dd \alpha}{\alpha} \\
&\lesssim  \int_{\RR} \sup_p \Bigg| p^t \int_\RR \hsch_{s+t,0}^{0,0} \TT(\alpha,\beta) g^{(s)}(p + \beta) \dd \beta \Bigg| \frac{\dd \alpha}{\alpha} \\
&\leq \int_{\RR} \int_\RR \sup_p \big|p^t g^{(s)}(p + \beta) \big| \big| \hsch_{s+t,0}^{0,0} \TT(\alpha,\beta) \big| \frac{\dd \beta \dd \alpha}{\alpha}\\
&\lesssim  \int_{\RR} \int_\RR \sup_p \big|(1+|p + \beta|^2)^{t/2} g^{(s)}(p + \beta) \big| \big| \hsch_{s+t,t}^{0,0} \TT(\alpha,\beta) \big| \frac{\dd \beta \dd \alpha}{\alpha}\\
&\lesssim  \int_{\Half} \big| \hsch_{s+t,t}^{0,0} \TT(\alpha,\beta) \big| \frac{\dd \beta \dd \alpha}{\alpha} \\
&\leq  \sup_{(\alpha, \beta) \in \Half} \big| \hsch_{s+t+\eps,t+\delta}^{0,0} \TT(\alpha,\beta) \big| \int_\Half \left( \alpha + 1/\alpha \right)^{-\eps} (1+\beta^2)^{-\delta/2} \frac{\dd \beta \dd \alpha}{\alpha},
\end{align}
where the second inequality follows by 
repeatedly applying $\partial_z [g(\alpha z + \beta)] = \alpha \cdot g'(\alpha z+\beta)$ and $\alpha \lesssim \alpha + 1/\alpha$;
the third inequality follows by changing the variable $p \gets \alpha z$ and applying $(\alpha + 1/\alpha)^s  \cdot \alpha^{-t} \lesssim (\alpha + 1/\alpha)^{s+t}$; the fifth inequality follows by applying $|p| \lesssim (1+p^2)^{1/2}$ and Peetre's inequality $1+p^2 \lesssim (1+\beta^2)(1+ |p+\beta|^2)$; the sixth inequality follows by the assumption that $(1+p^2)^{t/2} g(p)$ is bounded for any $t$; the last inequality follows by H\"older's inequality and the integral is convergent when $\eps > 0$ and $\delta > 1$.

\subhead{Case 3e}{$(\spX = L^1 \mbox{ and } \spZ = L^p \cap C^0 \mbox{ then } \spB = L^p \cap C^0 \mbox{ and } \spA = \spTD)$}
Let $g \in L^1(\RR)$ and $\psi \in L^p \cap C^0(\RR)$.
We show that $\wav[g;\psi] \in \spTD(\Half)$, that is, it has at most polynomial growth at infinity.
Because $\psi$ is continuous, $g * \psi$ is continuous.
By Lusin's theorem, there exists a continuous function $g^\star$ such that $g^\star(x) = g(x)$ for almost every $x \in \RR$;
thus, by the continuity of $g * \psi$,
\begin{align}
g^\star * \psi(x) = g * \psi(x), \quad \mbox{for every } x \in \RR.
\end{align}
By the continuity and the integrability of $g^\star$ and $\psi$, there exist $s, t \in \RR$ such that
\begin{gather}
|g^\star (x)| \lesssim (1+x^2)^{-s/2}, \quad s > 1,\\
|\psi (x)| \lesssim (1+x^2)^{-t/2}, \quad tp > 1
\end{gather}
Therefore,
\begin{align}
\Bigg| \int_\RR g(x) \overline{\psi\left(\frac{x-\beta}{\alpha}\right)}\frac{1}{\alpha} \dd x \Bigg|
& \lesssim \Bigg| \int_\RR (1+x^2)^{-s/2}  \left( 1 + \left( \frac{x - \beta}{\alpha} \right)^2 \right)^{-t/2}  \dd x \Bigg| \alpha^{-1} \\
& \lesssim \Bigg| \int_\RR (1+x^2)^{-s/2}  \left( 1 + \left( x - \beta \right)^2 \right)^{-t/2} \dd x \Bigg| (1 + \alpha^2)^{t/2} \alpha^{-1} \\
& \lesssim (1+\beta^2)^{- \min \left( s, t \right)/2 } (\alpha + 1/\alpha)^{t-1},
\end{align}
which means $\wav[\psi; g]$ is a locally integrable function that grows at most polynomially at infinity.

Note that if $(t -1)p < m-1$ then $\wav[\psi;g] \in L^p(\Half; \alpha^{-m} \dd \alpha \dd \beta)$,
because $| \wav[\psi;g](\alpha,\beta) |^p$ behaves as $\beta^{-\min(s,t)} \alpha^{(t - 1)p}$ at infinity.

\subhead{Case 3f}{$(\spX = \spLD{1} \mbox{ and } \spZ = \spLD{p} \mbox{ then } \spB = \spLD{p} \mbox{ and } \spA = \spTD)$}
Let $g \in \spLD{1}(\RR)$ and $\psi \in \spLD{p}$.
We estimate \refeq{eq:case3dtd}.
Fix an arbitrary $\TT \in \spSch(\Half)$.
By the assumption that $\psi \in \spLD{p}(\RR)$, for every fixed $\alpha$, 
 $\TT(\alpha, \cdot) * \psi \in \spsmooth \cap L^p(\RR)$.
Therefore, we can take $\psi^\star \in \spsmooth \cap L^p(\RR)$ such that 
$\psi^\star = \psi$ a.e. and $\TT(\alpha,\cdot) * \psi^\star = \TT(\alpha,\cdot) * \psi$.
In the same way, we can take $g^\star \in \spsmooth \cap L^1(\RR)$ such that $g^\star = g$ almost everywhere and $\TT(\alpha,\cdot) * g^\star = \TT(\alpha,\cdot) * g$.
Therefore, this case reduces to show that if $\spX = L^1$ and $\spZ = L^p$ then $\spA = \spTD$. This coincides with case 3e.

\subhead{Step 4}{Class $\spY(\YY^{m+1})$ of $\rid_\psi f(\bu, \alpha,\beta)$}
The last column $(\spY)$ is obtained by applying $\spY(\YY^{m+1}) = \spX(\Sph^{m-1}) \ttprod \spA(\Half)$.
Recall that for $\Sph^{m-1}$, as it is compact, $\spcomp = \spSch = \spM = \spsmooth$ and $\spED = \spCD = \spTD = \spLD{p} = \spSD$.
Therefore, we have $\spY$ as in the last column of \reftab{tab:weakridge}.

\section{Proof of \refthm{thm:eq.ac}} \label{app:proof.eq.ac}

Let $(\psi, \eta) \in \spSch(\RR) \times \spTD(\RR)$.
Assume that $\widehat{\eta}$ is singular at $0$. That is, there exists $k \in \NN_0$ such that
\begin{align}
\widehat{\eta}(\zeta) = \sum_{j=0}^k c_j \delta^{(j)}(\zeta), \quad \zeta \in \{ 0 \}.
\end{align}
Assume there exists a neighborhood $\neighbour$ of $0$ such that $\widehat{\eta} \in C^0(\neighbour \setminus \{ 0 \})$.
Note that the continuity implies local integrability.
We show that $\psi$ and $\eta$ are admissible
if and only if there exists $u \in \spM(\RR)$ such that
\begin{gather}
\bp^m u = \overline{\refl{\psi}} * \left( \eta - \sum_{j=0}^k c_j z^j \right), \quad \mbox{ and } \quad 
\int_{\RR \setminus \{ 0 \}} \widehat{u}(\zeta) \dd \zeta \neq 0.
\end{gather}
Recall that the Fourier transform $\spM(\RR) \to \spCD(\RR)$ is bijective. Thus, the action of $\widehat{u}$ on the indicator function $\ind_{\RR \setminus \{ 0 \}}(\zeta)$ is always finite.

\subhead{Sufficiency}{}
On $\neighbour \setminus \{ 0 \}$, $\widehat{\eta}$ coincides with a function. Thus the product $\overline{\widehat{\psi}(\zeta)}\widehat{\eta}(\zeta)|\zeta|^{-m}$ is defined in the sense of ordinary functions, and coincides with $\widehat{u}(\zeta)$.
On $\RR \setminus \neighbour$, $|\zeta|^{-m}$ is in $\spM(\RR \setminus \neighbour)$. Thus, the product $\overline{\widehat{\psi}(\zeta)}\widehat{\eta}(\zeta)|\zeta|^{-m}$ is defined in the sense of distributions, which is associative because it contains at most one tempered distribution ($\spSch \cdot \spTD \cdot \spM $), and reduces to $\widehat{u}(\zeta)$.
Therefore, 
\begin{align}
\frac{K_{\psi,\eta}}{(2 \pi)^{m-1}} = \left( \int_{\neighbour \setminus \{ 0 \} }  + \int_{\RR \setminus \neighbour } \right) \widehat{u}(\zeta) \dd \zeta,
\end{align}
which is finite by assumption.

\subhead{Necessity}{}
Write $\neighbour_0 := \neighbour \cap [-1,1]$ and $\neighbour_1 := \RR \setminus \neighbour_0$.
By the assumption that 
$\int_{\neighbour_0  \setminus \{ 0 \}} \overline{\widehat{\psi}(\zeta)}\widehat{\eta}(\zeta) |\zeta|^{-m} \dd \zeta $
is absolutely convergent and $\widehat{\eta}$ is continuous in $\neighbour_0  \setminus \{ 0 \}$, there exists $\eps > 0$ such that
\begin{align}
\Big| \overline{\widehat{\psi}(\zeta)}\widehat{\eta}(\zeta) \Big| \lesssim |\zeta|^{m-1 + \eps}, \quad \zeta \in \neighbour_0  \setminus \{ 0 \}.
\end{align}
Therefore, there exists $v_0 \in L^1( \RR) \cap C^0(\RR \setminus \{ 0 \})$ such that its restriction to $\neighbour_0 \setminus \{ 0 \}$ coincides with
\begin{align}
\overline{\widehat{\psi}(\zeta)}\widehat{\eta}(\zeta) = |\zeta|^{m} v_0(\zeta), \quad \zeta \in \neighbour_0 \setminus \{ 0 \}.
\end{align}
By integrability and continuity, $v_0 \in L^\infty(\RR)$. In particular, both $\lim_{\zeta \to +0} v_0(\zeta)$ and $\lim_{\zeta \to -0} v_0(\zeta)$ are finite.

However, in $\neighbour_1$, $|\zeta|^{-m} \in \spM(\neighbour_1)$. By the construction, $\overline{\widehat{\psi}}\cdot\widehat{\eta} \in \spCD(\RR)$.
Thus, there exists $v_1 \in \spCD(\RR)$ such that
\begin{align}
\overline{\widehat{\psi}(\zeta)} \widehat{\eta}(\zeta) = |\zeta|^{m}  v_1(\zeta), \quad \zeta \in \neighbour_1.
\end{align}
where the equality is in the sense of distribution.

Let
\begin{align}
v := v_0 \cdot \ind_{\neighbour_0} + v_1 \cdot \ind_{\neighbour_1}.
\end{align}
Clearly, $v \in \spCD(\RR)$ because $v_0 \cdot \ind_{\neighbour_0} \in \spED(\RR)$ and $v_1 \cdot \ind_{\neighbour_1} \in \spCD(\RR)$.
Therefore, there exists $u \in \spM(\RR)$ such that $\widehat{u} = v$ and
\begin{align}
\overline{\widehat{\psi}(\zeta)} \widehat{\eta}(\zeta) = |\zeta|^{m}  \widehat{u}(\zeta), \quad \zeta \in \RR \setminus \{ 0 \}.
\end{align}
By the admissibility condition,
\begin{align}
\int_{\RR \setminus \{ 0\}} \widehat{u}(\zeta) \dd \zeta = 
\int_{\neighbour_0 \setminus \{ 0\}} v_0(\zeta) \dd \zeta + \int_{\neighbour_1} v_1(\zeta) \dd \zeta \neq 0.
\end{align}
In consideration of the singularity at $0$, we have
\begin{align}
 \overline{\widehat{\psi}(\zeta)} \left( \widehat{\eta}(\zeta) - \sum_{j=0}^k c_j \delta^{(j)}(\zeta) \right) = |\zeta|^{m}  \widehat{u}(\zeta), \quad \zeta \in \RR.
\end{align}
By taking the Fourier inversion in the sense of distributions, 
\begin{align}
\left[ \overline{\widetilde{\psi}} * \left( \eta - \sum_{j=0}^k c_j z^{j} \right) \right] (z)
&= \bp^{m} u(z) , \quad z \in \RR.
\end{align}

\section{Proof of \refthm{thm:formula}} \label{app:proof.formula}
Let $f \in L^1(\RR^m)$ satisfy $\widehat{f} \in L^1(\RR^m)$ and $(\psi,\eta) \in \spSch(\RR) \times \spLizD(\RR)$ be admissible.
\rsho{For simplicity, we rescale $\psi$ to satisfy $K_{\psi,\eta} = 1$.}
Write
\begin{align}
I(\mathbf{x}; \eps, \delta) := 
\int_{\Sph^{m-1}} 
\int_\eps^\delta \lebesgue \rid_\psi f \left( \bu, \alpha, \bu \cdot \xx - \alpha \pz \right) \eta(\pz) \frac{\dd\pz\dd \alpha \dd \bu}{\alpha^{\rsho{m}}}.
\end{align}
We show that
\begin{align}
\lim_{\substack{\delta \to \infty \\ \eps \to 0}}
 I(\mathbf{x}; \eps, \delta) = f(\mathbf{x}), \quad \mbox{a.e. } \mathbf{x} \in \mathbb{R}^m.
\end{align}
and the equality holds at every continuous point of $f$.

By using the Fourier slice theorem \rsho{in the sense of distribution}, 
\begin{align}
\lebesgue \rid_\psi f \left( \bu, \alpha, \beta - \alpha \pz \right) \eta(\pz) \dd \pz
&= \frac{1}{2 \pi} \lebesgue \widehat{f}(\fp \bu) \overline{\widehat{\psi}(\alpha \fp)} \widehat{\eta}(\alpha \fp) e^{i \fp \beta} \dd \fp \\
&= \frac{1}{2 \pi} \int_{\RR \setminus \{ 0 \}} \widehat{f}(\fp \bu) \rsho{\widehat{u}(\alpha \fp) |\alpha \omega|^m} e^{i \fp \beta} \dd \fp, \label{eq:riddz} 
\end{align}
\rsho{where $|\zeta|^m \widehat{u}(\zeta) := \overline{\widehat{\psi}(\zeta)} \widehat{\eta}(\zeta) \ (\zeta \neq 0)$ is defined as in \refthm{thm:eq.ac}.}

Then,
\begin{align}
\int_\eps^\delta \refeq{eq:riddz} \frac{\dd \alpha}{\alpha^{\rsho{m}}} 
&= \frac{1}{2 \pi} \int_{\RR \setminus \{ 0 \}} \int_\eps^\delta \widehat{f}(\fp \bu) \rsho{\widehat{u}(\alpha \fp) |\omega|^{m} } e^{i \fp \beta} \dd \alpha \dd \fp \\
&= \rsho{ \frac{1}{2 \pi} \lebesgue \int_{\eps \leq \frac{\fz}{\fp} \leq \delta } \rsho{\widehat{u}(\fz)} \widehat{f}(\fp \bu) e^{i \fp \beta} |\fp|^{m-1} \dd \fz \dd \fp } \\
&= \frac{1}{2 \pi} \int_0^\infty \int_{r \eps \leq |\fz| \leq r \delta} \rsho{\widehat{u}(\fz)} \widehat{f}(\sign(\fz) r \bu) \exp(\sign(\fz) i r \beta) r^{m-1} \dd \fz \dd r, \label{eq:riddzda}
\end{align}
\rsho{where the second equation follows by changing the variable $\fz \gets \alpha \omega$ with $\alpha^{m-1} \dd \alpha = |\fp|^{m-1} |\fz|^{-m} \dd \fz$;
the third equation follows by changing the variable $r \gets |\fp|$ with $\sign \, \fp = \sign \, \fz$.}
In the following, we substitute $\beta \gets \bu \cdot \xx$. Observe that in $\int_{\Sph^{m-1}} \dd \bu$, 
\begin{align}
\int_{\Sph^{m-1}} \widehat{f}(- r \bu) \exp(-ir\bu \cdot \xx)\dd \bu = \int_{\Sph^{m-1}} \widehat{f}( r \bu) \exp(ir\bu \cdot \xx)\dd \bu;
\end{align}
hence, we can omit $\sign \, \fz$.

Then, by substituting $\beta \gets \bu \cdot \xx$ and changing the variable $\xxi \gets r \bu$,
\begin{align}
I(\mathbf{x}; \eps, \delta) &=
\int_{\Sph^{m-1}} \eqref{eq:riddzda} \, \dd \bu \\
&= \frac{1}{2 \pi} \int_{\Sph^{m-1}} \int_0^\infty \left[ \int_{r \eps \leq |\fz| \leq r \delta} \rsho{\widehat{u}(\fz)} \dd \zeta \right] 
\widehat{f}( r \bu) e^{ i r \bu \cdot \xx} r^{m-1}  \dd r \dd \bu, \\
&= \frac{1}{2 \pi} \int_{\RR^m} \left[ \int_{\| \xxi \| \eps \leq |\fz| \leq \| \xxi \| \delta} \rsho{\widehat{u}(\fz)} \dd \zeta \right] 
 \widehat{f}(\xxi) e^{i \xxi \cdot \xx} \dd \zeta \dd \xxi.
\end{align}
\rsho{Recall that $\widehat{u} \in \spCD(\RR)$; thus, its action is continuous. That is, the limits and the integral commute.}

\rsho{Therefore,
\begin{align}
\drid_\eta \rid_\psi f(\xx)
&=
\lim_{\substack{\delta \to \infty \\ \eps \to 0}} I(\mathbf{x}; \eps, \delta) \\
&= \frac{1}{(2 \pi)}
\int_{\RR^m} \left[ \int_{\RR \setminus \{ 0 \} }\widehat{u}(\zeta) \dd \zeta \right] \widehat{f}(\xxi) e^{i \xxi \cdot \xx} \dd \xxi \\
&= \frac{1}{(2 \pi)^{m}}
\int_{\RR^m} \widehat{f}(\xxi) e^{i \xxi \cdot \xx} \dd \xxi \\
&= f(\xx), \quad \mbox{a.e. } \xx \in \RR^m
\end{align}
where the last equation follows by the Fourier inversion formula,
a consequence of which the equality holds at $\xx_0$
 if $f$ is continuous at $\xx_0$.}

\section{Proof of \refthm{thm:formula.radon}} \label{app:proof.formula.radon}
Let $f \in L^1(\RR^m)$ and $(\psi,\eta) \in \spSch(\RR) \times \spTD(\RR)$. Assume that there exists $u \in \spsmooth \cap L^1(\RR)$ that is real-valued, $\bp^m u = \overline{\widetilde{\psi}} * \eta$ and $\int_\RR \widehat{u}(\zeta) \dd \zeta = -1$.
Write
\begin{align}
I(\mathbf{x}; \eps, \delta) := 
\int_{\Sph^{m-1}} 
\int_\eps^\delta \lebesgue \rid_\psi f \left( \bu, \alpha, \bu \cdot \xx - \alpha \pz \right) \eta(\pz) \frac{\dd\pz\dd \alpha \dd \bu}{\alpha^{m}}.
\end{align}
We show that
\begin{align}
\lim_{\substack{\delta \to \infty \\ \eps \to 0}}
 I(\mathbf{x}; \eps, \delta) = \drad \bp^{m-1} \rad (\mathbf{x}), \quad \mbox{a.e. } \mathbf{x} \in \mathbb{R}^m.
\end{align}

In the following we write $(\cdot)_\alpha(p) = (\cdot)(p/\alpha)\rsho{/\alpha}$.
\rsho{By using the convolution form,}
\begin{align}
\lebesgue \rid_\psi f \left( \bu, \alpha, \beta - \alpha \pz \right) \eta(\pz) \dd \pz
&= \left[ \rad f(\bu, \cdot) * \left( \overline{\refl{\psi}} * \eta \right)_\alpha \right] (\beta) \\
&= \left[ \rad f(\bu, \cdot) * \left( \bp^m u \right)_\alpha \right] (\beta). \label{eq:convdz}
\end{align}

Observe that
\begin{align}
\int_\eps^\delta ( \bp^m u)_\alpha(p) \frac{\dd \alpha}{\alpha^{m}} 
&= \bp^{m-1} \left[ \int_\eps^\delta (\bp u) \left( \frac{p}{\alpha} \right) \frac{\dd \alpha}{\alpha^{2}} \right]\\
&= \bp^{m-1} \left[ \frac{1}{p} \int_{p/\delta}^{p/\eps} (\bp u) (\pz) \dd \pz \right]\\
&= \bp^{m-1} \left[ \frac{1}{p} \hil u \left( \frac{p}{\eps} \right) - \frac{1}{p} \hil u \left( \frac{p}{\delta} \right) \right]\\
&= \bp^{m-1} [ k_\eps(p) - k_\delta(p) ], \label{eq:bpuda}
 \end{align}
where the first equality follows by repeatedly applying $(\bp u)_\alpha = \alpha \bp (u_\alpha)$;
the second equality follows by substituting $z \gets p/\alpha$;
the fourth equality follows by defining
\begin{align}
k ( z ) := \frac{1}{z} \hil u \left( z \right) \quad \mbox{ and } \quad  k_\gamma(p) := \frac{1}{\gamma} k \left( \frac{p}{\gamma} \right) \quad \mbox{ for } \gamma = \eps, \delta.
\end{align} 
Therefore, we have
\begin{align}
\int_\eps^\delta \refeq{eq:convdz} \frac{\dd \alpha}{\alpha^{m}}
&= \left[\rad f(\bu, \cdot) * \refeq{eq:bpuda} \right](\beta)  \\
&= \left[ \bp^{m-1} \rad f(\bu, \cdot) * ( k_\eps - k_\delta ) \right] (\beta).
\end{align}

\rsho{We show that $k \in L^1 \cap L^\infty(\RR)$ and $\int_\RR k(z) \dd z = 1$.
To begin with, $k \in L^1(\RR)$ because there exist $s,t > 0$ such that
\begin{align}
&| k(z) | \lesssim |z|^{-1+s}, \quad \mbox{ as } |z| \to 0 \\
&| k(z) | \lesssim |z|^{-1-t}, \quad \mbox{ as } |z| \to \infty.
\end{align}
The first claim holds because $u$ is real-valued and thus $\widehat{u}$ is odd, then
\begin{align}
\hil u(0)
&= \int_\RR \sign \zeta \cdot \widehat{u}(\zeta) \dd \zeta \\
&= \int_{(-\infty,0]} \widehat{u}(\zeta) \dd \zeta- \int_{(0,\infty)} \widehat{u}(\zeta) \dd \zeta \\
&=0.
\end{align}
The second claim holds because $u \in L^1(\RR)$ and thus $u$ as well as $\hil u$ decays at infinity.
Then, by the continuity and the integrability of $k$, it is bounded.
By the assumption that $\int_\RR \widehat{u}(\zeta) \dd \zeta = -1$,
\begin{align}
\int_\RR k(z) \dd z
&= -\int_\RR \frac{\hil u(z)}{0 - z} \dd z \\
&= -u(0) \\
&= 1.
\end{align}}

Write
\begin{align}
J(\bu,p) := \bp^{m-1} \rad f(\bu, p).
\end{align}
Because $k \in L^1(\RR)$ and $\lebesgue k(\pz) \dd\pz= 1$,
$k_\eps$ is an approximation of the identity \cite[III, Th.2]{Stein.singular}. Then,
\begin{align}
\lim_{\eps \to 0} J(\bu,\cdot) * k_\eps(p) = J(\bu,p), \quad \mbox{ a.e. } (\bu,p) \in \Sph^{m-1} \times \RR.
\end{align}
However, as $k \in L^\infty(\RR)$,
\begin{gather}
\| J * k_\delta \|_{L^\infty(\Sph^{m-1} \times \RR)} \leq \delta^{-1} \| J \|_{L^1(\Sph^{m-1} \times \RR )} \| k \|_{L^\infty(\RR)},
\intertext{and thus,}
\lim_{\delta \to \infty} J(\bu,\cdot) * k_\delta (p) = 0, \quad \mbox{ a.e. } (\bu,p) \in \Sph^{m-1} \times \RR.
\end{gather}

Because it is an approximation to the identity, $J * k_\gamma \in  L^1(\Sph^{m-1} \times \RR)$ for $0 \leq \gamma$.
Hence, there exists a maximal function $M(\bu, p)$ \cite[III, Th.2]{Stein.singular}
such that
\begin{align}
\sup_{0 < \eps} |(J(\bu,\cdot) * \vk_\eps) (p)| \lesssim M(\bu, p).
\end{align}
\rsho{Therefore, $|J(\bu,\cdot) * (\vk_\eps - \vk_\delta)(\bu \cdot \xx)|$ is uniformly integrable \cite[Ex.~4.15.4]{Brezis.new} on $\Sph^{m-1}$.
That is, if $\Omega \subset \Sph^{m-1}$ satisfies $\int_\Omega \dd \bu \leq A$ then 
\begin{align}
\int_\Omega | J(\bu, \cdot ) * k_\gamma |(\bu \cdot \xx ) \dd \bu \lesssim A \sup_{\bu,p}|M(\bu,p)|, \quad \forall \gamma \geq 0.
\end{align}
Thus, by the Vitali convergence theorem, we have}
\begin{align}
\drid_\eta \rid_\psi f(\xx)
&= \lim_{\substack{\delta \to \infty \\ \eps \to 0}} \int_{\Sph^{m-1}} [J(\bu,\cdot) * (\vk_\eps - \vk_\delta)](\bu \cdot \xx) \dd \bu \\
&\mathop{=} \int_{\Sph^{m-1}} J(\bu, \bu \cdot \xx) \dd \bu, \quad \mathrm{a.e.} \ \xx \in \RR^m\\
&= \drad \bp^{m-1} \rad f(\xx).
\end{align}

\section{Proof of \refthm{thm:formula.L2}} \label{app:proof.formula.L2}
Let $f \in L^2(\RR^m)$ and $(\psi,\eta)$ be admissible with $K_{\psi,\eta}=1$.
Assume without loss of generality that $(\psi,\psi)$ and $(\eta,\eta)$ are self-admissible respectively.
Write
\begin{align}
I[f; (\eps,\delta)](\xx)
&:= 
\int_{\Sph^{m-1}} 
\int_\eps^\delta \lebesgue \rid_\psi f \left( \bu, \alpha, \bu \cdot \xx - \alpha \pz \right) \eta(\pz) \frac{\dd\pz\dd \alpha \dd \bu}{\alpha^{m}}.
\end{align}
In the following we write $\Omega[\eps,\delta] := \Sph^{m-1} \times [\RR_+ \setminus (\eps,\delta)] \times \RR \subset \YY^{m+1}$.
We show that
\begin{align}
\lim_{\substack{\delta \to \infty \\ \eps \to 0}} \big\| f - I[ f; (\eps, \delta) ] \big\|_2 =0.
\end{align}
Observe that 
\begin{align}
\big \|f - I[f; (\eps, \delta) ] \big \|_2
&= \sup_{\|g\|_2 = 1} \big| \ip{ f -I[f; (\eps, \delta) ], g } \big| \\
&= \sup_{\|g\|_2 = 1} \big| \ip{ \rid_\psi f, \rid_\eta g }_{\Omega[\eps,\delta] } \big| \\
&\leq \sup_{\|g\|_2 = 1} \big| \rid_\psi f \big|_{L^2(\Omega[\eps,\delta])}  \big\| \rid_\eta g \big\|_{L^2(\YY^{m+1})} \\
&= \sup_{\|g\|_2 = 1} \big| \rid_\psi f \big|_{L^2(\Omega[\eps,\delta])} \big\| g \big\|_2 \\
& \to 0 \cdot 1, \quad \mbox{ as } \eps \to 0, \delta \to \infty
\end{align}
where the third inequality follows by the Schwartz inequality; the last limit follows by $\| \rid_\psi f \|_{L^2(\Omega[\eps,\delta])}$, which
shrinks as the domain $\Omega[\eps,\delta]$ tends to $\emptyset$.

\section{Proofs of \refeg{eg:sig} and \refeg{eg:adm.sig}} \label{app:proof.sig}

Let $\sig(z) := (1+e^{-z})^{-1}$. Obviously $\sig(z) \in \spsmooth(\RR)$.

\subhead{Step 0}{Derivatives of $\sig(z)$.} 
For every $k \in \NN$, 
\begin{align}
\sig^{(k)}(z) &= S_k(\sigma(z)),
\end{align}
where $S_k(z)$ is a polynomial defined by 
\begin{align}
S_k(z) := 
\left\{
\begin{array}{ll}
z (1-z) & k=1 \\
S_{k-1}'(z) S_1(z) & k>1,
\end{array}
\right.
\end{align}
which is justified by induction on $k$.

\subhead{Step 1}{$\sig, \tanh \in \spM(\RR)$.}
Recall that
$|\sigma(z)| \leq 1$.
Hence, for every $k \in \NN$,
\begin{align}
\big| \sigma^{(k)}(z) \big| = |S_k(\sig(z))| \leq \max_{z \in [0,1]} |S_k(z)| < \infty.
\end{align}
Therefore, every $k \in \NN_0$, $\sig^{(k)}(z)$ is bounded, which concludes $\sig(z) \in \spM(\RR)$.

Hence, immediately $\tanh \in \spM(\RR)$ because
\begin{align}
\tanh(z) = 2 \sig(2 z) -1.
\end{align}

\subhead{Step 2}{$\sig^{(k)} \in \spSch(\RR), \ k \in \NN$.}
Observe that
\begin{align}
\sig'(z) = (e^{\,z/2}+e^{-z/2})^{-2}.
\end{align}
Hence, $\sig'(z)$ decays faster than any polynomial, which means $\sup_z |z^\ell \sig'(z)| < \infty$ for any $\ell \in \NN_0$.
Then, for every $k, \ell \in \NN_0$,
\begin{align}
\sup_z \big| z^\ell \sig^{(k+1)}(z) \big| 
= \sup_z \big| z^\ell S_{k+1}(\sig(z)) \big| 
\leq \max_z |z^\ell \sig'(z)| \cdot \max_z |S'_{k}(\sig(z))| < \infty,
\end{align}
which concludes $\sig' \in \spSch(\RR)$. Therefore, $\sigma^{(k)} \in \spSch(\RR)$ for every $k \in \NN$.

\subhead{Step 3}{$\sps \in \spM(\RR)$.}
Observe that
\begin{align}
\sps(z) = \int_0^z \sig(w) \dd w. \label{eq:proof.sig.sps}
\end{align}
Hence, it is already known that $[\sps]^{(k)} = \sig^{(k-1)} \in \spM(\RR)$ for every $k \in \NN$.
We show that $\sps(z)$ has at most polynomial growth.
Write
\begin{align}
\rho(z) := \sps(z) - z_+.
\end{align}
Then $\rho(z)$ attains at $0$ its maximum $\max_z \rho(z) =\log 2$, because
$\rho'(z) < 0$ when $z >0$ and $\rho'(z) >0$ when $z < 0$.
Therefore,
\begin{align}
\big| \sps(z) \big|
\leq |\rho(z)| + |z_+|
\leq \log 2 + |z|,
\end{align}
which concludes $\sps(z) \in \spM$.

\subhead{Step 4}{$\eta = \sig^{(k)}$ is admissible with $\psi = \bp^m \gauss$ when $k \in \NN$ is positive and odd.}
Recall that $\eta = \sig^{(k)} \in \spSch(\RR)$. Hence, $\dual{\widehat{\eta},\widehat{\psi_0}}=\dual{\eta,\psi_0}$.
Observe that if $k$ is odd, then $\sig^{(k)}$ is an odd function and thus $\dual{\eta,\psi_0} = 0$.
However, if $k$ is even, then $\sig^{(k)}$ is an even function and thus $\dual{\eta,\psi_0} \neq 0$.

\subhead{Step 5}{$\sig$ and $\sps$ cannot be admissible with $\psi = \bp^m \gauss$.}
This follows by \refthm{thm:eq.ac}, because both
\begin{align}
\int_\RR \left( \overline{\widetilde{\gauss}} * \sig \right) (z) \dd z \quad \mbox{ and } \quad  \int_\RR \left( \overline{\widetilde{\gauss}} * \sps\right)  (z) \dd z,
\end{align}
diverge.

\subhead{Step 6}{$\sig$ and $\sps$ are admissible with $\psi = \bp^m \gauss'$ and $\psi = \bp^m \gauss''$, respectively.}
Observe that both
\begin{align}
u_0 := \overline{\widetilde{\gauss'}} * \sig = \overline{\widetilde{\gauss}} * \sig' \quad \mbox{ and } \quad  
u_{-1} := \overline{\widetilde{\gauss''}} * \sps = \overline{\widetilde{\gauss}} * \sig',
\end{align}
belong to $\spSch(\RR)$. Hence, $u_0$ and $u_{-1}$ satisfy the sufficient condition in \refthm{thm:eq.ac}.


\begin{thebibliography}{10}
\expandafter\ifx\csname url\endcsname\relax
  \def\url#1{\texttt{#1}}\fi
\expandafter\ifx\csname urlprefix\endcsname\relax\def\urlprefix{URL }\fi
\expandafter\ifx\csname href\endcsname\relax
  \def\href#1#2{#2} \def\path#1{#1}\fi

\bibitem{Murata1996}
N.~Murata,
  \href{http://www.sciencedirect.com/science/article/pii/0893608096000007}{{An
  Integral representation of functions using three-layered betworks and their
  approximation bounds}}, Neural Networks 9~(6) (1996) 947--956.
\newblock \href {http://dx.doi.org/10.1016/0893-6080(96)00000-7}
  {\path{doi:10.1016/0893-6080(96)00000-7}}.
\newline\urlprefix\url{http://www.sciencedirect.com/science/article/pii/0893608096000007}

\bibitem{Sonoda2014}
S.~Sonoda, N.~Murata, {Sampling hidden parameters from oracle distribution},
  in: 24th Int. Conf. Artif. Neural Networks, Vol. 8681, Springer International
  Publishing, Hamburg, Germany, 2014, pp. 539--546.
\newblock \href {http://dx.doi.org/10.1007/978-3-319-11179-7{\_}68}
  {\path{doi:10.1007/978-3-319-11179-7{\_}68}}.

\bibitem{ReLU.Glorot}
X.~Glorot, A.~Bordes, Y.~Bengio,
  \href{http://jmlr.org/proceedings/papers/v15/glorot11a/glorot11a.pdf}{{Deep
  sparse rectifier neural networks}}, in: 14th Int. Conf. Artif. Intell. Stat.
  (AISTATS 2011), Vol.~15, JMLR W\&CP, Fort Lauderdale, FL, USA, 2011, pp.
  315--323.
\newline\urlprefix\url{http://jmlr.org/proceedings/papers/v15/glorot11a/glorot11a.pdf}

\bibitem{maxout}
I.~Goodfellow, D.~Warde-Farley, M.~Mirza, A.~Courville, Y.~Bengio,
  \href{http://jmlr.csail.mit.edu/proceedings/papers/v28/goodfellow13.pdf}{{Maxout
  networks}}, in: 30th Int. Conf. Mach. Learn., Vol.~28, JMLR W\&CP, 2013, pp.
  1319--1327.
\newline\urlprefix\url{http://jmlr.csail.mit.edu/proceedings/papers/v28/goodfellow13.pdf}

\bibitem{ReLU.Hinton}
G.~E. Dahl, T.~N. Sainath, G.~E. Hinton, {Improving deep neural networks for
  LVCSR using rectified linear units and dropout}, in: Acoust. Speech Signal
  Process. (ICASSP), 2013 IEEE Int. Conf., IEEE, 2013, pp. 8609--8613.
\newblock \href {http://dx.doi.org/10.1109/ICASSP.2013.6639346}
  {\path{doi:10.1109/ICASSP.2013.6639346}}.

\bibitem{ReLU.Leaky}
A.~L. Maas, A.~Y. Hannun, A.~Y. Ng,
  \href{https://sites.google.com/site/deeplearningicml2013/relu_hybrid_icml2013_final.pdf}{{Rectifier
  nonlinearities improve neural network acoustic models}}, in: ICML 2013 Work.
  Deep Learn. Audio, Speech, Lang. Process., Atlanta, 2013.
\newline\urlprefix\url{https://sites.google.com/site/deeplearningicml2013/relu_hybrid_icml2013_final.pdf}

\bibitem{Jarrett2009}
K.~Jarrett, K.~Kavukcuoglu, M.~Ranzato, Y.~LeCun, {What is the best multi-stage
  architecture for object recognition?}, in: Comput. Vision, 2009 IEEE 12th
  Int. Conf., Kyoto, 2009, pp. 2146--2153.
\newblock \href {http://dx.doi.org/10.1109/ICCV.2009.5459469}
  {\path{doi:10.1109/ICCV.2009.5459469}}.

\bibitem{Krizhevsky2012}
A.~Krizhevsky, I.~Sutskever, G.~E. Hinton,
  \href{http://papers.nips.cc/paper/4824-imagenet-classification-with-deep-convolutional-neural-networks.pdf}{{ImageNet
  classification with deep convolutional neural networks}}, in: F.~Pereira,
  C.~J.~C. Burges, L.~Bottou, K.~Q. Weinberger (Eds.), Adv. Neural Inf.
  Process. Syst. 25, Curran Associates, Inc., 2012, pp. 1097--1105.
\newline\urlprefix\url{http://papers.nips.cc/paper/4824-imagenet-classification-with-deep-convolutional-neural-networks.pdf}

\bibitem{ReLU.Zeiler}
M.~D. Zeiler, M.~Ranzato, R.~Monga, M.~Z. Mao, K.~Yang, Q.~{Viet Le},
  P.~Nguyen, A.~W. Senior, V.~Vanhoucke, J.~Dean, G.~E. Hinton, {On rectified
  linear units for speech processing}, in: Acoust. Speech Signal Process.
  (ICASSP), 2013 IEEE Int. Conf., IEEE, Vancouver, BC, 2013, pp. 3517--3521.
\newblock \href {http://dx.doi.org/10.1109/ICASSP.2013.6638312}
  {\path{doi:10.1109/ICASSP.2013.6638312}}.

\bibitem{Mhaskar.Micchelli}
H.~Mhaskar, C.~A. Micchelli,
  \href{http://www.sciencedirect.com/science/article/pii/019688589290016P}{{Approximation
  by superposition of sigmoidal and radial basis functions}}, Adv. Appl. Math.
  13~(3) (1992) 350--373.
\newblock \href {http://dx.doi.org/10.1016/0196-8858(92)90016-P}
  {\path{doi:10.1016/0196-8858(92)90016-P}}.
\newline\urlprefix\url{http://www.sciencedirect.com/science/article/pii/019688589290016P}

\bibitem{Leshno1993}
M.~Leshno, V.~Y. Lin, A.~Pinkus, S.~Schocken,
  \href{http://www.sciencedirect.com/science/article/pii/S0893608005801315}{{Multilayer
  feedforward networks with a nonpolynomial activation function can approximate
  any function}}, Neural Networks 6~(6) (1993) 861--867.
\newblock \href {http://dx.doi.org/10.1016/S0893-6080(05)80131-5}
  {\path{doi:10.1016/S0893-6080(05)80131-5}}.
\newline\urlprefix\url{http://www.sciencedirect.com/science/article/pii/S0893608005801315}

\bibitem{Pinkus.survey}
A.~Pinkus, {Approximation theory of the MLP model in neural networks}, Acta
  Numer. 8 (1999) 143--195.
\newblock \href {http://dx.doi.org/10.1017/S0962492900002919}
  {\path{doi:10.1017/S0962492900002919}}.

\bibitem{Ito.Radon}
Y.~Ito,
  \href{http://www.sciencedirect.com/science/article/pii/089360809190075G}{{Representation
  of functions by superpositions of a step or sigmoid function and their
  applications to neural network theory}}, Neural Networks 4~(3) (1991)
  385--394.
\newblock \href {http://dx.doi.org/10.1016/0893-6080(91)90075-G}
  {\path{doi:10.1016/0893-6080(91)90075-G}}.
\newline\urlprefix\url{http://www.sciencedirect.com/science/article/pii/089360809190075G}

\bibitem{Kainen2007}
P.~C. Kainen, V.~K\r{u}rkov\'{a}, A.~Vogt,
  \href{http://linkinghub.elsevier.com/retrieve/pii/S0021904507000081}{{A
  Sobolev-type upper bound for rates of approximation by linear combinations of
  Heaviside plane waves}}, J. Approx. Theory 147~(1) (2007) 1--10.
\newblock \href {http://dx.doi.org/10.1016/j.jat.2006.12.009}
  {\path{doi:10.1016/j.jat.2006.12.009}}.
\newline\urlprefix\url{http://linkinghub.elsevier.com/retrieve/pii/S0021904507000081}

\bibitem{Kurkova2012}
V.~K\r{u}rkov\'{a},
  \href{http://www.sciencedirect.com/science/article/pii/S0893608012001311}{{Complexity
  estimates based on integral transforms induced by computational units}},
  Neural Netw. 33 (2012) 160--7.
\newblock \href {http://dx.doi.org/10.1016/j.neunet.2012.05.002}
  {\path{doi:10.1016/j.neunet.2012.05.002}}.
\newline\urlprefix\url{http://www.sciencedirect.com/science/article/pii/S0893608012001311}

\bibitem{Carroll.Dickinson}
S.~M. Carroll, B.~W. Dickinson, {Construction of neural nets using the Radon
  transform}, in: Int. Jt. Conf. Neural Networks, 1989. IJCNN., Vol.~1, IEEE,
  1989, pp. 607--611.
\newblock \href {http://dx.doi.org/10.1109/IJCNN.1989.118639}
  {\path{doi:10.1109/IJCNN.1989.118639}}.

\bibitem{Helgason.new}
S.~Helgason, {Integral Geometry and Radon Transforms}, Springer-Verlag New
  York, 2011.
\newblock \href {http://dx.doi.org/10.1007/978-1-4419-6055-9}
  {\path{doi:10.1007/978-1-4419-6055-9}}.

\bibitem{Irie.Miyake}
B.~Irie, S.~Miyake, {Capabilities of three-layered perceptrons}, in: IEEE Int.
  Conf. Neural Networks, IEEE, 1988, pp. 641--648.
\newblock \href {http://dx.doi.org/10.1109/ICNN.1988.23901}
  {\path{doi:10.1109/ICNN.1988.23901}}.

\bibitem{Funahashi1989}
K.-I. Funahashi,
  \href{http://www.sciencedirect.com/science/article/pii/0893608089900038}{{On
  the approximate realization of continuous mappings by neural networks}},
  Neural Networks 2~(3) (1989) 183--192.
\newblock \href {http://dx.doi.org/10.1016/0893-6080(89)90003-8}
  {\path{doi:10.1016/0893-6080(89)90003-8}}.
\newline\urlprefix\url{http://www.sciencedirect.com/science/article/pii/0893608089900038}

\bibitem{Jones1992}
L.~K. Jones, {A simple lemma on greedy approximation in Hilbert space and
  convergence rates for projection pursuit regression and neural network
  training}, Ann. Stat. 20~(1) (1992) 608--613.
\newblock \href {http://dx.doi.org/10.1214/aos/1176348546}
  {\path{doi:10.1214/aos/1176348546}}.

\bibitem{Barron1993}
A.~R. Barron, {Universal approximation bounds for superpositions of a sigmoidal
  function}, IEEE Trans. Inf. Theory 39~(3) (1993) 930--945.
\newblock \href {http://dx.doi.org/10.1109/18.256500}
  {\path{doi:10.1109/18.256500}}.

\bibitem{kainen.survey}
P.~C. Kainen, V.~K\r{u}rkov\'{a}, M.~Sanguineti, {Approximating multivariable
  functions by feedforward neural nets}, in: M.~Bianchini, M.~Maggini, L.~C.
  Jain (Eds.), Handb. Neural Inf. Process., Vol.~49 of Intelligent Systems
  Reference Library, Springer Berlin Heidelberg, 2013, pp. 143--181.
\newblock \href {http://dx.doi.org/10.1007/978-3-642-36657-4}
  {\path{doi:10.1007/978-3-642-36657-4}}.

\bibitem{Candes.HA}
E.~J. Cand\`{e}s,
  \href{http://www.sciencedirect.com/science/article/pii/S1063520398902482}{{Harmonic
  analysis of neural networks}}, Appl. Comput. Harmon. Anal. 6~(2) (1999)
  197--218.
\newblock \href {http://dx.doi.org/10.1006/acha.1998.0248}
  {\path{doi:10.1006/acha.1998.0248}}.
\newline\urlprefix\url{http://www.sciencedirect.com/science/article/pii/S1063520398902482}

\bibitem{Candes.PhD}
E.~J. Cand\`{e}s, {Ridgelets: theory and applications}, Ph.D. thesis, Standford
  University (1998).

\bibitem{Rubin.calderon}
B.~Rubin, \href{http://dx.doi.org/10.1007/BF02475988}{{The Calder\'{o}n
  reproducing formula, windowed X-ray transforms, and radon transforms in
  $L^p$-spaces}}, J. Fourier Anal. Appl. 4~(2) (1998) 175--197.
\newblock \href {http://dx.doi.org/10.1007/BF02475988}
  {\path{doi:10.1007/BF02475988}}.
\newline\urlprefix\url{http://dx.doi.org/10.1007/BF02475988}

\bibitem{Donoho.kplane}
D.~L. Donoho, {Tight frames of $k$-plane ridgelets and the problem of
  representing objects that are smooth away from $d$-dimensional singularities
  in $\mathbb{R}^n$}, Proc. Natl. Acad. Sci. United States Am. 96~(5) (1999)
  1828--1833.
\newblock \href {http://dx.doi.org/10.1073/pnas.96.5.1828}
  {\path{doi:10.1073/pnas.96.5.1828}}.

\bibitem{Donoho.ridgelet}
D.~L. Donoho,
  \href{http://www.sciencedirect.com/science/article/pii/S0021904501935683}{{Ridge
  functions and orthonormal ridgelets}}, J. Approx. Theory 111~(2) (2001)
  143--179.
\newblock \href {http://dx.doi.org/10.1006/jath.2001.3568}
  {\path{doi:10.1006/jath.2001.3568}}.
\newline\urlprefix\url{http://www.sciencedirect.com/science/article/pii/S0021904501935683}

\bibitem{Starck2010}
J.-L. Starck, F.~Murtagh, J.~M. Fadili,
  \href{http://www.cambridge.org/9780521119139}{{The ridgelet and curvelet
  transforms}}, in: Sparse Image Signal Process. Wavelets, Curvelets, Morphol.
  Divers., Cambridge University Press, 2010, pp. 89--118.
\newblock \href {http://dx.doi.org/10.1017/CBO9780511730344.006}
  {\path{doi:10.1017/CBO9780511730344.006}}.
\newline\urlprefix\url{http://www.cambridge.org/9780521119139}

\bibitem{Rubin.ridgelet}
B.~Rubin,
  \href{http://www.sciencedirect.com/science/article/pii/S1063520304000168}{{Convolution–backprojection
  method for the $k$-plane transform, and Calder\'{o}n's identity for ridgelet
  transforms}}, Appl. Comput. Harmon. Anal. 16~(3) (2004) 231--242.
\newblock \href {http://dx.doi.org/10.1016/j.acha.2004.03.003}
  {\path{doi:10.1016/j.acha.2004.03.003}}.
\newline\urlprefix\url{http://www.sciencedirect.com/science/article/pii/S1063520304000168}

\bibitem{Kostadinova2014}
S.~Kostadinova, S.~Pilipovi\'{c}, K.~Saneva, J.~Vindas, {The ridgelet transform
  of distributions}, Integr. Transform. Spec. Funct. 25~(5) (2014) 344--358.
\newblock \href {http://dx.doi.org/10.1080/10652469.2013.853057}
  {\path{doi:10.1080/10652469.2013.853057}}.

\bibitem{Kostadinova2015}
S.~Kostadinova, S.~Pilipovi\'{c}, K.~Saneva, J.~Vindas, {The Ridgelet Transform
  and Quasiasymptotic Behavior of Distributions}, Oper. Theory Adv. Appl. 245
  (2015) 185--197.
\newblock \href {http://dx.doi.org/10.1007/978-3-319-14618-8{\_}13}
  {\path{doi:10.1007/978-3-319-14618-8{\_}13}}.

\bibitem{Schwartz.new}
L.~Schwartz, {Th\'{e}orie des Distributions}, nouvelle Edition, Hermann, Paris,
  1966.

\bibitem{Treves.new}
F.~Tr\`{e}ves, {Tological Vector Spaces, Distributions and Kernels}, Academic
  Press, 1967.

\bibitem{Rudin.FA.new}
W.~Rudin, {Functional Analysis}, 2nd Edition, Higher Mathematics Series,
  McGraw-Hill Education, 1991.

\bibitem{Brezis.new}
H.~Brezis, {Functional Analysis, Sobolev Spaces and Partial Differential
  Equations}, 1st Edition, Universitext, Springer-Verlag New York, 2011.
\newblock \href {http://dx.doi.org/10.1007/978-0-387-70914-7}
  {\path{doi:10.1007/978-0-387-70914-7}}.

\bibitem{Yosida.new}
K.~Yosida, {Functional Analysis}, 6th Edition, Springer-Verlag Berlin
  Heidelberg, 1995.
\newblock \href {http://dx.doi.org/10.1007/978-3-642-61859-8}
  {\path{doi:10.1007/978-3-642-61859-8}}.

\bibitem{Yuan}
W.~Yuan, W.~Sickel, D.~Yang, {Morrey and Campanato Meet Besov, Lizorkin and
  Triebel}, Lecture Notes in Mathematics, Springer Berlin Heidelberg, 2010.
\newblock \href {http://dx.doi.org/10.1007/978-3-642-14606-0}
  {\path{doi:10.1007/978-3-642-14606-0}}.

\bibitem{Holschneider.new}
M.~Holschneider, {Wavelets: An Analysis Tool}, Oxford mathematical monographs,
  The Clarendon Press, 1995.

\bibitem{Hertle1983}
A.~Hertle, \href{http://dx.doi.org/10.1007/BF01252856}{{Continuity of the radon
  transform and its inverse on Euclidean space}}, Math. Zeitschrift 184~(2)
  (1983) 165--192.
\newblock \href {http://dx.doi.org/10.1007/BF01252856}
  {\path{doi:10.1007/BF01252856}}.
\newline\urlprefix\url{http://dx.doi.org/10.1007/BF01252856}

\bibitem{Grafakos.classic}
L.~Grafakos, {Classical Fourier Analysis}, 2nd Edition, Graduate Texts in
  Mathematics, Springer New York, 2008.
\newblock \href {http://dx.doi.org/10.1007/978-0-387-09432-8}
  {\path{doi:10.1007/978-0-387-09432-8}}.

\bibitem{GelfandShilov}
I.~M. Gel'fand, G.~E. Shilov, {Generalized Functions, Vol. 1: Properties and
  Operations}, Academic Press, New York, 1964.

\bibitem{LS.phantom}
L.~A. Shepp, B.~F. Logan, {The Fourier reconstruction of a head section}, Nucl.
  Sci. IEEE Trans. 21~(3) (1974) 21--43.
\newblock \href {http://dx.doi.org/10.1109/TNS.1974.6499235}
  {\path{doi:10.1109/TNS.1974.6499235}}.

\bibitem{Stein.singular}
E.~M. Stein, {Singular Integrals and Differentiability Properties of
  Functions}, Princeton Mathematical Series (PMS), Princeton University Press,
  1970.

\end{thebibliography}
\end{document}